%% file: 0.main.tex
\documentclass[12pt]{article}
\usepackage{amsmath, amsthm}
\usepackage{graphicx}
\usepackage{enumerate}
\usepackage{natbib}
\usepackage{url} 
\usepackage{kotex}
\usepackage{amsfonts,amssymb}
\usepackage{mathtools}
\usepackage{algorithm}
\usepackage{bbm}
\usepackage{mathrsfs}
\usepackage{verbatim,lscape}
\usepackage{fancyhdr}
\usepackage{enumitem}
\usepackage{algpseudocode}
\usepackage{setspace}
\usepackage{subcaption}
\usepackage{xcolor}

\newcommand{\blind}{1}

\newcommand{\norm}[1]{\lVert#1\rVert}

\newcommand{\TU}{\tilde{U}}

\newcommand{\TY}{\tilde{Y}}

\renewcommand{\norm}[1]{\|{#1} \|}

\newcommand{\wh}{\widehat}

\newcommand{\fnorm}[1]{\left\|#1\right\|_{F}}
\newcommand{\opnorm}[1]{\left\|#1\right\|_{op}}

\newcommand{\Tr}{\mathop{\sf Tr}}

\newcommand{\iprod}[2]{ \langle #1, #2 \rangle}

\newcommand{\TZ}{\tilde{Z}}
\newcommand{\E}{\mathbb{E}}
\newcommand{\R}{\mathbb{R}}
\renewcommand{\P}{\mathbb{P}}
\newcommand{\calE}{\mathcal{E}}

\newcommand{\I}{\mathcal{I}}
\newcommand{\nj}{n_{\mathcal{C}_l}}
\newcommand{\C}{\mathcal{C}}

\newcommand{\gamp}{\Gamma^{\dagger}}

\newcommand{\<}{\langle}
\renewcommand{\>}{\rangle}

\theoremstyle{plain} 
\newtheorem{theorem}{Theorem}
\newtheorem{lemma}{Lemma}

\newtheorem{corollary}{Corollary}

\theoremstyle{definition} 

\newtheorem{assumption}{Assumption}

\theoremstyle{remark} 
\newtheorem{remark}{Remark}

\def\ntopics{K}

\begin{document}
\date{}
\bibliographystyle{agsm}

\def\spacingset#1{\renewcommand{\baselinestretch}%
{#1}\small\normalsize} \spacingset{1}


\if1\blind
{
  \title{\bf Graph Topic Modeling for Documents with Spatial or Covariate Dependencies}
  \author{Yeo Jin Jung\\
    Department of Statistics, The University of Chicago\\
    and \\
    Claire Donnat \\
    Department of Statistics, The University of Chicago}
  \maketitle
} \fi

\if0\blind
{
  \title{\bf Graph Topic Modeling for Documents with Spatial or Covariate Dependencies}
    \author{Anonymous Authors}
  \maketitle
}\fi

\bigskip
\begin{abstract}
We address the challenge of incorporating document-level metadata into topic modeling to improve topic mixture estimation.  To overcome the computational complexity and lack of theoretical guarantees in existing Bayesian methods,  we extend probabilistic latent semantic indexing (pLSI), a frequentist framework for topic modeling, by incorporating document-level covariates or known similarities between documents through a graph formalism. Modeling documents as nodes and edges denoting similarities, we propose a new estimator based on a fast graph-regularized iterative singular value decomposition (SVD) that encourages similar documents to share similar topic mixture proportions. We characterize the estimation error of our proposed method by deriving high-probability bounds and develop a specialized cross-validation method to optimize our regularization parameters. We validate our model through comprehensive experiments on synthetic datasets and three real-world corpora, demonstrating improved performance and faster inference compared to existing Bayesian methods.
\end{abstract}

\noindent%
\vfill

\newpage
\spacingset{1.4} 

\section{Introduction}\label{sec:intro}

\input{1.introduction}

\section{Graph-Aligned pLSI}\label{sec:model}

\input{2A.model}

\input{2B.algorithm}

\section{Theoretical results}\label{sec:theory}
\input{3.theory}

\subsection{Synthetic Experiments}\label{sec:synthetic_experiment}
\input{4.synthetic_experiments}

\section{Real-World Experiments}\label{sec:experiments}
\input{5.real_experiments}


\section{Conclusion}
\label{sec:conc}

In this paper, we present Graph-Aligned pLSI (GpLSI), a topic model that leverages document-level metadata to improve estimation of the topic mixture matrix. We incorporate metadata by translating it into document similarity, which is then represented as edges connecting two documents on a graph. GpLSI is a powerful tool that integrates two complementary sources of information: word frequencies that traditional topic models use, and the document graph induced from metadata, which encodes which documents should share similar topic mixture proportions. To the best of our knowledge, this is the first framework to incorporate document-level metadata into topic models with theoretical guarantees.

At the core of GpLSI is an iterative graph-aligned singular value decomposition applied to the observed document-word matrix $X$. This procedure projects word frequencies to low-dimensional topic spaces, while ensuring that neighboring documents on the graph share similar topic mixtures. Our SVD approach can also be applied to other works that require dimension reduction with structural constraints on the samples. Additionally, we propose a novel cross validation technique to optimize the level of graph regularization by using the hierarchy of minimum spanning trees to define folds. 

Our theoretical analysis and synthetic experiments confirm that GpLSI outperforms existing methods, particularly in ``short-document'' cases, where the scarcity of words is mitigated by smoothing mixture proportions along neighboring documents. Overall, GpLSI is a fast, highly effective topic model when there is a known structure in the relationship of documents. 

We believe that our work offers valuable insights into structural topic models and opens up several avenues for further exploration. A promising direction is to incorporate structure to the topic matrix $A$ while jointly optimizing structural constraints on $W$. While our work focuses on low-$p$ regime, real world applications, such as genomics data with large $p$, could benefit from introducing sparsity to the word composition of each topic. 

\clearpage
\appendix  

\begin{center}
    {\large \bf Appendix for ``Graph Topic Modeling for Documents with Spatial or Covariate Dependencies''}\\[10pt]
    {\normalsize Yeo Jin Jung and Claire Donnat}\\[5pt]
    {\small Department of Statistics, The University of Chicago}
\end{center}

\spacingset{1.4} 

\section{Optimizing graph regularization parameter $\rho$}

In this section, we propose a novel graph cross-validation method which effectively finds the optimal graph regularization parameter by partitioning nodes into folds based on a natural hierarchy derived from a minimum spanning tree. The procedure is summarized in the following algorithm.

\begin{algorithm}
\setstretch{1.35}
\caption{Cross Validation using Minimum spanning tree at iteration $t$} \label{algo:2}
\begin{algorithmic}[1]
\State \textbf{Input:} Observation $X$, incidence matrix $\Gamma$, minimum spanning tree $\mathcal{T}$ of $\mathcal{G}$, previous estimate $\wh V^{t-1}$
\State \textbf{Output:} $\wh \rho^{t}$
\State 1. Randomly choose the source document $d_s$.
\State 2. Divide documents into $b$ folds : $d_i \in \mathcal{I}_k$ if $\emph{d}_{\mathcal{T}}(d_i, d_s) \mod b=k-1$, for $i\in [n]$ and $k \in [b]$.
\For{ \text{each leave-out fold} $\mathcal{I}_k, k \in [b]$}
\State Interpolation of $X^k$ with average of neighbors: $X^{k}_{i\cdot} = \frac{1}{\lvert \mathcal{N}(i) \rvert }\sum_{j \in \mathcal{N}(i) \setminus \I_k} X_{j}$ for $i \in \mathcal{I}_k$
\For{$\rho \in \{\rho_1, \rho_2, \cdots, \rho_r\}$}
\State $\mathrm{CVERR}_k(\rho) = \lVert X_{\mathcal{I}_k\cdot} - \wh U^{\rho, k}(\wh V^{t-1})^{\top}  \rVert_F^2$ where
\State $\wh U^{\rho, k} =  \arg \min_{U} \lVert U- X^k \wh V^{t-1}\rVert_F + \rho \lVert \Gamma U\rVert_{21}$ 
\EndFor
\EndFor
\State 4. Choose optimal $\rho$: $\hat{\rho}^{t} = \arg \min_{\rho} \sum_{k} \mathrm{CVERR}_k(\rho)$
\end{algorithmic}
\end{algorithm}

Conventional cross validation techniques sample either nodes or edges to divide the dataset into folds. However, these approaches can disrupt the graph structure and underestimate the strength of the connectivity of the graph. We instead devise a new rule for dividing documents into folds using a minimum spanning tree. This technique is an extension of the cross-validation procedure proposed by \cite{tibshirani2012degrees} for the line graph.  
Given a minimum spanning tree $\mathcal{T}$ of $\mathcal{G}$, we randomly choose a source document $d_s$. For each document $d_i$, we calculate the shortest path distance $\mathrm{d}_{\mathcal{T}}(d_i, d_s)$. Note that this distance is always an integer. We divide the documents into $b$ folds based on the modulus of their distance from the source node: $\mathrm{d}_{T}(d_i, d_s)\mod b$. Through this construction of folds, we can ensure that for every document, at least one of its 1-hop neighbors is in a different fold.

Let $X_{i\cdot}$ be the $i^{th}$ row of $X$. For each leave-out fold $\mathcal{I}_k$, $k\in[b]$, we interpolate the corresponding documents $X_{i\cdot}$$\forall i \in \mathcal{I}_k$, filling the missing document information with the average of corresponding neighbors in $\mathcal{I}_k^{C}$. 
This prevents us from using any information from the leave-out fold in training when calculating the cross-validation error. 

\begin{figure}[h]
    \centering
    \includegraphics[width=0.7\textwidth]{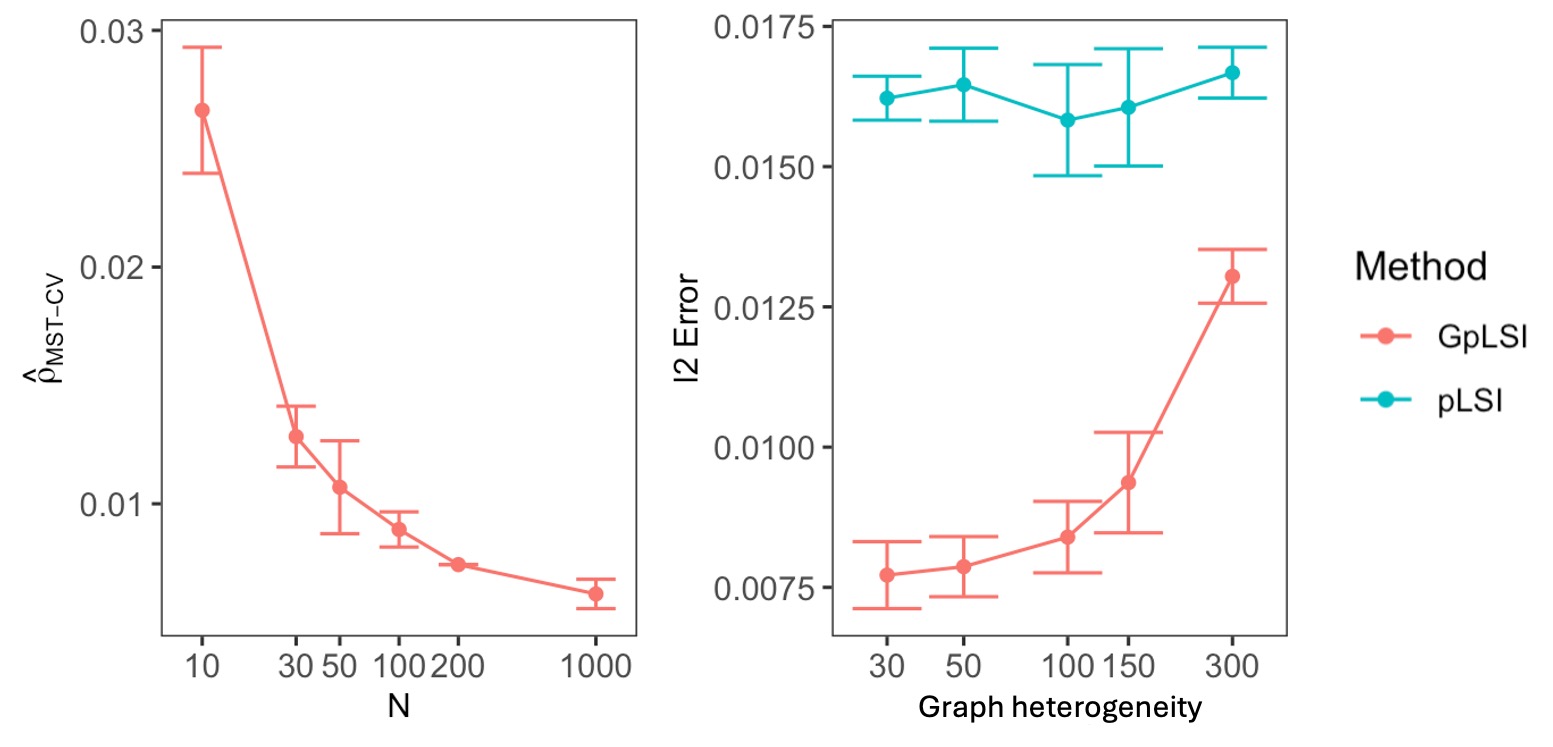}
    \caption{Behavior of $\hat{\rho}_{MST-CV}$ as $N$ increases (left). Behavior of $\ell_2$ error over graph heterogeneity (right). Graph heterogeneity is characterized by $n_{grp}$, the number of patches of documents across the unit square. Each patch is assigned similar topic mixture weights.}
    \label{fig:graph}
\end{figure}

Figure~\ref{fig:graph} demonstrates how GpLSI leverages the graph information. As $N$ increases, GpLSI chooses smaller graph regularization parameter $\hat{\rho}_{MST-CV}$, since the need to share information across neighbors diminishes as documents become longer and more informative. Additionally, when $W$ is more heterogeneous over the graph---meaning neighboring documents exhibit more heterogeneous topic mixture weights---the $\ell_2$ error of $W$ increases. Here, graph heterogeneity is characterized by our simulation parameter $n_{grp}$ (the number of patches that we create). As $n_{grp}$ increases, the unit square is divided into finer patches, and the generated documents within the same topic become more dispersed. Our result indicates that GpLSI works well in settings where the mixture weights are smoother over the document graph and the performance of GpLSI and pLSI become similar as neighboring documents become more heterogeneous.

\section{Analysis of the initialization step}\label{app:theory_init}
\input{appendixA_initialization}

\section{Analysis of iterative graph-aligned denoising}\label{app:theory_algo}
\input{appendixB_iterative}

\section{Auxiliary Lemmas}\label{app:lemmas}

\input{appendixC_technical_lemmas}

\section{Synthetic Experiments}\label{app:synthetic_exp}
\input{appendixD_synthetic_experiments}

\section{Real data}\label{app:experiments}
\input{appendixE_real_data}

\newpage

\bibliography{Reference}
\end{document}
\bibliography{Reference}

\end{document}

%% file: 1.introduction.tex
Consider a corpus of $n$ documents, each composed of words (or more generally, terms) from a vocabulary of size $p$. This corpus can be represented by a \emph{document-term} matrix $D \in \mathbb{N}^{n \times p}$, where each entry $D_{ij}$ denotes the number of times term $j$ appears in document $i$. The objective of topic modeling is to retrieve low-dimensional representations of the data by representing each document as a mixture of latent topics, defined as distributions over term frequencies.

In this setting, each document $D_{i \cdot} \in \mathbb{R}^{p}$ is usually assumed to be sampled from a multinomial distribution with an associated probability vector $M_{i\cdot} \in \mathbb{R}^{p}$ that can be decomposed as a mixture of $K$ topics. In other words, letting $N_i$ denote the length of document $i$:

\begin{equation}\label{eq:generating_mechanism}
    \forall i =1,\cdots, n,   \qquad D_{i\cdot}\sim \text{Multinomial}(N_i, M_{i\cdot}), \qquad \text{with } M_{i\cdot} = \sum_{k=1}^K W_{ik} A_{k\cdot}
\end{equation} 

In the previous equation, $W_{ik}$ corresponds to the proportion of topic $k$ in document $i$, and the vector $W_{i\cdot}$ provides a low-dimensional representation of document $i$ in terms of its topic composition. Each entry $A_{kj}$ of the vector $A_{k\cdot} \in \mathbb{R}^{p}$ corresponds to the expected frequency of word $j$ in topic $k$. Since document lengths $\{N_i\}_{i=1}^n$ are usually treated as nuisance variables, most topic modeling approaches work in fact directly with the word frequency matrix $X = \text{diag}(\{\frac{1}{N_i}\}_{i=1\cdots n}) D$, which can be written in a ``signal + noise'' form as:

\begin{equation}\label{eq:signal_plus_noise}
X  = M + Z = WA + Z.
\end{equation}
Here, $M \in \mathbb{R}^{n \times p}$ is the true signal whose entry $M_{ij}$ denotes the expected frequency of  word $i$ in document $j$, while
$Z = X - M$ denotes some centered multinomial noise. The objective of topic modeling is thus to estimate $ W$ and $ A$ from $X$.

Originally developed to reduce document representations to low-dimensional latent semantic spaces, topic modeling has been successfully deployed for the analysis of count data in a number of applications, ranging from image processing \citep{tu2016topic, zheng2015deep} and image annotation\citep{feng2010topic,shao2009semi}, to biochemistry \citep{reder2021supervised}, genetics \citep{dey2017visualizing,kho2017novel, liu2016overview, yang2019characterizing}, and microbiome studies \citep{sankaran2019latent,symul2023sub}. 

A notable extension of topic modeling occurs when additional document-level data is available. Although original topic modeling approaches rely solely on the analysis of the empirical frequency matrix $X$, this additional information has the potential to significantly improve the estimation of the \emph{word-topic} matrix $A$ and the \emph{document-topic}  mixture matrix $W$, particularly in difficult inference settings, such as when the number of words per document is small. Examples include:
\begin{enumerate}
    \item{\it Analyzing tumor microenvironments:}  In this context, slices of tumor samples are partitioned into smaller regions known as tumor microenvironments, where the frequency of specific immune cell types is recorded \citep{chen2020modeling}. Here, documents correspond to microenvironments, and words to cell types. The objective is to identify communities of co-abundant cells (topics), taken here as proxies for tumor-immune cell interactions and potential predictors of patient outcomes.
    In this setting, we assume that neighboring microenvironments share similar topic proportions. Since these microenvironments are inherently small, leveraging the spatial smoothness of the mixture matrix $W$ can significantly improve inference~\citep{chen2020modeling}. We develop this example in further detail in Section~\ref{sec:experiments:microenvironment}.
    \item{\it Microbiome studies:} Topic models have also been proven to be extremely useful in microbiome analysis \citep{sankaran2019latent,symul2023sub}. In this setting, the data consists of a microbiome count matrix recording the amount of different types of bacteria found in each sample. In this case, microbiome samples are identified to documents, with bacteria playing the roles of the vocabulary, and the goal becomes to identify communities of co-abundant bacteria \citep{sankaran2019latent}. When additional covariate information is available (such as age, gender, and other demographic details), we can expect similar samples (documents) to exhibit similar community compositions (topic proportions).
    \item{\it The analysis of short documents}, such as a collection of scientific abstracts or recipes: In this case, while recipes might be short, information on the origin of the recipe can help determine the topics and mixture matrix more accurately by leveraging the assumption that recipes of neighboring countries will typically share similar topic proportions. We elaborate on this example in greater detail in Section~\ref{sec:experiments:whatscooking}.
\end{enumerate}

\noindent\textbf{Prior works.} Previous attempts to incorporate metadata  in topic estimation have focused on the Bayesian extensions of latent Dirichlet allocation (LDA) model of  \cite{blei2001latent}. 
By and large, these methods typically incorporate the metadata---often in the form of a covariate matrix---within a prior distribution \citep{blei2006correlated, blei2006dynamic, roberts2014structural, mcauliffe2007supervised}. However, these models (a) are difficult to adapt to different types of covariates or information encoded as a graph, and (b) typically lack theoretical guarantees. Recent work by \cite{chen2020modeling} proposes extending LDA to analyze documents with known similarities by smoothing the topic proportion hyperparameters along the edges of the graph. However, this method does not empirically yield spatially smooth structures (see Sections ~\ref{sec:synthetic_experiment} and ~\ref{sec:experiments}), and significantly increases the algorithm's running time. 

In the frequentist realm, probabilistic latent semantic indexing (pLSI) has gained renewed interest over the past five years. Similar to LDA, it effectively models documents as bags of words but differs by treating matrices $A$ and $W$ as fixed parameters.
In particular, new work by \cite{ke2017new} and \cite{klopp2021assigning} suggest procedures to reliably estimate the topic matrix $A$ and mixture matrix $W$, characterizing consistency and efficiency through high-probability error bounds
Although recent work has begun investigating the use of structures in pLSI-based topic models, most approaches have limited this to considering various versions of sparsity \citep{bing2020optimal,wu2023sparse} or weak sparsity \citep{tran2023sparse} for the topic matrix $A$. To the best of our knowledge, no pLSI approach has yet been proposed that effectively leverages similarities between documents nor characterizes the consistency of these estimators.

\subsection*{Contributions} In this paper, we propose the first pLSI method that can be made amenable to the inclusion of additional information on the similarity between documents, as encoded by a known graph. More specifically:

\begin{itemize}
    \item[a. ] We propose a scalable algorithm based on a graph-aligned singular decomposition of the empirical frequency matrix $X$ to provide estimates of $W$ and $A$ (Section~\ref{sec:model}). Additionally, we develop a new cross-validation procedure for our graph-based penalty that allows us to choose the optimal regularization parameter adaptively (Section A of the Appendix).
    \item[b. ] We prove the benefits of the graph alignment procedure by deriving high probability upper bounds for both $W$ and $A$ in Section~\ref{sec:theory}, which we verify through extensive simulations in Section~\ref{sec:synthetic_experiment}.
    \item[c. ] Finally, we showcase the advantage of our method over LDA-based methods and non-structured pLSI techniques on three real-world datasets: two spatial transcriptomics examples and a recipe dataset in Section~\ref{sec:experiments}.
\end{itemize}

\subsection*{Notations}
Throughout this paper, we use the following notations. For any $t\in \mathbb{Z}_+$, $[t]$ denotes the set $\{1, 2, ..., t\}$. 
For any $a,b\in \R$, we write 
$a\vee b = \max(a,b)$ and ${a\wedge b = \min(a,b)}$. 
We use $\mathbf{1}_d \in \R^d$ to denote the vector with all entries equal to 1 and $\mathbf{e}_k\in \R^d$ to denote the vector with $k^{th}$ element set to 1 and 0 otherwise.
For any vector $u$, its $\ell_2$, $\ell_1$ and $\ell_0$  norms are defined respectively as $\|u\|_2=\sqrt{\sum_i u_i^2}$, $\|u\|_1=\sum_i|u_i|$, and  $\|u\|_0=\sum_i\mathbf{1}\{u_i\neq 0\}$. Let $I_m$ denote the $m \times m$ identity matrix. For any matrix ${A=(a_{ij})\in \R^{n\times p}}$, $A(i, j)$ denote the $(i,j)$-entry of $A$, $A_{i \cdot}$ and $A_{\cdot j}$ denote the $i^\text{th}$ row and $j^{th}$ column of $A$ respectively. Throughout this paper, $\lambda_i(A)$ stands for the $i^{th}$ largest singular value of the matrix $A$ with ${\lambda_{\max}(A)=\lambda_1(A)}$, ${\lambda_{\min}(A)=\lambda_{p\wedge \text{rank}(A)}(A)}$. We also denote as $U_K(A)$ and $V_K(A)$  the left and right singular matrix from the rank-$K$ SVD of $A$. The Frobenius, entry-wise $\ell_1$ norm and the operator norms of $A$ are denoted as $\fnorm{A}=\sqrt{\sum_{i,j}a_{ij}^2}$, $\lVert A \rVert_{11} = \sum_{i,j} | a_{ij} |$, and ${\opnorm{A}=\lambda_1(A)}$, respectively. The $\ell_{21}$ norm is denoted as  $\big\|A\big\|_{21}=\sum_{i}\|A_{i\cdot}\|_2$.
For any positive semi-definite matrix $A$, $A^{1/2}$ denotes its principal square root that is positive semi-definite and satisfies $A^{1/2} A^{1/2} = A$. 
The trace inner product of two matrices $A,B\in\mathbb{R}^{n \times p}$ is denoted by $\iprod{A}{B}=\Tr(A^\top B)$. $A^{\dagger}$ denotes the pseudo-inverse of the matrix $A$ and $P_A$ denotes the projection matrix onto the subspace spanned by columns of $A$.

%% file: 2A.model.tex
In this section, we introduce graph-aligned pLSI (GpLSI), an extension of the standard pLSI framework that leverages known similarities between documents to improve inference in topic modeling using a graph formalism. We begin by introducing a set of additional notations and model assumptions, before introducing the algorithm in Section~\ref{sec:two-step}.

\subsection{Assumptions}

Let $\mathcal{G} = (\mathcal{D}, \calE)$ denote an undirected graph induced by a known adjacency matrix on the $n$ documents in the corpus. The documents are represented as nodes $\mathcal{D}$, and $\calE$  denotes the edge set of size $\lvert \calE \rvert = m$. Throughout this paper, for simplicity, we will assume that $\mathcal{G}$ is binary, but our approach---as discussed in Section~\ref{sec:two-step}---can be in principle extended to weighted graphs.
We denote the graph's incidence matrix as $\Gamma \in \mathbb{R}^{m \times n}$ where, for any edge $e=(i,j), i<j$ between nodes $i$
and $j$ in the graph, $  \Gamma_{ei}=1$,  $  \Gamma_{ej}=-1$ and $\Gamma_{ek}=0$ for any $k \neq i,j.$
It is easy to show that the Laplacian of the graph can be expressed in terms of the incidence matrix as $L = \Gamma^{\top}\Gamma$ \citep{hutter2016optimal}. 
Let $\Gamma^{\dagger}$ be the pseudo-inverse of $\Gamma$,  and denote by $\mathbf{s}_i, i=1\cdots m$ its columns, so that  $\Gamma^{\dagger} = [\mathbf{s}_1, \cdots, \mathbf{s}_m]$. We also define the \emph{inverse scaling factor} of the incidence matrix $\Gamma$ \citep{hutter2016optimal}, a quantity necessary for assessing the performance of GpLSI:

\begin{equation} \label{eq:inv_scaling}
\rho(\Gamma) \coloneqq \max_{l \in [m]} \norm{\mathbf{s}_l}_2
\end{equation}

We focus on the estimation of the topic mixture matrix under the assumption that neighboring documents have similar topic mixture proportions: $W_{i\cdot} \approx W_{ j\cdot}$ if $i\sim  j$. This implies that the rows of $W$ are assumed to be smooth with respect to the graph $\mathcal{G}$. Noting that the difference of mixture proportions between neighboring nodes $i$ and $j$ ($e=(i,j) \in \mathcal{E}$) can be written as $(\Gamma W)_{d\cdot} = W_{i\cdot}-W_{ j\cdot}$, this smoothness assumption effectively implies sparsity on the rows of the matrix $\Gamma W$.

   \begin{assumption}[Graph-Aligned mixture matrix] 
  \label{assumption:smoothness}
The support (i.e, the number of non-zero rows) of the difference matrix $\Gamma W = \left( W_{i \cdot} - W_{j \cdot}\right)_{(i,j)\in \mathcal{E}}$ is small:
\begin{equation}\label{eq:s}
    |\text{supp}(\Gamma W)| \leq s,
\end{equation}
where $s \ll |\mathcal{E}|,n.$
  \end{assumption}

The previous assumption is akin to assuming that the underlying mixture matrix $W$ is piecewise-continuous with respect to the graph $\mathcal{G}$, or more generally, that it can be well approximated by a piecewise-continuous function.

Our setting is not limited to connected document graphs. Denote $n_{\mathcal{C}}$ the number of connected subgraphs of $\mathcal{G}$ and $n_{\mathcal{C}_l}$ the number of nodes in the $l^{th}$ connected subgraph. Let $n_{\mathcal{C}_{\min}} $ be the size of the smallest connected component:
\begin{equation}\label{eq:n_c}
    n_{\mathcal{C}_{\min}} = \min_{l\in [n_{\mathcal{C}}]}n_{\mathcal{C}_l}
\end{equation}

The error bound of our estimators will depend on both $n_{\mathcal{C}}$ and $n_{\mathcal{C}_{\min}}$. In the rest of this paper, we will assume that all connected components have roughly the same size: $ n_{\C_1} \asymp \dots \asymp n_{\C_{n_{\mathcal{C}}}} $.

   \begin{assumption}[Anchor document] 
  \label{assumption:anchordoc}
    For each topic \(k = 1, \dots, \ntopics,\) there exists at least one document \(i\) (called an anchor document) such that \(W_{ik} = 1\) and \(W_{ik^{'}} = 0\) for all \(k^{'} \not= k\).
  \end{assumption}

\begin{remark}
Assumption 2 is standard in the topic modeling community, as it is a sufficient condition for the identifiability of the topic and mixture matrices $A$ and  $W$ \citep{donoho2003does}. Beyond identifiability, we contend that the ``anchor document" assumption functions not only as a sufficient condition for identifiability but also as a necessary condition for interpretability: Topics are interpretable only when associated with archetypes---that is, ``extreme`` representations (in our case, anchor documents)---that illustrate the topic more expressively.
\end{remark}

\begin{assumption}[Equal Document Sizes] 
  \label{assumption:N}
   In this paper, for ease of presentation, we will also assume that the documents have equal sizes: $N_1 = \dots = N_n = N$. More generally, our results also hold if we assume that the document lengths satisfy $\max_{i \in [n]} N_i \leq C^* \min_{i\in [n]} N_i$ ($N_1 \asymp \dots \asymp N_n$), in which case $N = \frac{1}{n}\sum_{i=1}^n N_i$ denotes the average document length.
  \end{assumption}

  \begin{assumption}[Condition number of $M$ and $W$] \label{assumption:eigenvalues}
    There exist two constants $c, c^* > 0$  
    such that
    \begin{equation*}
      \lambda_{\ntopics}(M) \geq c \lambda_{1}(W) \quad
      \text{and} 
      \quad \max\left\{\kappa(M), \kappa(W)\right\} \le c^*.
    \end{equation*}
  \end{assumption}

\begin{assumption}[Assumption on the minimum word frequency]\label{assumption:h_j}    
We assume that the expected word frequencies $h_j$ defined as:
$ \forall j \in [p], h_j = \sum_{k=1}^K A_{kj}$ are bounded below by:
$$ \min_{j \in [p]} h_j \geq c_{\min} \frac{\log(n)}{N}$$
where $c_{\min}$ is a constant that does not depend on parameters $n,p,N$ or $K$.
\end{assumption}

\begin{remark}\label {remark:small_p}
    Assumption~\ref{assumption:h_j} is a relatively strong assumption that essentially restricts the scope of this paper's analysis to small vocabulary sizes (thereafter referred to as the ``low-$p$'' regime). Indeed, since $\sum_{j=1}^p \sum_{k=1}^K A_{kj} = K$, under Assumption~\ref{assumption:h_j}, it immediately follows that:
    $$ p c_{\min} \frac{\log(n)}{N} \leq K \implies p \leq \frac{K N}{\log(n)c_{\min}}$$
    This assumption might not reflect the large vocabulary sizes found in many practical problems, where we could expect $p$ to grow with $n$. A solution to this potential limitation is to assume weak sparsity on the matrix $A$ and to threshold away rare terms using the thresholding procedure proposed in \cite{tran2023sparse}, selecting a subset $J$ of words with large enough frequency. In this case, the rest of our analysis should follow, replacing simply the data matrix $X$ by its subset, $X_{\cdot J}.$
\end{remark}

\subsection{Estimation procedure: pLSI}\label{sec:oracle}

Since the smoothness assumption (Assumption~\ref{assumption:smoothness}) pertains to the rows of the document-topic mixture matrix $W$, we build on the pLSI algorithm developed by \cite{klopp2021assigning}. 
In this subsection, we provide a brief overview of their method.

When we assume we directly observe the true expected frequency matrix $M$ defined in Equation \eqref{eq:signal_plus_noise}, \cite{klopp2021assigning} propose a fast and simple procedure to recover the mixture matrix $W$. Specifically, let $U\in \mathbb{R}^{n \times K}$ and $V\in \mathbb{R}^{p \times K}$ be the left and right singular vectors obtained from the singular value decomposition (SVD) of the true signal $M\in \mathbb{R}^{n \times p}$, so that $M = U \Lambda V^T$. A critical insight from their work is that $U$ can be decomposed as:
\begin{equation}\label{eq:simplex} U = W H, \end{equation} where $H$ is a full-rank, $K \times K$-dimensional matrix. From this decomposition, it follows that the rows of $U$, which can be viewed as $K$-dimensional embeddings of the documents, lie on the $K$-dimensional simplex $\Delta_{K-1}$. The simplex’s vertices, represented by the rows of $H$, correspond to the anchor documents (Assumption~\ref{assumption:anchordoc}). Thus, identifying these vertices through any standard vertex-finding algorithm applied to the rows of $U$ will enable the estimation of $W$. The procedure of \cite{klopp2021assigning} can be  summarized as follows:
\begin{description}
    \item[Step 1:]   Compute the singular value decomposition  (SVD) of the matrix $M$, reduced to rank $K$, to obtain: $M = U\Lambda V^T$.
    
    \item[Step 2: Vertex-Hunting Step:]
    Apply the successive projection algorithm (SPA) \citep{araujo2001successive} (a vertex-hunting algorithm) on the rows of $U$. 
    This algorithm returns the indices of the selected ``anchor documents,'' $J \subseteq [n]$ with $|J| = K$. Define $\widehat{H} = U_{J\cdot}$, where each row corresponds to one of the $K$ vertices of the simplex $\Delta_{K-1}$.
    \item[Step 3: Recovery of $W$:] $W$ can simply be recovered from $U$ and $\widehat{H}$ as
    \begin{equation}\label{eq:H_hat}
        \wh W ={U} \wh H^{-1}.
    \end{equation}
    \item[Step 4: Recovery of $A$: ] Finally, $A$ can subsequently be estimated as $\wh A = \wh H \Lambda V^{\top}.$
\end{description}

In the noisy setting, 
the procedure is adapted by plugging the observed frequency $X$ instead of $M$ in Step 1 and getting estimates of the singular vectors: $X=\wh U \wh \Lambda \wh V^{\top}$. 
Under a similar set of assumptions as ours (Assumptions~\ref{assumption:anchordoc}-\ref{assumption:eigenvalues}), Theorem 1 of \citep{klopp2021assigning} states that the error of $\wh W$ is such that: $
     \min_{P\in \mathcal{P}}\|\widehat{W} - W P\|_{F}
     \le
     {\rm C_0}  K \sqrt{n \log(n+ p)/N}$,
   where $\mathcal{P}$ denotes the set of all permutation matrices. Their analysis provides one of the best error bounds on the estimation of the topic mixture matrix $W$ for pLSI. 

However, this approach has two key limitations. First, the consistency of their estimator relies on having a sufficiently large number of words per document, $N$. In particular, a necessary condition for the aforementioned results to hold is that $N  \ge C K^5    \log(n + p)$. The authors establish minimax error bounds, showing that the rate of any estimator’s error for $W$ is bounded below by a term on the order of $O(\sqrt{n/N})$ (Theorem 3, \cite{klopp2021assigning}). In other words, the accurate estimation of $W$ requires that each document contains enough words. In many practical scenarios---such as the tumor microenvironment example mentioned earlier---this condition may not hold. However, we might still have access to additional information indicating that certain documents are more similar to one another.

Second, the method is relatively rigid and does not easily accommodate additional structural information, such as document-level similarities. Indeed, this method does not rely on a convex optimization formulation to which we could simply add a regularization term, and the vertex-hunting algorithm does not readily incorporate metadata of the documents.

\subsection{Estimation procedure: GpLSI}\label{sec:two-step}

Theoretical insights from \cite{klopp2021assigning} help explain why topic modeling deteriorates in low-$N$ regimes. When the number of words per document is too small, the observed frequency matrix $X$ can be viewed as a highly noisy approximation of $M$, causing the estimated singular vectors $\widehat{U}$ to deviate significantly from the true underlying point cloud $U$. To mitigate this issue, \cite{klopp2021assigning} suggest a preconditioning step that improves the estimation of the singular vectors in noisy settings.

In this paper, we take a different approach by exploiting the graph structure associated with the documents to reduce the noise in $X$. Rather than preconditioning the empirical frequency matrix, we propose an additional denoising step that leverages the graph structure to produce more accurate estimates of $U$, $V$, and $\Lambda$. Specifically, we modify the SVD of $X$ in Step 1 and  estimation of topic matrix $A$ in Step 4 described in Section~\ref{sec:oracle}.


\begin{description}
    \item[Step 1: Iterative Graph-Aligned SVD of $X$:]  We replace Step 1 of Section~\ref{sec:oracle} with a graph-aligned SVD of the empirical word-frequency matrix $X$. More specifically, in the graph-aligned setting, we assume that the underlying frequency matrix $M$ belongs to the set:
    \begin{equation}\label{eq:param_space}
    \begin{split}
        \mathcal{F}(n,p,K,s) = 
        \{
        & M = U\Lambda V^{\top} \in \mathbb{R}^{n \times p}, \ \mathrm{rank}(M)=K:\\
        &U \in \mathbb{R}^{n \times K}, V \in \mathbb{R}^{p \times K}, \Lambda = \mathrm{diag}(\lambda_1, \lambda_2, \cdots, \lambda_K), \\
        &|\mathrm{supp}(\Gamma U)| \leq s, \lambda_K>0
        \}.
    \end{split}
    \end{equation}
     Throughout the paper, we shall allow $s, p, N,$ and $n$ to vary. We will assume the number of topics $K$ to be fixed. 
    \item[Step 2, 3] Same as Step 2,3 in Section~\ref{sec:oracle}.
    \item[Step 4: Recovery of $A$: ] $A$ can subsequently be estimated by solving a constrained regression problem of $X$ on $\wh W$:
    \begin{equation}
    \label{eq:stepA}
    \begin{split}
        \widehat{A}  &= \text{argmin}_{A\in \R^{K \times p}} \| X - \wh W A\|_F^2 \\
        &\text{ such that } \forall k \in [K], \ \sum_{j=1}^p A_{kj}=1, \ A_{kj}\geq0
        \end{split}
    \end{equation}
\end{description}

%% file: 2B.algorithm.tex
\paragraph{ Iterative Graph-Aligned SVD.}
We propose a power iteration method for Step 1 that iteratively updates the left and right singular vectors while constraining the left singular vector to be aligned with the graph (Algorithm~\ref{algo:1}). A similar approach has already been studied under Gaussian noise in \cite{yang2016rate} where sparsity constraints were imposed on the left and right singular vectors. 

Drawing inspiration from \cite{yang2016rate} and adapting this method to the multinomial noise setting, Algorithm~\ref{algo:1}  iterates between three steps. The first consists of denoising the left singular subspace by leveraging the graph-smoothness assumption (Assumption~\ref{assumption:smoothness}). At iteration $t$, we solve:
    \begin{equation}
    \label{eq:opt}
        \bar U ^{t} = \arg \min_{U \in \R^{n \times K}} \lVert U- X \wh V^{t-1} \rVert_F^2 + {\hat{\rho}}^{t} \lVert \Gamma U \rVert_{21}
    \end{equation}

Here, $\bar{U}^t$ is a denoised version of the projection $X\wh V^{t-1}$ that leverages the graph structure to yield an estimate closer to the true $U$. We then take a rank-$K$ SVD of $\bar{U}^t$ to extract $\wh U^t$ (an estimate of $U$) with orthogonal columns.

\begin{algorithm}[t]
\setstretch{1.35}
\caption{Two-way Iterative Graph-Aligned SVD} \label{algo:1}
\begin{algorithmic}[1]
\State \textbf{Input:} Observation $X$, initial matrix $\wh V^0$, incidence matrix $\Gamma$, number of topics $K$, tolerance level $\epsilon$
\State \textbf{Output:} Denoised singular vectors $\wh U$, $\wh V$
\Repeat
\State 1. Graph denoising of $U$ : $\bar U ^{t} = \arg \min_{U\in \R^{n \times K}} \lVert U- X \wh V^{t-1} \rVert_F^2 + \hat{\rho}^{t} \lVert \Gamma U \rVert_{21}$
\State 2.  SVD of $\bar U^t$: $\wh U^t \leftarrow \text{Left Singular Vectors in } SVD_K(\bar{U}^t)$
\State 3. SVD of $\wh V^t$: 
$\wh V^t \leftarrow \text{Left Singular Vectors in } SVD_K(X^{\top}\wh U^t)$
\State 4. Calculate the score $s = \lVert \wh P_u^{t} X\wh P_v^{t}-\wh P_u^{t-1}X\wh P_v^{t-1}\rVert$, $\wh P_u = \wh U^{t}(\wh U^{t})^{\top}$, $\widehat{P}_v = \wh V^{t}(\wh V^{t})^{\top}$
\Until $s < \epsilon$ 
\end{algorithmic}
\end{algorithm}

Finally, we update $\wh V^{t}$. Since we are not assuming any particular structure on the right singular subspace, we simply apply a rank-$K$ SVD on the projection $X^{\top} \wh U^t$. Denoting the projections onto the columns of the estimates as $P_u^{t}=\wh U^{t}(\wh U^{t})^{\top}$ and $P_v^{t}=\wh V^{t}(\wh V^{t})^{\top}$, we iterate the procedure until $\|P_u^{t}XP_v^{t}-P_u^{t-1}XP_v^{t-1}\|_F \leq \epsilon$ for a fixed threshold $\epsilon$.

Denoting the final estimates after $t_{\max}$ iterations as $\wh U, \wh V$, these estimates can then be plugged into Steps 2-4 to estimate $W$ and $A$. The improved estimation of $\wh U, \wh V$ can be shown to translate into a more accurate estimation of the matrices $W$ and $A$ (Theorems~\ref{theorem:err_W} and \ref{theorem:err_A} presented in the next section). 

Although our theoretical results depend on choosing an appropriate level of regularization $\hat{\rho}^t$, the theoretical value of $\hat{\rho}^t$ depends on unknown graph quantities. In practice, therefore, the optimal $\hat{\rho}^t$ must be chosen in each iteration using cross-validation. We devise a novel graph cross-validation method which effectively finds the optimal graph regularization parameter by partitioning nodes into folds using a minimum spanning tree. We defer the detailed procedure to Section A of the Appendix.

 \begin{remark}
     While in the rest of the paper, we typically assume that the graph is binary,  our method is in principle generalizable to a weighted graph $\mathcal{G} = (\mathcal{D}, \mathcal{W})$ where $\mathcal{W}$ represents the weighted edge set.
In this case, we denote weighted incidence matrix as $\Tilde{\Gamma} = \mathrm{T}\Gamma$ where $\mathrm{T}\in \mathbb{R}^{m \times m}$ is a diagonal matrix with entry $t_{dd}$ corresponding to the weight of the $d^{th}$ edge. We note that scaling $\Gamma$ with $\mathrm{T}$ does not change the projection onto the row space of $\Gamma$, thus preserving our theoretical results. Without loss of generality, we work with an unweighted incidence matrix $\Gamma$. 
 \end{remark}

\begin{remark}
    The penalty $\|\Gamma U\|_{21}$ is known as the total-variation penalty in the computer vision literature. As noted in \cite{hutter2016optimal}, this type of penalty is usually a good idea whenever the rows of $W$ take similar values, or may at least be well approximated by piecewise-constant functions. In the implementation of our algorithm, we employ the solver of developed by \cite{sun2021convex}, as it is the most efficient algorithm available for this type of problem. 
\end{remark}

\paragraph{Initialization.}\label{para:init}
The success of the procedure heavily depends on having access to good initial values for $V$. Since, as highlighted in Remark~\ref{remark:small_p}, this manuscript assumes a ``low-$p$'' regime, we propose to simply take the rank-$K$ eigendecomposition of the matrix $X^{\top}X - \frac{n}{N} \hat{D}_0$ to obtain an initial estimate $\wh V^0$:
\begin{equation}
    \wh V^0 = U_K({X^{\top}X} - \frac{n}{N} \widehat{D}_0 )
\end{equation}
where $\wh D_0$ is a diagonal matrix where each entry is defined as: $(\wh D_0)_{jj} = \frac{1}{n}\sum_{i=1}^n X_{ij},$  and
where $U_K({X^{\top}X}- \frac{n}{N} \hat{D}_0)$ denotes the matrix of $K$ leading eigenvectors of ${X^{\top}X- \frac{n}{N} \widehat{D}_0}$.

\begin{theorem}\label{theorem:initXtx}
 Suppose $\max(K,p) \leq n$ and $\sqrt{K} \leq p$. 
    Under Assumptions~\ref{assumption:smoothness} to \ref{assumption:h_j}, the eigenvectors of the matrix ${X^{\top}X} -\frac{n}{N} \hat{D}_0$ provide a reasonable approximation to the right singular vectors, in that with probability at least $1-o(n^{-1})$:
    \begin{equation*}
    \begin{split}
          \|\sin \Theta(V, \wh V^0)\|_{op}\leq       \|\sin \Theta(V, \wh V^0)\|_F
          \leq  \frac{C}{\lambda_K(M)^2} K\sqrt{\frac{n\log(n)}{N}}
        &\leq C^* K^2\sqrt{\frac{\log(n)}{nN}}
    \end{split}
\end{equation*}
for some constants $C$ and $C^*>0$.
\end{theorem}

The proof of the theorem is provided in Section B of the Appendix.

%% file: 3.theory.tex
In this section, we provide high-probability bounds on the Frobenius norm of the errors for $\wh W$ and $\wh A$. We begin by characterizing the effect of the denoising on the estimates of the singular values of $X$, before showing how the improved estimation of the singular vectors translates into improved error bounds for both $W$
and $A$.

\subsection{Denoising the singular vectors}

The improvement in the estimation of the singular vectors induced by our iterative denoising procedure is characterized in the following theorem.
\begin{theorem}\label{theorem:err_GSVD}
Let Assumption~\ref{assumption:smoothness} to \ref{assumption:h_j} hold. Assume $\max(K,p) \leq n$ and $\sqrt{K} \leq p$. Denote $\wh U, \wh V$ as outputs of Algorithm~\ref{algo:1} after $t_{\max}$ iterations. 
If $N$ satisfies 
    \begin{equation}\label{cond:N}
    N \geq c^*_{\min}\left(K^4\frac{\log(n)}{n}\left(n_{\mathcal{C}}+\rho^2(\Gamma) s\lambda_{\max}(\Gamma)\right) \vee \frac{\log(n)}{n_{\mathcal{C}_{\min}}}\right),
    \end{equation}
there exists a constant $C_0>0$ such that with probability at least $1-o(n^{-1})$,
  \begin{equation}\label{eq:SVD_error}
        \max \{\inf_{O \in \mathbb{O}_K}\|\wh U - UO\|_F, \inf_{O \in \mathbb{O}_K}\|\wh V - VO\|_F\} \leq C_0 K\sqrt{\frac{\log(n)}{nN}} \left(\sqrt{n_{\mathcal{C}}}+\rho(\Gamma)\sqrt{s\lambda_{\max}(\Gamma)}\right)
    \end{equation}
\end{theorem}

The proof of this result is provided in Section C of the Appendix.

\begin{remark}
    This result is to be compared against the rates of the estimates obtained without any regularization. In this case, the results of \cite{klopp2021assigning}  show that with probability at least $1-o((n+p)^{-1})$, the error in \eqref{eq:SVD_error} is of the order of $O(K\sqrt{ \log(n)/{N}})$.
Both rates thus share a factor $K \sqrt{\log(n)/N}$. However, the spatial regularization in our setting allows us to introduce an additional factor of the order of $\frac{1}{\sqrt{n}} (\sqrt{n_{\mathcal{C}}}+\rho(\Gamma)\sqrt{s\lambda_{\max}(\Gamma)}) $. The numerator in this expression can be interpreted as the effective degrees of freedom of the graph, and for disconnected graphs ($\lambda_{\max}(\Gamma)=0, n_{\mathcal{C}}=n$), the results are identical. However, for other graph topologies (e.g. the 2D grid, for which $\lambda_{\max}(\Gamma)$ is bounded by a constant, $\rho \lesssim \log(n)$ (see \cite{hutter2016optimal}) and $n_{\mathcal{C}}=1$), our estimator can considerably improve the estimation of the singular vectors (see Section~\ref{sec:special}).
\end{remark}

\subsection{Estimation of $W$ and $A$}

We now show how our denoised singular vectors can be used to improve the estimation of the mixture matrix $W.$
\begin{theorem}\label{theorem:err_W}
\vspace{0.1cm}

    Let Assumptions~\ref{assumption:smoothness} to \ref{assumption:h_j} hold. Let $\rho(\Gamma),s,n_{\mathcal{C}}$, and $n_{\mathcal{C}_{\min}}$ be given as \eqref{eq:inv_scaling}-\eqref{eq:n_c}. Assume $\max(K,p) \leq n$ and $\sqrt{K} \leq p$. Let $\widehat{W}$ denote the estimator of the mixture matrix (Equation~\ref{eq:signal_plus_noise}) obtained by running the SPA algorithm on the denoised estimates of the singular vectors (Algorithm~\ref{algo:1}). If $N$ satisfies the condition stated in \eqref{cond:N}, 
    then there exists a constant $C>0$ such that with probability at least $1-o(n^{-1})$,
    \begin{equation}\label{eq:main}
        \min_{P \in \mathcal{P}}\|\wh W-WP\|_F \leq CK\sqrt{\frac{\log(n)}{N}}\left ( \sqrt{n_{\mathcal{C}}}+\rho(\Gamma) \sqrt{s\lambda_{\max}(\Gamma)}\right )
    \end{equation}
    where $\mathcal{P}$ denotes the set of all permutations.
\end{theorem}
Theorem~\ref{theorem:err_W} shows that $\wh W$ is highly accurate as long as the document lengths are large enough, as defined by $N\gtrsim ({n_\mathcal{C}} +\rho^2(\Gamma){s \lambda_{\max}(\Gamma) })\log(n)/n$. This requirement is more relaxed than the condition $N \gtrsim \log(n+p)$ for pLSI provisioned in Theorem 1 and Corollary 1 of \cite{klopp2021assigning}. This indicates that GpLSI can produce accurate estimates even for smaller $N$, by sharing information among neighboring documents. The shrinkage of error due to graph-alignment is characterized by the term $\frac{1}{\sqrt{n}}(\sqrt{n_{\mathcal{C}}}+\rho(\Gamma) \sqrt{s\lambda_{\max}(\Gamma)})$, which is equal to one when the graph $\mathcal{G}$ is empty. In general, the effect of the regularization depends on the graph topology. \cite{hutter2016optimal} show in fact that the inverse scaling factor verifies: $\rho(\Gamma) \leq \sqrt{2}/\sqrt{\lambda_{n-1}(L)}$. The quantity $\lambda_{n-1}(L)$, also known as the \emph{algebraic connectivity}, provides insights into the properties of the graph, such as its connectivity.  Intuitively, higher values of $\lambda_{n-1}(L)$ reflect more tightly connected graphs \citep{chung1997spectral}, thereby reducing the effective degrees of freedom induced by the graph-total variation penalty. By contrast, $\lambda_{\max}(\Gamma)$ can be coarsely bounded using the maximum degree $d_{\max}$ of the graph:  $\lambda_{\max}(\Gamma) \leq \sqrt{2d_{\max} }$ \citep{anderson1985eigenvalues, zhang2011laplacian}. Consequently, we can expect our procedure to work well on well-connected graphs with bounded degree. Examples include for instance grid-graphs, $k$-nearest neighbor graphs, or spatial (or planar) graphs.  We provide a more detailed discussion and  more explicit bounds for specific graph topologies in Section~\ref{sec:special}.

Furthermore, using the inequality $\|\wh W-WP\|_{11} \leq \sqrt{Kn}\|\wh W-WP\|_{F}$, it immediately follows that:

\begin{corollary}
    Let the conditions of Theorem~\ref{theorem:err_W} hold. If $N$ satisfies the condition stated in \eqref{cond:N}, then there exists a constant $C>0$ such that with probability at least $1-o(n^{-1})$,
    \begin{equation}\label{eq:l1_W}
         \min_{P \in \mathcal{P}}\|\wh W-WP\|_{11} \leq CK^{3/2}\sqrt{\frac{n\log(n)}{N}}\left ( \sqrt{n_{\mathcal{C}}}+\rho(\Gamma) \sqrt{s\lambda_{\max}(\Gamma)}\right ).
    \end{equation}
    where $\mathcal{P}$ denotes the set of all permutations.
\end{corollary}
 Finally, we characterize the error bound of $\wh A$. The full proofs of Theorems 3 and 4 are deferred to Section D of the Appendix.
 
 \begin{theorem}\label{theorem:err_A}
     Let the conditions of Theorem~\ref{theorem:err_W} hold. If $N$ satisfies the condition stated in \eqref{cond:N}, then there exists a constant $C>0$ such that with probability at least $1-o(n^{-1})$,
     \begin{equation}\label{eq:A}
         \|\wh A-\tilde{P}A\|_F \leq CK^{3/2}\sqrt{\frac{\log(n)}{N}} \left(\sqrt{n_{\mathcal{C}}}+\rho(\Gamma)\sqrt{s\lambda_{\max}(\Gamma)}\right).
     \end{equation}
     where, denoting $P$ the permutation matrix that minimises the distance between $\widehat{W}$ and $W$ in \eqref{eq:main}, we take $\tilde{P}$ to be its inverse: $\tilde{P}=P^{-1}$.
 \end{theorem}

\begin{remark}\label{remark:A}
    The previous error bound of $A$ indicates that the accuracy of $\wh A$ primarily relies on the accuracy of $\wh W$, which is to be expected, since  $A$ is estimated by regressing $X$ on the estimator $\wh W$. While the error rate may not achieve the minimax-optimal rate $C\sqrt{p/(nN)
    }$ derived in \cite{ke2017new}, we found that this procedure is more accurate than the estimator $\wh A$ proposed in \cite{klopp2021assigning}, as confirmed by synthetic experiments in Section~\ref{sec:synthetic_experiment}.
\end{remark}

\subsection{Refinements for special graph structures}\label{sec:special}
We now analyze the behavior of the error bound provided in Theorem \ref{theorem:err_W} for different graph structures, further expliciting the dependency of our bounds on graph properties.

\paragraph{Erd{\"o}s-R{\'e}nyi random graphs.}
We first assume that the graph $\mathcal{G}$ is an Erd{\"o}s-R{\'e}nyi random graph where each pair of nodes is connected with probability $p=p_0\frac{\log (n)}{n}$ for a constant $p_0>1$.  In this case, \cite{hutter2016optimal} show that with high probability, the algebraic connectivity $\rho(\Gamma)$ is of order $O(\frac{1}{\log (n)})$. Moreover, the maximal degree is of order $\log(n)$ and the graph is almost surely connected \citep{van2024random}. Under this setting, the error associated to our estimator $\widehat{W}$ becomes:

\begin{equation}\label{eq:main2}
        \min_{P \in \mathcal{P}}\|\wh W-WP\|_F \leq C_1K\sqrt{\frac{\log(n)}{N}}\left(1+ s^{\frac12} \log(n)^{-\frac34}\right).
\end{equation}


\paragraph{Grid graphs.} We also derive error bounds for grid graphs, which are commonly occurrences in the analysis of spatial data and for applications in image processing:
\begin{description}[leftmargin=1cm]

\item [2D grid graph:] Let $\mathcal{G}$ be a 2D grid graph on $n $ vertices.
\cite{hutter2016optimal} show that, in that case, the inverse scaling factor is such that $\rho(\Gamma)\lesssim \sqrt{\log (n)}$ . The error of our estimator thus becomes:
\begin{equation}\label{eq:main_2d}
    \min_{P \in \mathcal{P}}\|\wh W-WP\|_F \leq CK\sqrt{\frac{\log(n)}{N}}\left(1+\sqrt{s\log(n)}\right)\leq C_3K\log(n)\sqrt{\frac{s}{N}}.
\end{equation}
\item[$K$-grid graph, $k \geq 3$:] In this case, \cite{hutter2016optimal} show that the inverse scaling factor is bounded by a constant $C(k)$, that depends on the dimension $k$ but is independent of $n$. In this case, the error of our estimator  becomes:
\begin{equation}\label{eq:main_dd}
    \min_{P \in \mathcal{P}}\|\wh W-WP\|_F \leq CK\sqrt{\frac{\log(n)}{N}}\left(1+\sqrt{s}\right)\leq C_3K\sqrt{\frac{s\log(n)}{N}}.
\end{equation}
\end{description}

%% file: 4.synthetic_experiments.tex
We evaluate the performance of our method using synthetic datasets where $W$ is aligned with respect to a known graph. 

\paragraph{ Experimental Protocol} To generate documents, we sample $n$ points uniformly on unit square $[0,1]^2$, and cluster them into ${n_{grp}}=30$ groups using a simple k-means algorithm. For each group, we generate the topic mixture as $\mathbf{\alpha} \sim \mathrm{Dirichlet}(\mathbf{u})$ where $u_k \sim \mathrm{Unif}(0.1, 0.5)$ ($k \in [K]$). Small random noise  $N(0, 0.03)$  is added to $\mathbf{\alpha}$ for each document in the group, and we permute it so that the biggest element of $\alpha$ is assigned to the group's predominant topic. $A$ is generated by sampling each entry $A_{kj}$ from  a uniform distribution, normalizing each row to ensure that $A$ is a stochastic matrix. A detailed description of the data generating process is provided in Section F of the Appendix. 

\begin{figure}[h]
    \centering
    \includegraphics[width=0.9\textwidth]{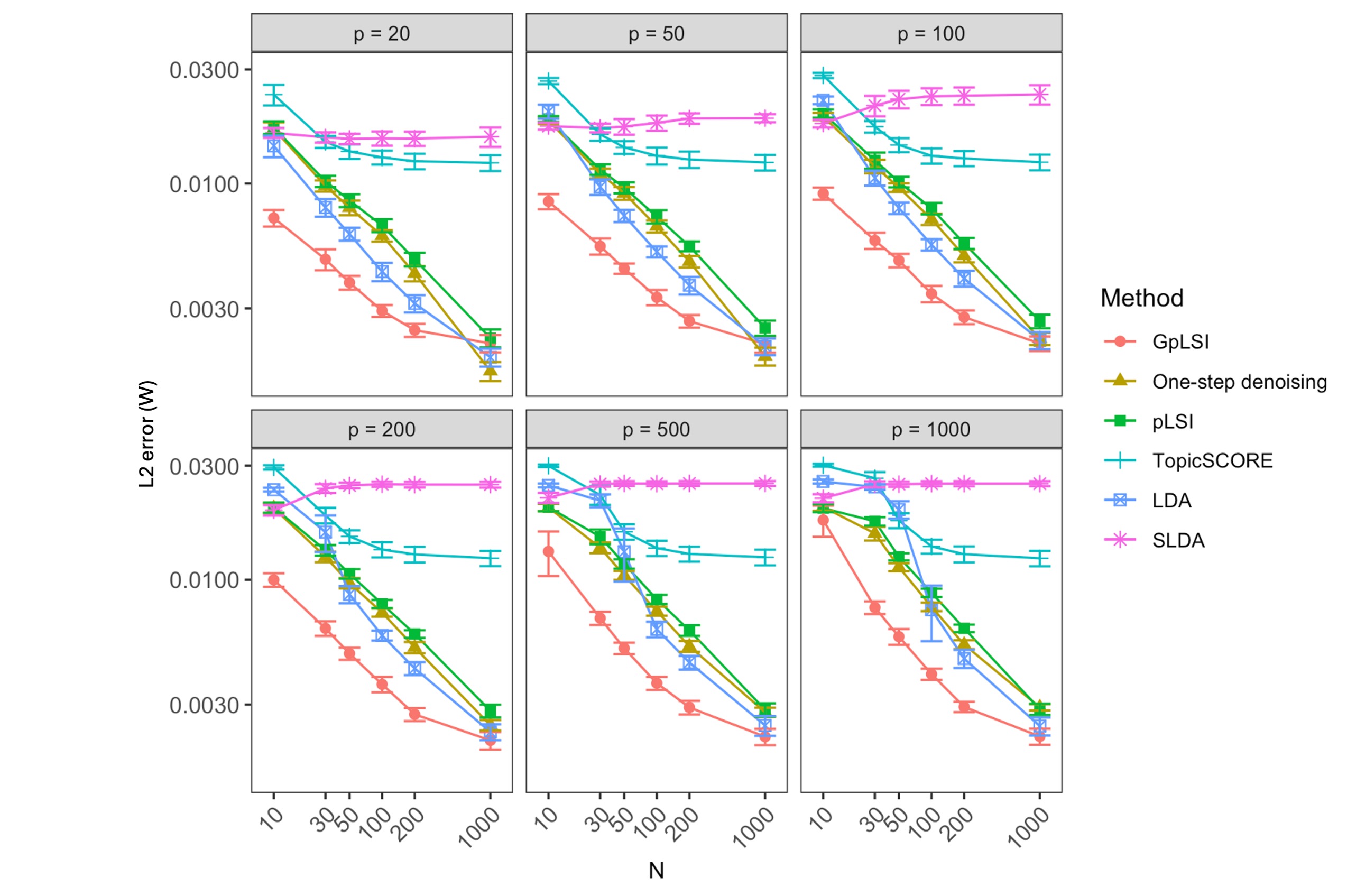}
    \caption{$\ell_2$ error for the estimator $\wh W$ (defined as $\text{min}_{P \in \mathcal{P}}\frac{1}{n}\| \wh W - WP\|_{F}$)  for different combinations of document length $N$ and vocabulary size $p$. Here, $n=1000$ and $K=3$.}
    \label{fig:Wl2pN}
\end{figure}


To assess the performance of GpLSI, we compare it against several established methods, including the original pLSI algorithm proposed by \cite{klopp2021assigning}, TopicSCORE \citep{ke2017new}, LDA \citep{blei2001latent}, and the Spatial LDA (SLDA) approach of \cite{chen2020modeling}. In addition, to highlight the efficiency of our iterative algorithm, we present results from a baseline variant that employs only a single denoising step. This one-step method is described in greater detail in Section C of the Appendix. To implement these algorithms, we use the R package \texttt{TopicScore} and the LDA implementation of the Python library \texttt{sklearn}. For SLDA, we use of the Python package \texttt{spatial-lda} with the default settings of the algorithm. We run 50 simulations and record the $\ell_1$ error, $\ell_2$ error of $W$ and $A$, and the computation time across various parameter settings $(p, N, n, K)$, reporting medians and interquartile ranges. To evaluate the performance of methods in difficult scenarios where document length $N$ is small compared to vocabulary size $p$, we check the errors for different combinations of $N=10, 30, 50, 100, 200, 1000$ and $p=20, 30, 50, 100, 200, 500$.

\paragraph{ Results}  Figure~\ref{fig:Wl2pN} demonstrates that GpLSI achieves the lowest $\ell_2$ error for $W$, even in scenarios with very small $N$. This shows that sharing information across similar documents on a graph improves the estimation of topic mixture matrix. Notably, while LDA and pLSI exhibit modest performance, they fail in regimes where $N<100$ and $p  \geq 200$. We also confirm that the one-step denoising variant of our method achieves a lower error estimation error than pLSI and LDA in settings where $N << p$.

\begin{figure}[h]
    \centering
    \includegraphics[width=0.9\textwidth]{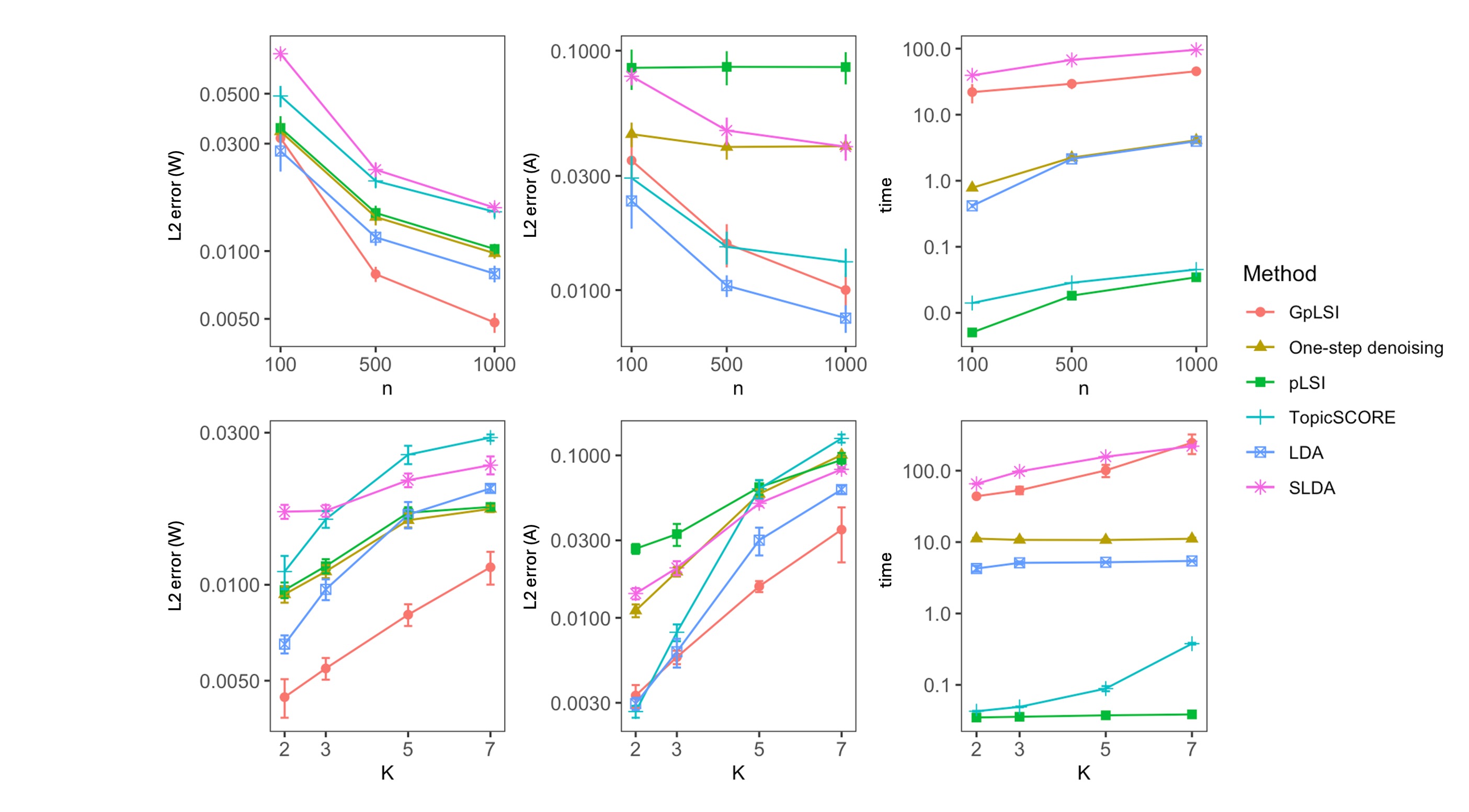}
    \caption{$\ell_2$ error of $W$ (left) and $A$ (middle) and computation time (right) for different corpus size $n$ and number of topics $K$. Here, $N=30$ and $p=30$. Errors are normalzied by $n$.}
    \label{fig:nK2}
\end{figure}

We also examine how the estimation errors scale with the corpus size $n$ and the number of topics $K$, as shown in Figure~\ref{fig:nK2}. 
We observe that GpLSI substantially outperforms other methods, particularly for the estimation of $W$. GpLSI achieves the lowest error for $A$, as we show in Section F of the Appendix. Similar patterns hold for $\ell_1$ errors of $A$ and $W$ also provided in Section F of the Appendix.

%% file: 5.real_experiments.tex
To highlight the applicability of our method, we deploy it on three real-life examples. \if1\blind
{All the code for the experiments is openly accessible at \url{https://github.com/yeojin-jung/GpLSI}.}\fi
\if0\blind
{All the code for the experiments will be openly accessible through Github}.}\fi

\subsection{Tumor Microenvironment discovery: the Stanford Colorectal Cancer dataset}
\label{sec:experiments:microenvironment}

We first consider the analysis of CODEX data, which allows the identification of individual cells in tissue samples, holding valuable insights into cell interaction profiles, particularly in cancer, where these interactions are crucial for immune response. Since cellular interactions are hypothesized to be local, these patterns are often referred to as “tumor microenvironments”. In the context of topic modeling, we can regard a tumor microenvironment as a \textit{document}, immune cell types as \textit{words}, and latent characteristics of a microenvironment as a \textit{topic}. However, due to the small number of words per document, the recovery of the topic mixture matrix and the topics themselves can prove challenging. \cite{chen2020modeling} propose using the adjacency of documents to assign similar topic proportions to neighboring tumor cells. Similarly, we construct a spatial graph based on proximity of tumor microenvironments to uncover novel tumor-immune cell interaction patterns.
 


The first CODEX dataset is a collection of 292 tissue samples from 161 colorectal cancer patients collected at Stanford University \citep{wu2022spacegm}. 
The locations of the cells were retrieved using a Voronoi partitioning of the sample, and the corresponding spatial graphs were constructed encoding the distance between microenvironments. 
More specifically, we define a tumor microenvironment as the 3-hop neighborhood of each cell, following the definition originally used by \cite{wu2022spacegm}. Each microenvironment contains 10 to 30 immune cells of 8 possible types. This aligns with the setting where the document length $N<30$ is small compared to the vocabulary size $p=8$. 

 \begin{figure}[h]
    \centering
    \includegraphics[width=1.0\textwidth]{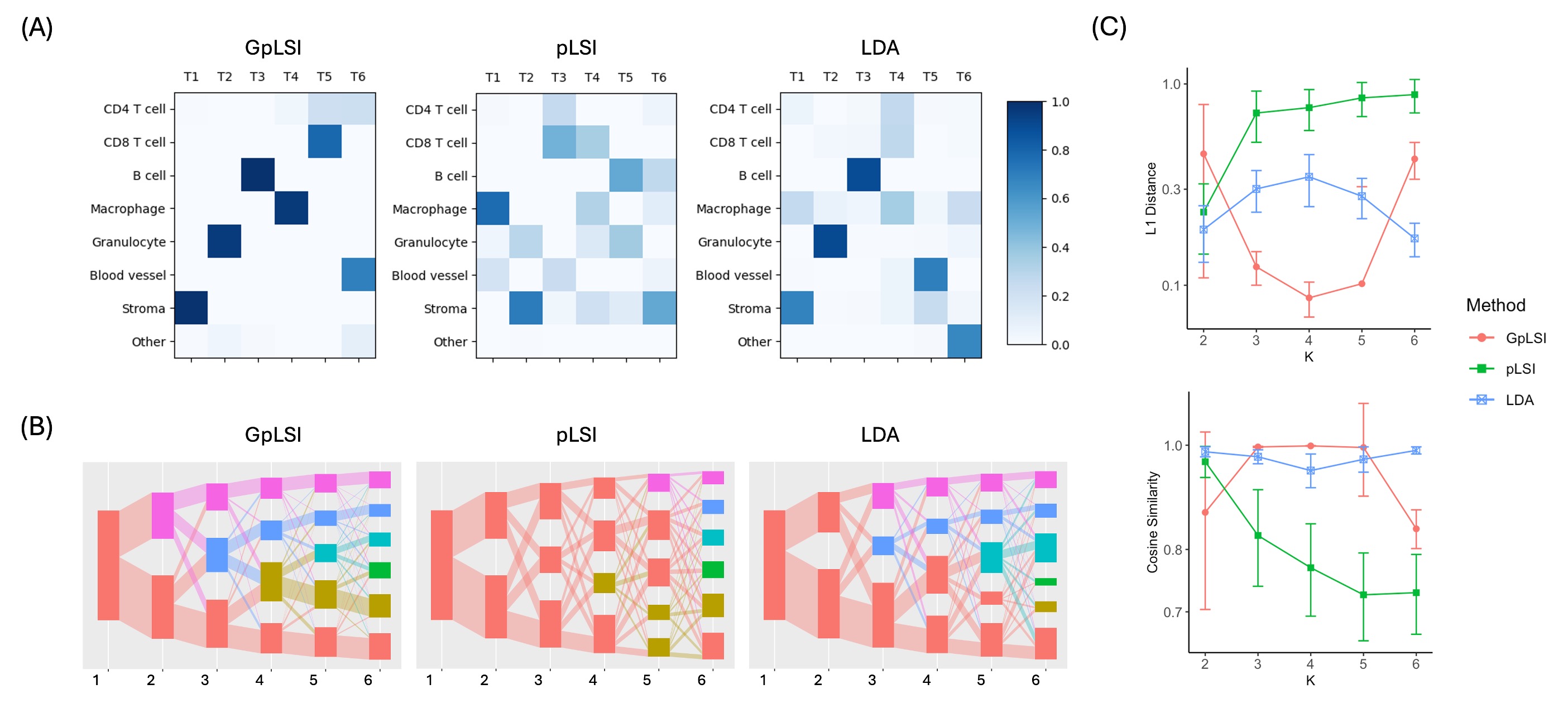}
    \caption{(A) Estimated tumor-immune topic weights of GpLSI, pLSI, and LDA. Topic weights are aligned across methods using cosine similarity. (B) Topic alignment paths of GpLSI, pLSI, and LDA using R package \texttt{alto}. (C) Pairwise $\ell_1$ distance and cosine similarity of topic weights from different batches of patients.}
    \label{fig:4}
\end{figure}

We aggregate frequency matrices and tumor cellular graphs of all samples and fit three methods: our proposed GpLSI, the original pLSI approach of \cite{klopp2021assigning}, and LDA \citep{blei2001latent} to estimate tumor-immune topics. The estimated topic weights of $K=6$ are illustrated in Figure~\ref{fig:4}(A). After aligning topic weights across methods, we observe similar immune topics (Topic 1, 2, 3) in GpLSI and LDA which are not found in the estimated topics of pLSI. 

To determine the optimal number of topics $K$, we use the method proposed by \cite{fukuyama2023multiscale}. In this work, the authors construct “topic paths” to track how individual topics evolve, split or merge, as the number of topics $K$, increases. 
We observe in Figure~\ref{fig:4}(B) that while the GpLSI path has non-overlapping topics up to $K=6$, other methods fail to provide consistent and well-separated topics.

To evaluate the quality and stability of the recovered topics, we also measure the coherence of the estimated topic weights of batches of samples, as suggested in \cite{tran2023sparse}. We divide 292 samples into five batches and estimate the topic weights $A^b$ for $b \in [5]$. For every pair of batches $(b, b^{'})$, we align $A^b$ and $A^{b^{'}}$ (we permute $A^{b^{'}}$ with $P$ where $P = \arg\min_{P \in \mathcal{P}}\|A^b-PA^{b^{'}}\|$) and measure the entry-wise $\ell_1$ distance and cosine similarity.
We repeat this procedure five times and plot the scores in Figure~\ref{fig:4}(C). We notice that GpLSI provides the most coherent topics across batches for $K=3,4,5$. 
Combining with the metrics of LDA, we can choose the optimal $K$ as 5 or 6. 

\begin{figure}[h]
    \centering
    \includegraphics[width=0.9\textwidth]{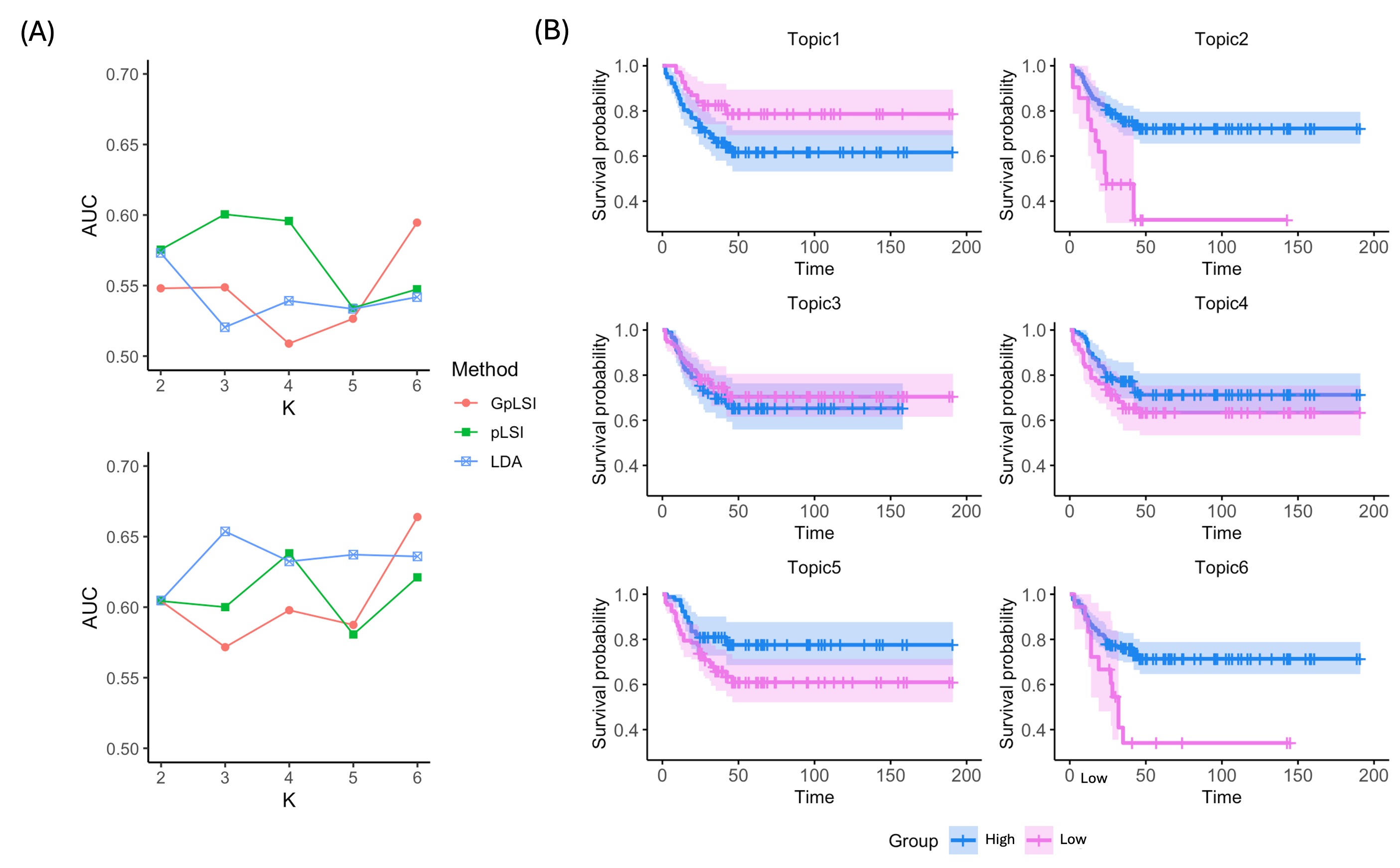}
    \caption{(A) AUC for predicting cancer recurrence using isometric log-ratio transformed topic proportions (top) and dichotomized topic proportions (bottom) as covariates. (B) Kaplan-Meier curves based on dichotomized topic proportions using GpLSI.}
    \label{fig:5}
\end{figure}

Next, we conduct survival analysis to identify the immune topics associated with higher risk of cancer recurrence. We consider two logistic models with different covariates to predict cancer recurrence and calculate the area under the curve (AUC) of the receiver-operating characteristic (ROC) curves to evaluate model performance. 

In the first model, we use the proportion of each microenvironment topic as covariates for each sample. Since the $K$ covariates sum up to one, we apply isometric log-ratio transformation to represent it with $K-1$ orthonormal basis vectors 
In the second model, we dichotomize each topic proportion to low and high proportion groups. The cutoffs are determined using the maximally selected rank statistics. 

The AUC for each number of topics is shown in Figure~\ref{fig:5}(A). GpLSI achieves the highest area under the curve (AUC) at $K=6$ in both models. We also plot Kaplan Meier curves for each topic using the same dichotomized topic proportions. The result for GpLSI is illustrated in Figure~\ref{fig:5}(B). We observe that Topic 2, which is characterized by a high prevalence of granulocytes, and Topic 6, a mixture of CD4 T cells and blood vessels, are associated with lower cancer recurrence. Positive effect of granulocytes on cancer prognosis was also reported by \cite{wu2022spacegm}, who found out that a microenvironment with clustered granulocyte and tumor cells is associated with better patient outcomes. We observe the same association of granulocyte with lower risk in LDA Figure 16 of the Appendix.

\subsection{Understanding Structure in Mouse Spleen Samples}

\begin{figure}[t]
    \centering
    \includegraphics[width=1.0\textwidth]{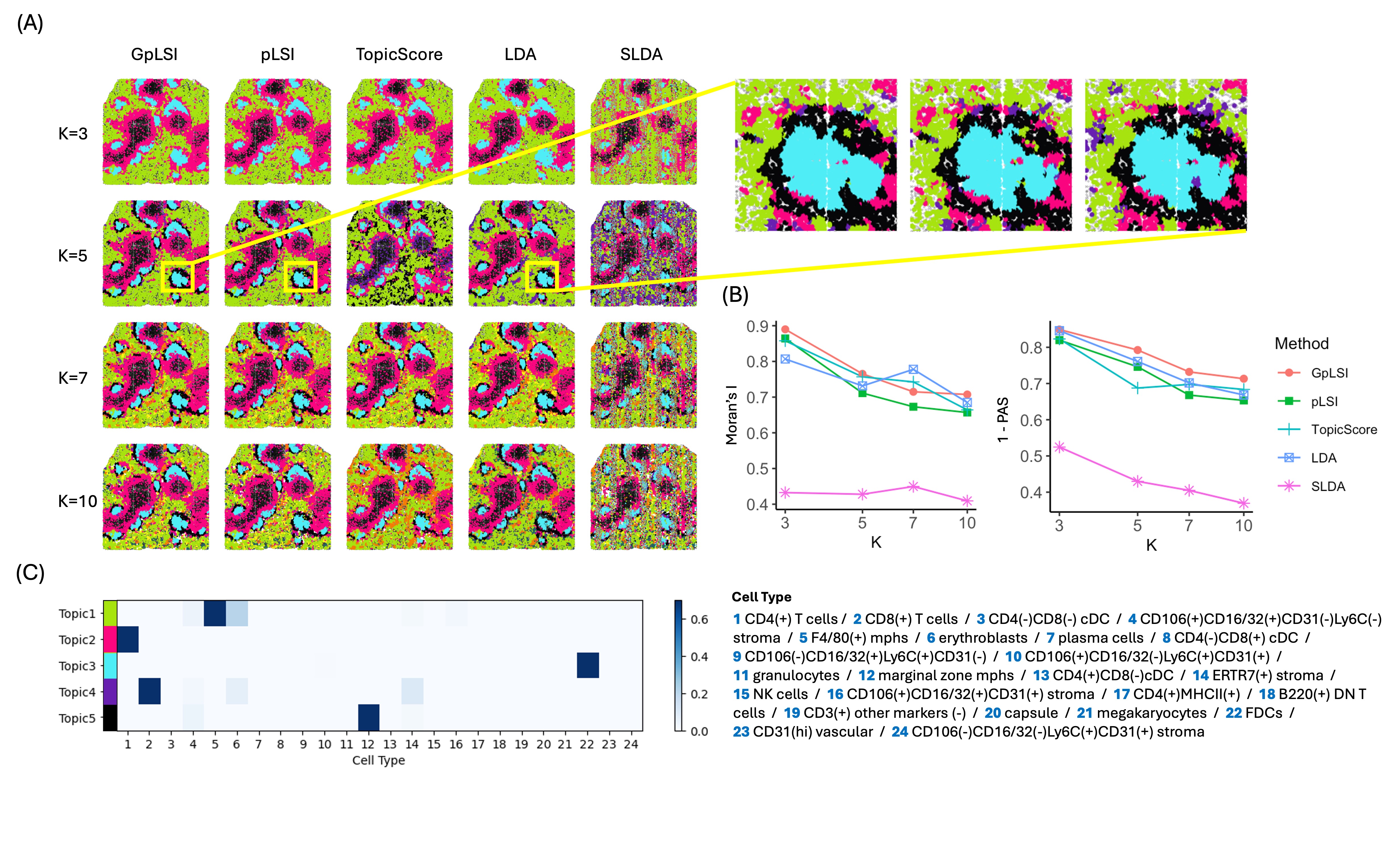}
    \caption{(A) Visualization of estimated B cell microenvironment topics for $K=3,5,7,10$. (B) Comparison of clustering performance using Moran's I and PAS score. We plot 1-PAS for better interpretation. (C) Estimated B cell microenvironment topic weights for $K=5$ using GpLSI.} 
    \label{fig:6}
\end{figure}

We also apply our method to identify immune topics in mouse spleen. In this setting, each document is anchored to a B cell \citep{chen2020modeling}. A previous study has processed the original CODEX images from \cite{goltsev2018deep} to obtain the frequencies of non-B cells in the 100 pixel neighborhood of each B cell \citep{chen2020modeling}. The final input for the topic models consists of a 35,271 B cell microenvironments by 24 cell types frequency matrix, along with the positional data of B cells.

In this example, we evaluate GpLSI by examining whether the introduction of our graph-based regularization term in the estimation of topic mixture matrices enhances document clustering. Figure~\ref{fig:6}(A) presents the estimated topics for all models at $K=3,5,7,10$. Notably, the topics derived from GpLSI, pLSI, and LDA more clearly demarcate distinct B cell microenvironment domains compared to those estimated by TopicSCORE and LDA. Among these three methods, GpLSI yields the least noisy cellular clustering, as evidenced by the magnified view of a selected subdomain. 

We also evaluate the quality of clusters with two metrics, Moran's I and the percentage of abnormal spots (PAS) \citep{shang2022spatially}. Moran's I is a classical measure of spatial autocorrelation that assesses the degree to how values are clustered or dispersed across a spatial domain. PAS score measures the percentage of B cells for which more than 60\% of its neighboring B cells have different topics. Higher Moran's I and lower PAS score indicate more spatial smoothness of the estimated topics. From Figure~\ref{fig:6}(B), we conclude that GpLSI has the highest Moran I, and the lowest PAS scores, demonstrating improved spatial smoothness of the topics. 

We observe that the B cell microenvironment topics identified with GpLSI align well with their biological context (Figure~\ref{fig:6}(C)). By referencing the manual annotations of B cells from the original study by \cite{goltsev2018deep}, we infer that Topic 1, Topic 2, Topic 3, and Topic 5 correspond to the red pulp, periarteriolar lymphoid sheath (PALS), B-follicles, and the marginal zone. This interpretation is further supported by high expression of CD4+T cells in Topic 2 (PALS) and high expression of marginal zone macrophages in Topic 5 (marginal zone).

\subsection{Analysis of the ``What's Cooking'' dataset}
\label{sec:experiments:whatscooking}

This dataset contains recipes from 20 different cuisines across Europe, Asia, and South America. Each recipe is a list of ingredients which allow us to convert to a count matrix with 13,597 recipes (documents) and 1,019 unique ingredients (words). Under the assumption that neighboring countries would have similar cuisine styles, we construct a graph of recipes based on the geographical proximity of the countries. Specifically, for each recipe, we select the five closest recipes from neighboring countries (including its own country) based on the $\ell_1$ distance of the ingredient count vectors and define them as neighboring nodes on the graph. Through this, we aim to identify general cooking styles that are prevalent across various countries worldwide. 



 \begin{figure}[h]
    \centering
    \includegraphics[width=1.0\textwidth]{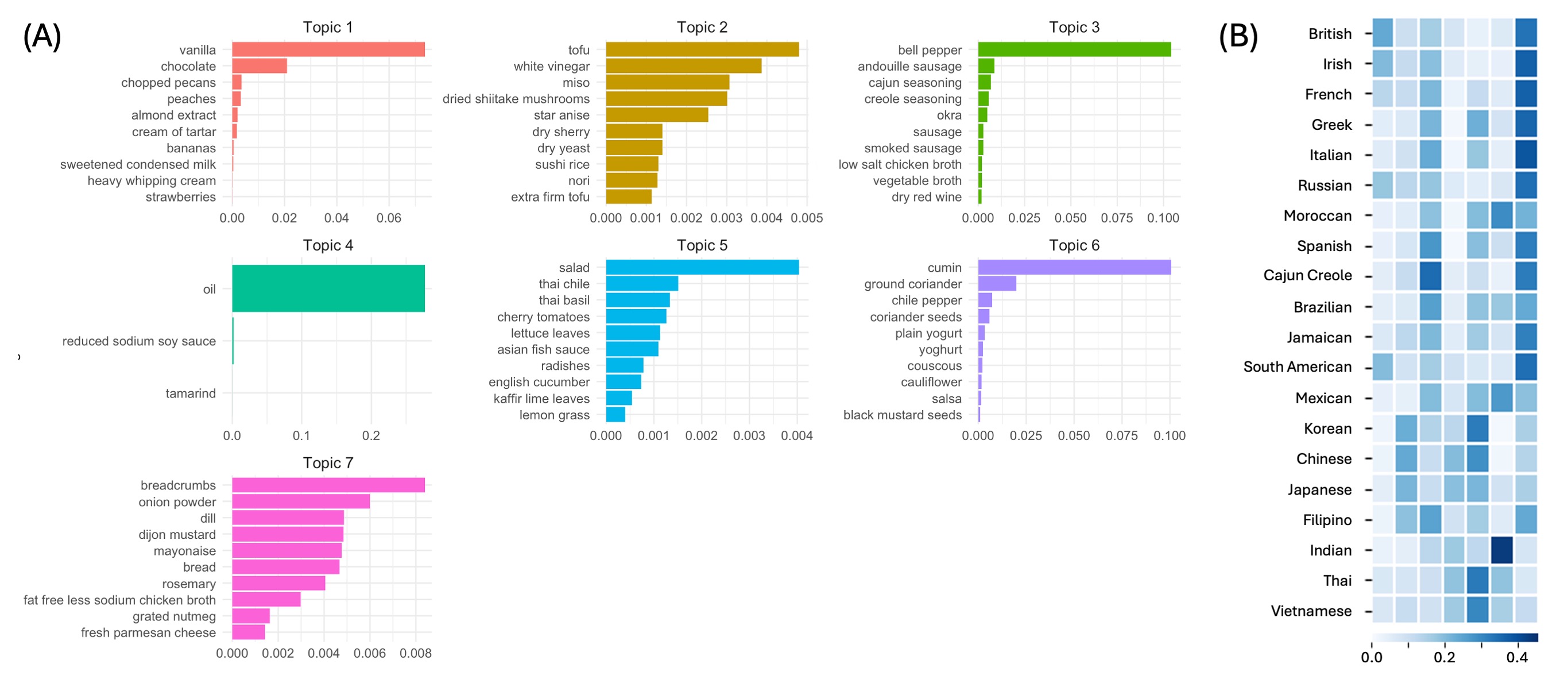}
    \caption{(A) Estimated anchor ingredients for each topic using GpLSI. (B) Proportion of topics for each cuisine. Each recipe was assigned to a topic with the highest document-topic mixture weight. For each cuisine, we count the number of recipes for each topic and divide by the total number of recipes in the cuisine.}
    \label{fig:anchor_gplsi}
\end{figure}

We run GpLSI, pLSI, and LDA with $K=5, 7, 10, 15, 20$ topics. We illustrate the estimated topics of GpLSI for $K=7$ in Figure~\ref{fig:anchor_gplsi}. The results for pLSI and LDA are provided in Section G of the Appendix.
With this approach, Topic 1 is clearly a baking topic and Topic 6 is defined by strong spices and sauces common in Mexican or parts of Southeast Asian cuisines. We also observe a general topic for Asian cuisines (Topic 2) and another for Western countries (Topic 7). To evaluate the estimated topics, we compare each topic's characteristics with the cuisine-by-topic proportion (Figure~\ref{fig:anchor_gplsi}(C)). Indeed, the style of each topic defined by the anchor ingredients aligns with the cuisines that have a high proportion of that topic. For example, the baking topic (Topic 1) is prevalent in British, Irish, French, Russian, and South American cuisines. 

In contrast, for pLSI, it is difficult to analyze the characteristics for each topic because Topics 1-4 have one or no identified anchor ingredients. 
Comparing the cuisine-by-topic proportions of GpLSI and LDA, we observe that GpLSI reveals many cuisines as mixtures of different cooking styles (Figure~\ref{fig:anchor_gplsi}(B)). In contrast, for LDA, many cuisines such as Moroccan, Mexican, Korean, Chinese, Thai have their recipes predominantly classified to a single topic (Figure 18(B) of the Appendix). GpLSI provides estimates of topic mixture and topic weights that are more relevant to our goal of discovering global cooking styles.

%% file: appendixA_initialization.tex
In this section, we show that the initialization step of GpLSI provides reasonable estimators of $U$ and $V$.

\subsection*{Proof of Theorem 1}
\begin{proof}
Let $D_0$ denote the diagonal matrix where each entry $(D_0)_{jj} $  is defined as: $(D_0)_{jj} = \frac{1}{n}\sum_{i=1}^n M_{ij}.$
Let $\wh D_0$ denote its empirical counterpart, that is, the diagonal matrix defined as: $\wh D_0 = \text{diag}( \frac{1}{n}\{\sum_{i=1}^n X_{ij}\}_{j\in [p]} ) $, so that $\E[\wh D_0] = D_0$. 
    We have, by definition of the initialization procedure: $$ \wh V_0 = U_K(X^\top X-\frac{n}{N}\wh D_0),$$
where the notation $U_K(A)$ denotes the first $K$ left singular vectors of the matrix $A$.

We write $X = M + Z$, where $Z$ denotes some multinomial noise. E
    We have:
    \begin{equation}
        \begin{split}
        {Z^\top Z} &=  \sum_{i=1}^n Z_{i\cdot}^\top Z_{i\cdot}\\
          \implies  \E[{Z^\top Z}] &=\sum_{i=1}^n \text{Cov}(Z_{i\cdot})=  \sum_{i=1}^n \text{Cov}(X_{i\cdot})
        \end{split}
    \end{equation}
    as $Z$ is a centered version of $X$ ($Z = X-M$). Since each row $X_{i\cdot}$ is distributed as a Multinomial$(1, M_{i\cdot})$:
$$( \text{Cov}(X_{i\cdot}))_{jj'} = \begin{cases}
    \frac{ M_{ij} (1-M_{ij}) }{N} \qquad \text{ if } j= j'\\
  -   \frac{ M_{ij}M_{ij'} }{N} \qquad \text{ if } j\neq j'\\
\end{cases} \implies \sum_{i=1}^n ( \text{Cov}(X_{i\cdot}))_{jj'} =  \begin{cases}
   \sum_{i} \frac{ M_{ij} (1-M_{ij}) }{N} \qquad \text{ if } j= j'\\
   -   \sum_{i=1}^n\frac{ M_{ij}M_{ij'} }{N} \qquad \text{ if } j\neq j'\\
   \end{cases}$$

Thus:
 \begin{equation}
        \begin{split}
            \E[{Z^\top Z}] &= \frac{n}{N} D_0 - \frac{M^\top M}{N} \\
            &= \frac{n}{N} D_0 - \frac{V^\top \Lambda^2 V}{N}.
        \end{split}
    \end{equation}

Therefore:
    \begin{equation}
    \begin{split}
      {X^\top X} - \frac{n}{N} \wh D_0-(1-\frac{1}{N}){M^{\top}M}  &=  Z^{\top}Z + Z^{\top}M + M^{\top}Z   - \frac{n}{N}\wh D_0 -\E[Z^{\top}Z] + \frac{n}{N}\E[\wh D_0]\\
    \end{split}
\end{equation}

Thus $ \E[{X^\top X} - \frac{n}{N} \wh D_0] =(1-\frac{1}{N}) M^{\top} M. $
We further note that  $(1-\frac{1}{N}) M^{\top} M =  V \tilde{\Lambda}^2V$ with $\tilde{\Lambda} = \sqrt{1-\frac{1}{N}} \Lambda,$ so the eigenvectors of the matrix ${X^\top X} - \frac{n}{N} \wh D_0$ can be considered as estimators of those of the matrix $M^\top M.$

    By the Davis-Kahan theorem \citep{giraud2021introduction}:
    \begin{equation}
        \begin{split}
            \|\sin \Theta(V, \wh V^0)\|_F &\leq 2\frac{\|  {X^\top X} - \frac{n\wh D_0}{N} -(1-\frac{1}{N}){M^{\top}M} \|_F}{(1-\frac{1}{N})\lambda_K(M)^2}\\
            &\leq 2\frac{\| Z^{\top}Z -\E[Z^{\top}Z]\|_F+ \frac{n}{N}\| \wh D_0 - \E[\wh D_0] \|_F+\|  Z^{\top}M\|_F + \|M^{\top}Z   \|_F}{ (1-\frac{1}{N})\lambda_K(M)^2}\\
        \end{split}
    \end{equation}

By Lemma~\ref{lem:concentration_xtx}, we have with probability at least $1-o(n^{-1})$:
$$ \| Z^{\top}Z -\E[Z^{\top}Z]\|_F \leq C_1 K\sqrt{\frac{n \log(n)}{N}},$$
$$\|  Z^{\top}M\|_F  = \|  M^{\top}Z\|_F   \leq C_2K \sqrt{\frac{n\log(n)}{N}},  $$
and
$$\frac{n}{N} \| \wh D_0- \E[\wh D_0] \|_F \leq  \frac{C_3}{N} \sqrt{\frac{Kn \log(n)}{N}} .$$

Thus, assuming $N>1$, so $ \frac{1}{1 - \frac{1}{N}} <2$ and $\frac{1}{N} \leq \frac{1}{2}:$
\begin{equation*}
    \begin{split}
        \|\sin \Theta(V, \wh V^0)\|_F 
        &\leq \frac{4C}{\lambda_K(M)^2} K\sqrt{\frac{n\log(n)}{N}}
    \end{split}
\end{equation*}
with $C = C_1 \vee C_2 \vee C_3.$ 
Under Assumption 4, we have $\lambda_K(M) \geq c \lambda_1(W) \geq c \sqrt{n/K}$ (see Lemma~\ref{lemma:singular_value_H}), therefore:
\begin{equation*}
    \begin{split}
       \|\sin \Theta(V, \wh V^0)\|_F 
       &\leq  \frac{4C}{c^2}K^2\sqrt{\frac{\log(n)}{nN}}\\
   \end{split}
\end{equation*}

The condition on $N$ assumed in Theorem 2 ensures that $ \|\sin \Theta(V, \wh V^0)\|_F<\frac{1}{2}$.
\end{proof}

%% file: appendixB_iterative.tex
Our proof is organized along the following outline:
\begin{enumerate}
    \item We begin by showing that our graph-total variation penalty yields better estimates of the left and right singular vectors. To this end, we must show that, provided that the initialization is good enough, the estimation error of the singular vectors decreases with the number of iterations.
    \item We show that, by a simple readaptation of the proof by \cite{klopp2021assigning}, our estimator---which simply plugs in our singular vector estimates in their procedure ---yields a better estimate of the mixture matrix $W$.
    \item Finally, we show that our estimator of the topic matrix $A$ yields better error.
\end{enumerate}

\subsection{Analysis of the graph-regularized SVD procedure}
In this section, we derive high-probability  error bounds for the estimates $\wh U$ and $\wh V$ that we obtain in Algorithm 1. For each $t>0$, we define the error $L_t$ at iteration $t$ as:
    \begin{equation}\label{eq:L_t}
    L_t = \max\{ \| \sin\Theta(V, \widehat{V}^t)\|_{F}, \|\sin\Theta(U, \widehat{U}^t)\|_{F}\}.
    \end{equation}

 Our proof operates by recursion. We explicit the dependency of $L_t$ on the error at the previous iteration $L_{t-1}$, and show that $\{L_t\}_{t=1,\cdots, t_{\max}}$ forms a geometric series. To this end, we begin by analyzing the error of the denoised matrix $\bar{U}^t$, of which we later take an SVD to extract $\wh U^t$.

 At each iteration $t$, the first step of Algorithm 1 is to consider the following optimization problem:
    \begin{equation}\label{eq:objective}
        \bar{U}^{t}
          \in \arg \min_{ U \in \R^{n \times K} }\norm{U-X\wh V^{t-1}}_F^2+\rho\norm{\Gamma U}_{21} \\
    \end{equation}
    
   Fix $t>0$. To simplify notations, we let 
    \begin{equation}\label{eq:def_low_rank}
        \TY = X\wh V^{t-1}, \qquad \TU = M \wh V^{t-1}, \qquad \TZ = Z \wh V^{t-1}
    \end{equation} Note that with these notations,  $\TY$ can be written as:
    \begin{equation}\label{eq:TY}
        \TY = \TU  + \TZ
    \end{equation}

\begin{lemma}[Error bound of Graph-aligned Denoising] \label{lemma:denoising}
Let Assumption 1 to 5 hold and let $L_t$, $\bar{U}^{t}$, $\TY$, $\TU$, $\rho$ be given as \eqref{eq:L_t}-\eqref{eq:TY}. Assume $\max(K,p) \leq n$ and $\sqrt{K} \leq p$.
Then, for a choice of $\rho = 4C^*\rho(\Gamma)\sqrt{\frac{Kp_n}{N}}(1+L_{t-1})$ with a constant $C^*>0$, there exists a constant $C>0$ such that with probability at least $1-o(n^{-1})$, for any $t > 0$,
    \begin{equation}\label{eq:err_denoising}
        \|\bar U^t-\tilde{U}\|_F  
        \leq C\sqrt{\frac{K\log(n)}{N}}\left(\sqrt{n_{\C}}+\rho(\Gamma)\sqrt{s}\sqrt{\lambda_{\max}(L)}(1+L_{t-1})\right)
    \end{equation}
    where $L$ denotes the graph Laplacian.
\end{lemma}

\begin{proof}
        By the KKT conditions, the solution $\bar{U}^t$ of \eqref{eq:objective} verifies :
        \begin{equation*}
            \begin{split}
              2 (\bar{U}^t - \TY) +\rho \Gamma^{\top} D \Gamma \bar{U}^t &=0 \qquad \text{with } D = \text{diag}(\{\frac{1}{\|(\Gamma U^t)_{e\cdot}\|_{2}}\}_{e \in \mathcal{E}})   
            \end{split}
        \end{equation*}
        This implies:
        \begin{equation*}
            \begin{split}
            \< \TY - \bar{U}^t, \bar{U}^t\>
               &=  \frac{\rho}{2} \< \Gamma\bar{U}^t, D \Gamma \bar{U}^t\>  = \frac{\rho}{2}  \|\Gamma \bar{U}^t\|_{21}\\
               \text{and} \quad \forall U \in \R^{ n \times K}, \qquad  \< \TY - \bar{U}^t , U\>
               &= \frac{\rho}{2} \< \Gamma U, D \Gamma \bar{U}^t\> \leq \frac{\rho}{2}  \|\Gamma {U}\|_{21}\\
            \end{split}
        \end{equation*}
        Therefore:
        \begin{equation*}
            \begin{split}
            \< \TY - \bar{U}^t, U - \bar{U}^t\>
               &\leq   \frac{\rho}{2} (\|\Gamma {U}\|_{21} -  \|\Gamma \bar{U}^t\|_{21})\\
               \<\tilde{U} - \bar{U}^t, U - \bar{U}^t\>
               &\leq \<\tilde{Z}, \bar{U}^t-U\> +  \frac{\rho}{2} (\|\Gamma {U}\|_{21} -  \|\Gamma \bar{U}^t\|_{21})\\
            \end{split}
        \end{equation*}
       Using the polarization inequality:
        \begin{equation*}
            \begin{split}
            \|U - \bar{U}^t\|_F^2  + \| \tilde{U} - \bar{U}^t \|_F^2 &\leq \| \tilde{U} - {U} \|_F^2 + 2\<\tilde{Z}, \bar{U}^t-U\> +  \rho (\|\Gamma {U}\|_{21} -  \|\Gamma \bar{U}^t\|_{21})\\
            \end{split}
        \end{equation*} and, choosing $U = \tilde{U}:$
        \begin{equation*}\label{eq:basic1}
            \| \tilde{U} - \bar{U}^t \|_F^2 \leq \<\tilde{Z}, \bar{U}^t-\tilde{U}\> + \frac{\rho}{2} (\|\Gamma \tilde{U}\|_{21} -  \|\Gamma \bar{U}^t\|_{21})
        \end{equation*}
    
    Let $\Delta = \TU-\bar{U}^t$. By the triangle inequality, the right-most term in the above inequality can be rewritten as:
        \begin{equation*}
            \begin{split}
                \|\Gamma \tilde{U}\|_{21} -  \|\Gamma \bar{U}^t\|_{21} 
        &= \|(\Gamma \tilde{U})_{\mathcal{S}\cdot}\|_{21}  +  \|(\Gamma \tilde{U})_{\mathcal{S}^C\cdot}\|_{21}-\|(\Gamma \tilde{U}+\Gamma \Delta)_{\mathcal{S}\cdot}\|_{21} - \|(\Gamma \tilde{U}+\Gamma \Delta)_{\mathcal{S}^C\cdot}\|_{21}\\
        &\leq \|(\Gamma \Delta)_{\mathcal{S}\cdot}\|_{21} - \|(\Gamma \Delta)_{\mathcal{S}^C\cdot}\|_{21},\\
        \end{split}
        \end{equation*}
    since by assumption, $\|(\Gamma \tilde{U})_{\mathcal{S}^C\cdot}\|_{21}=0.$

        We turn to the control of the error term $\<\tilde{Z}, \bar{U}^t-\tilde{U}\>$. Using the decomposition of $\R^n$ induced by the projection  $\Gamma^{\dagger} \Gamma$ as $I_n = \Pi \oplus^{\perp} \Gamma^{\dagger} \Gamma$, we have:
        \begin{equation}\label{eq:bi_2}
            \begin{split}
                \<\tilde{Z}, \bar{U}^t-\tilde{U}\> 
                &=  \<\tilde{Z}, \Pi (\bar{U}^t-\tilde{U})\> + \<\tilde{Z}, \gamp \Gamma (\bar{U}^t-\tilde{U}
                )\>\\
                 &= \underbrace{\< \Pi\tilde{Z}, \Pi \Delta\>}_{(A)} + \underbrace{ \< (\gamp)^\top\tilde{Z}, \Gamma\Delta\>}_{(B)}.
            \end{split}
        \end{equation}

\paragraph{ Bound on (A) in Equation \eqref{eq:bi_2}} 
    By Cauchy-Schwarz:
    \begin{equation*}
        \begin{split}
            \< \Pi\tilde{Z}, \Pi \Delta\> &\leq \| \Pi \tilde{Z}\|_F \|\Pi \Delta \|_F
        \end{split}
    \end{equation*}

    By Lemma \ref{lemma:concentration_frob_norm_pitildeZ}, with probability at least $1-o(n^{-1})$:
    \begin{equation*}
        \|  \Pi \tilde{Z}\|_F^2 \leq C_1n_{\C}K\frac{\log(n)}{N}
    \end{equation*}

       \paragraph{ Bound on (B) in Equation \eqref{eq:bi_2}.}
    \begin{equation*}
        \begin{split}
            \< (\Gamma^{\dagger})^\top\tilde{Z}, \Gamma \Delta\> 
            &= \sum_{e \in [m]} \< ((\Gamma^{\dagger})^{\top}\tilde{Z})_{e\cdot}, (\Gamma \Delta)_{e\cdot} \>\\
            &\leq  \sum_{e \in [m]} \| ((\Gamma^{\dagger})^{\top}\tilde{Z})_{e\cdot}\|_2 \|(\Gamma \Delta)_{e\cdot}\|_{2}\qquad \text{by Cauchy-Schwarz}\\
            &\leq  \max_{e\in [m]}  \|(\Gamma^{\dagger})^{\top}\tilde{Z})_{e\cdot}\|_{2} \sum_{e \in [m]} \| (\Gamma \Delta)_{e\cdot}\|_2 \\
               &=  \max_{e\in [m]}  \|((\Gamma^{\dagger})^{\top}\tilde{Z})_{e\cdot}\|_{2} \| \Gamma \Delta\|_{21}\\
        \end{split}
    \end{equation*}
    Thus, on the event  $\mathcal{A} = \{ \rho \geq  4 \max_{e\in [m]}  \|(\gamp)^\top\tilde{Z})_{e\cdot}\|_{2}\}$, we have:
    \begin{equation*}
        \begin{split}
            \< (\Gamma^{\dagger})^\top\tilde{Z}, \Gamma \Delta\> &\leq  \frac{\rho}{4} \| \Gamma \Delta\|_{21}.\\
        \end{split}
    \end{equation*}

    To derive $\P(\mathcal{A})$, we first establish the relationship between $\tilde{Z}$ and $L_{t-1}$,
    \begin{equation*}
    \tilde{Z} = Z (P_{V}+P_{V_{\perp}})\wh V^{t-1} = ZVV^{\top}\wh V^{t-1}+ZV_{\perp}V_{\perp}^{\top}\wh V^{t-1}
    \end{equation*}
Then,
    \begin{equation*}
    \begin{split}
        \max_{e\in [m]}  \|(\gamp)^\top\tilde{Z})_{e\cdot}\|_{2} 
        &=  \max_{e\in [m]}  \|(\gamp)^\top(ZVV^{\top}\wh V^{t-1}+ZV_{\perp}V_{\perp}^{\top}\wh V^{t-1}))_{e\cdot}\|_{2}\\
        &\leq \max_{e\in [m]}  \|((\gamp)^\top Z)_{e\cdot}\|_{2}\| V^\top \wh V^{t-1} \|_{op}+\max_{e\in [m]}  \|((\gamp)^\top Z)_{e\cdot}\|_{2}\|V_{\perp}^{\top}\wh V^{t-1}\|_{op} \\
        &\leq \max_{e\in [m]}  \|((\gamp)^\top Z)_{e\cdot}\|_{2}(1+L_{t-1})
    \end{split}
    \end{equation*}
    
    where we used the fact $\|\sin \Theta (V, \wh V^{t-1})\|_F = \|V_{\perp}^{\top}\wh V^{t-1}\|_F \geq \|V_{\perp}^{\top}\wh V^{t-1}\|_{op} $. From Lemma~\ref{lem:concentration_Gamma_dagger_Z}, for a choice of $\rho = 4C^*\rho(\Gamma)\sqrt{\frac{K\log(n)}{N}}(1+L_{t-1})$, then $\mathbb{P}\left[\mathcal{A}\right] \geq 1-o(n^{-1})$.

    Therefore:
    \begin{equation}\label{eq:basic1_1}
            \begin{split}
        \|\Delta\|_F^2  
        &\leq \|\Pi\tilde{Z}\|_F \| \Delta\|_F + \frac{3\rho}{4} \|\Gamma \Delta\|_{21} \\
        &\leq  \|\Pi\tilde{Z}\|_F \| \Delta\|_F + \frac{3\rho}{4} \sqrt{s}\|\Gamma \Delta\|_{F} \\
        &\leq  \|\Pi\tilde{Z}\|_F \| \Delta\|_F + \frac{3\rho}{4} \sqrt{s} {\lambda_{\max}(\Gamma)}\|\Delta\|_{F} \\
            \end{split}
        \end{equation}

    and thus:
    \begin{equation*}
    \label{eq:bound1}
    \begin{split}
        \|\tilde{U} - \bar{U}^t\|_F 
        &\leq  C_1\sqrt{\frac{n_{\C} K\log(n)}{N}}+3C_2\rho(\Gamma)\sqrt{s}{\lambda_{\max}(\Gamma)}\sqrt{\frac{K\log(n)}{N}}(1+L_{t-1})\\
       & \leq  C\sqrt{\frac{K\log(n)}{N}} (\sqrt{n_{\C}}+\rho(\Gamma)\sqrt{s}{\lambda_{\max}(\Gamma)}(1+L_{t-1}))\\
    \end{split}
    \end{equation*}  
    The result follows by noting that the Laplacian of the graph $L$ is linked to $\Gamma$ by $L=\Gamma^\top \Gamma $.
\end{proof}





\subsection*{Proof of Theorem 2}
\label{proof:theorem:GSVD}
\begin{proof}
Recall $L_t$, the error at each iteration $t$:
\begin{equation}
    L_t = \max\{ \| \sin\Theta(V, \widehat{V}^t)\|_{F}, \|\sin\Theta(U, \widehat{U}^t)\|_{F}\}.
\end{equation}
  \paragraph{ Bound on $\|\sin(\Theta(U, \widehat{U}^t)\|_{F}$.}   We start by deriving a bound for $\|\sin \Theta (U, \wh U^t)\|_{F}$.  Let $U_{\perp}$ denote the orthogonal complement of $U,$ so that:
  $$ I_{n} = UU^{\top} + U_{\perp}U_{\perp}^{\top}.$$
    
Noting that $\widehat{U}^t$ is the matrix corresponding to the top $K$ left singular vectors of the matrix $\tilde{U}^t = (\tilde{U}^t - MV ) + MV=(\tilde{U}^t - M\wh{V}^{t} + M\wh{V}^{t} - MV ) + MV,$ by Theorem 1 of \cite{cai2018rate} (which we rewrote in Lemma ~\ref{lemma:perturbation} of this manuscript to make it self-contained):
    \begin{equation*}
    \begin{split}
        \|\sin \Theta (U, \wh U^t)\|_{F} 
        &\leq \frac{\| P_{U_{\perp}} (\tilde{U}^t - M\wh{V}^{t} + M\wh{V}^{t} - MV)\|_{F}}{\lambda_{\min} (U^\top \tilde{U}^t)  } \\
        &= \frac{ \| P_{U_{\perp}} (\tilde{U}^t - M\wh{V}^{t})\|_{F}}{\lambda_{\min} (U^\top \tilde{U}^t)}
    \end{split}
    \end{equation*}
    where the second line follows from noting that $P_{U_{\perp}} (M\wh{V}^{t} - MV)=0. $

Since $\Lambda$ is a diagonal matrix, we have:               \begin{equation*}
        \begin{split}
\lambda_{\min} (U^\top M \wh{V}^{t-1} ) = \lambda_{\min} (\Lambda V^{\top} \wh{V}^{t-1} )  &= \min_{u \in \R^{K},v \in \R^{p}: \|u\|=\|v\|=1} u^\top\Lambda V^{\top} \wh{V}^{t-1} v\\
&= \lambda_{K}(M) \min_{u \in \R^{K},v \in \R^{p}: \|u\|=\|v\|=1}  u^\top V^{\top} \wh{V}^{t-1} v \\
&=  \lambda_{K}(M) \lambda_{\min}( V^{\top} \wh{V}^{t-1}) 
        \end{split}
    \end{equation*}

Thus, by Weyl's inequality:
       \begin{equation*}
        \begin{split}
            \lambda_{\min} (U^\top \tilde{U}^t)&=   \lambda_{\min} (U^\top (\tilde{U}^t - M \wh{V}^{t-1} + M \wh{V}^{t-1} ))\\
            &\geq  - \lambda_{\max} (U^\top (\tilde{U}^t - M \wh{V}^{t-1})) +\lambda_{\min} (U^\top  M \wh{V}^{t-1} )\\
            &\geq \lambda_{\min} (\Lambda V^{\top} \wh{V}^{t-1})- \|\tilde{U}^t - M \wh{V}^{t-1}\|_{F} = \lambda_K(M)  \sqrt{1-L_{t-1}^2} - \|\Delta\|_{F}\\
        \end{split}
    \end{equation*}
where $\Delta =\tilde{U}^t - M \wh{V}^{t-1}. $
By Lemma~\ref{lemma:denoising}, we know that:
$$\|\Delta\|_{F} \leq C\sqrt{\frac{K\log(n)}{N}}\left(\sqrt{n_{\C}}+\rho(\Gamma)\sqrt{s}\sqrt{\lambda_{\max}}(1+L_{t-1})\right) = \eta_n + \delta_n L_{t-1}$$
with $\eta_n=C\sqrt{\frac{K\log(n)}{N}}\left(\sqrt{n_{\C}}+\rho(\Gamma)\sqrt{s}\sqrt{\lambda_{\max}(\Gamma)}\right)$ and $\delta_n =C\rho(\Gamma)\sqrt{s\lambda_{\max}(\Gamma)\frac{K\log(n)}{N}}. $
Thus:
     \begin{equation*}
    \begin{split}
    \|\sin \Theta (U, \wh U^t)\|_{F}  &\leq \frac{ \| \Delta\|_{F}}{\lambda_K(M)  \sqrt{1-L_{t-1}^2} - \|\Delta\|_{F} } \\
     &\leq \frac{ \eta_n + \delta_n L_{t-1}}{\lambda_K(M)  \sqrt{1-L_{t-1}^2} - (\eta_n + \delta_n L_{t-1}) } \\
     &\leq \frac{ \eta_n + \delta_n L_{t-1}}{\lambda_K(M)/2 - (\eta_n + \delta_n L_{t-1}) } \\
    \end{split}
\end{equation*}
   where the last line follows by assuming that $L_{t-1} \leq \frac{1}{2} \quad \forall t \geq 0$ (we will show that this indeed holds).
By using a first-order Taylor expansion around $0$ for the function $f(x) = \frac{a+bx}{c-a-bx}$ for $x\in (0,1/2)$, we obtain:
$$ f(x) < \frac{a}{c -a} + \frac{bc}{(c - a - b/2)^2}x, \quad \text{for } x \in (0, 1/2).$$

        Therefore, seeing that 
        we have $\eta_n \geq \delta_n$ and letting $u = \frac{ \eta_n}{ \lambda_K(M)/2 - \eta_n}= \frac{ 2\eta_n}{ \lambda_K(M) - 2\eta_n}$ and $r=\frac{\lambda_K(M)/2 \delta_n }{(\lambda_K(M)/2 - \eta_n -\delta_n/2)^2}=\frac{2\lambda_K(M) \delta_n }{(\lambda_K(M) -2 \eta_n -\delta_n)^2}\leq \frac{2\lambda_K(M) \eta_n }{(\lambda_K(M) -3\eta_n)^2}$, we have:
    \begin{equation*}
    \begin{split}
      \|\sin \Theta (U, \wh U^t)\|_{F}  &\leq  u + r L_{t-1} \\
    \end{split}
    \end{equation*}

By Assumption 4, we have $\lambda_K(M) \geq c\sqrt{\frac{n}{K}}$.
Therefore, $\lambda_K(M) \geq 10\eta_n$ as soon as:
\begin{equation}\label{cond:N_1}
    \begin{split}
     n &\geq \frac{100C^2}{c^2} {\frac{K^2\log(n)}{N}}\left({n_{\C}}+\rho^2(\Gamma)s\lambda_{\max}(\Gamma)\right)\\
     \implies N&\geq \frac{100C^2}{c^2}\frac{K^2\log(n)}{n}\left({n_{\C}}+\rho^2(\Gamma)s\lambda_{\max}(\Gamma)\right)
    \end{split}
\end{equation}
which is satisfied under the condition (12) of $N$ in Theorem 2. Thus, in this setting:
\begin{equation}\label{eq:r}
    r \leq \frac{2\lambda_K(M) \eta_n }{(\lambda_K(M) -3\eta_n)^2} \leq \frac{2\lambda_K(M) \eta_n}{(\frac{7}{10}\lambda_K(M))^2} \leq \frac{200/49\eta_n}{\lambda_K(M)} \leq \frac{20}{49}  \leq \frac{1}{2}.
\end{equation}
and
\begin{equation*}\label{eq:u}
    u \leq \frac{2\eta_n}{\lambda_K(M)-2\eta_n} \leq \frac{5/2\eta_n}{\lambda_K(M)} \leq \frac{5}{20} = \frac{1}{4}
\end{equation*}

Also given that $L_{t-1} \leq \frac{1}{2}$,
\begin{equation*}\label{eq:L_t_bound}
    \|\sin \Theta (U, \wh U^t)\|_{F} \leq u + rL_{t-1} \leq \frac{5/2\eta_n+100/49\eta_n}{\lambda_K(M)} \leq \frac{1}{2}.
\end{equation*}

\paragraph{ Bound on $ \|\sin \Theta (V, \wh V^t)\|_{F} $}

By definition of the second step:
$$ \wh V^t = U_K (X^\top \wh U^t).$$



    By Theorem 1 of \cite{cai2018rate} (summarized for our use case in Lemma~\ref{lemma:perturbation}):
        \begin{equation*} 
    \begin{split}
        \|\sin \Theta (V, \wh V^t)\|_{F} 
        &\leq \frac{  \| P_{V_{\perp}} (M^\top(\widehat{U}^t - U) + Z^\top \widehat{U}^t )\|_{F}}{\lambda_{\min} (V^\top X^\top\widehat{U}^t)} \\
        &= \frac{ \| P_{V_{\perp}} (Z^\top \widehat{U}^t )\|_{F}}{\lambda_{\min} (V^\top X^\top\widehat{U}^t)}  \qquad \text{ since } P_{V_{\perp}} M^\top(\widehat{U}^t - U)  = 0
    \end{split}
    \end{equation*}

    We have:
\begin{equation*}
    \begin{split}
        \lambda_{\min} (V^\top X^\top\widehat{U}^t) &=   \lambda_{\min} (V^\top M^\top\widehat{U}^t + V^\top Z^\top\widehat{U}^t)  \\
        &\geq \lambda_{\min} (\Lambda U^{\top}\widehat{U}^t) -\lambda_{\max} ( V^\top Z^\top\widehat{U}^t)  \qquad \text{ (Weyl's inequality)} \\
        &= \lambda_{K}(M) \underbrace{\lambda_{\min}( U^{\top}\widehat{U}^t)}_{= \sqrt{1-L_t^2}} - \|V^\top Z^\top\widehat{U}^t\|_{F}    \\
        &\geq \lambda_{K}(M) \sqrt{1-L_t^2} -\|V^\top Z^\top\widehat{U}^t\|_{F}
    \end{split}
\end{equation*}

Thus, assuming that $L_t \leq \frac{1}{2}, \forall t$:
        \begin{equation*}
    \begin{split}
        \|\sin \Theta (V, \wh V^t)\|_{F} &\leq \frac{\| V_{\perp}^{\top} Z^{\top} \wh U^t\|_{F}}{\frac{1}{2}\lambda_{K}(M)-\| V^{\top} Z^\top\widehat{U}^t \|_{F}}.
    \end{split}
    \end{equation*}
Furthermore:
\begin{equation*}
    \begin{split}
        \| V^{\top}  Z^{\top}\wh U^t \|_{F} 
        &\leq \|V^{\top} Z^{\top} U U^{\top}\widehat{U}^t \|_{F} +  \|V^{\top} Z^{\top} U_{\perp} U_{\perp}^{\top}\widehat{U}^t \|_{F} \\
        &\leq \|V^{\top}  Z^{\top} U\|_{F}\| U^{\top}\widehat{U}^t \|_{op} +  \|V^{\top} Z^{\top} U_{\perp}\|_{op}\| U_{\perp}^{\top}\widehat{U}^t \|_{F}\\
         &\leq C K\sqrt{\frac{\log(n)}{N}} +  C \sqrt{\frac{Kn\log(n)}{N}}\|\sin \Theta (U, \wh U^t)\|_{F}
    \end{split}
\end{equation*}

    where the last inequality follows by noting that $\| U^{\top}\widehat{U}^t \|_{F} \leq 1$ and from Lemma~\ref{lemma:concentration_frob_norm_Utz}, which show that with probability at least $1-o(\frac{1}{n})$:
    $$ \| Z^\top {U} \|_{F}  \leq CK\sqrt{\frac{\log(n)}{N}}$$
    and  since $U_{\perp} \in \R^{n \times (n-K)}$:   $$ \ \| Z^\top {U}_{\perp}  \|_{op}  \leq C\sqrt{ K n \frac{\log(n)}{N}}. $$

    Therefore, using the same arguments as in the previous paragraph, using $\tilde{\eta}_n =C K\sqrt{\frac{\log(n)}{N}}$ and $\tilde{\delta}_n = C \sqrt{\frac{Kn\log(n)}{N}}$, we have:
    
$$ f(x) < \frac{a}{c -a} + \frac{bc}{(c - a - b/2)^2}x, \quad \text{for } x \in (0, 1/2).$$

        Therefore,  we have $\tilde{\eta}_n \leq \tilde{\delta}_n$, and letting $\tilde{u} = \frac{ \tilde{\eta}_n}{ \lambda_K(M)/2 - \tilde{\eta}_n}$  and $\tilde{r}=\frac{\lambda_K(M)/2 \tilde{\delta}_n }{(\lambda_K(M)/2 - \tilde{\eta}_n -\tilde{\delta}_n/2)^2}=\frac{2\lambda_K(M) \tilde{\delta}_n }{(\lambda_K(M) -2 \tilde{\eta}_n -\tilde{\delta}_n)^2}\leq \frac{2\lambda_K(M) \tilde{\delta}_n }{(\lambda_K(M) -3\tilde{\delta}_n)^2}$, 
    \begin{equation*}
    \begin{split}
      \|\sin \Theta (V, \wh V^t)\|_{F} &\leq  \tilde{u} + \tilde{r} \|\sin \Theta (U, \wh U^t)\|_{F} \leq \tilde{u}+\tilde{r}L_{t-1}\\
        \end{split}
        \end{equation*}

when $L_t$ decreases with each iteration. Again, we note that $\lambda_K(M) \geq 10\tilde{\delta}_n$ as soon as:
\begin{equation}\label{cond:N_2}
    \begin{split}
     n &\geq \frac{100C^2}{c^2} {\frac{K^2n\log(n)}{N}} \\
     \implies N &\geq \frac{100C^2}{c^2} {{K^2\log(n)}} 
    \end{split}
\end{equation}

which is satisfied under the condition (12) of $N$ in Theorem 2. Then we can show that,
\begin{equation}\label{eq:r2}
    \tilde{r} \leq \frac{2\lambda_K(M) \tilde{\delta}_n }{(\lambda_K(M) -3\tilde{\delta}_n)^2} \leq \frac{2\lambda_K(M) \tilde{\delta}_n}{(\frac{7}{10}\lambda_K(M))^2} \leq \frac{200/49\tilde{\delta}_n}{\lambda_K(M)} \leq \frac{1}{2}
\end{equation}
and
\begin{equation*}\label{eq:u2}
    \tilde{u} \leq \frac{2\tilde{\delta}_n}{\lambda_K(M)-2\tilde{\delta}_n} \leq \frac{5/2\tilde{\delta}_n}{\lambda_K(M)} \leq \frac{5}{20} = \frac{1}{4}
\end{equation*}

Also given that $L_{t-1} \leq \frac{1}{2}$,
\begin{equation*}\label{eq:L_t_bound2}
    \|\sin \Theta (V, \wh V^t)\|_{F} \leq \tilde{u} + \tilde{r}L_{t-1} \leq \frac{5/2\tilde{\delta}_n+100/49\tilde{\delta}_n}{\lambda_K(M)} \leq \frac{1}{2}
\end{equation*}

and 
\begin{equation*}
    \frac{\tilde{u}}{1-\tilde{r}} \leq \frac{3\tilde{\delta}_n}{\lambda_K(M)}\times \frac{\lambda_K(M)}{\lambda_K(M)-4\tilde{\delta}_n} \leq \frac{1}{2}.
\end{equation*}

Therefore, for all $t$,
$$ L_t \leq \frac{1}{2}.$$
\paragraph{ Behavior of $L_t$} $L_t$ is a decreasing function of $t$
 for $t\geq1$, and by Theorem 1, $L_0 \leq \frac{1}{2}$ (We later show in \eqref{eq:proof_init}). From the previous sections,
 \begin{equation*}
     \begin{split}
        \|\sin \Theta (U, \wh U^t)\|_{F} 
        &\leq \frac{5/2\eta_n}{\lambda_K(M)}+ \frac{200/49\delta_n}{\lambda_K(M)}L_{t-1}\\
        &\leq \frac{5/2C}{\lambda_K(M)}\sqrt{\frac{K\log(n)}{N}}\left(\sqrt{n_{\mathcal{C}}}+\rho(\Gamma)\sqrt{s\lambda_{\max}(\Gamma)}\right)\\
        &\qquad +\frac{200/49C}{\lambda_K(M)}\sqrt{\frac{K\log(n)}{N}}\rho(\Gamma)\sqrt{s\lambda_{\max}(\Gamma)}L_{t-1}\\
         \|\sin \Theta (V, \wh V^t)\|_{F} 
         &\leq \frac{5/2\tilde{\eta}_n}{\lambda_K(M)}+ \frac{200/49\tilde{\delta}_n}{\lambda_K(M)}L_{t-1}\\
         &\leq \frac{5/2C}{\lambda_K(M)}K\sqrt{\frac{\log(n)}{N}}+\frac{200/49C}{\lambda_K(M)}\sqrt{\frac{Kn\log(n)}{N}}L_{t-1}
     \end{split}
 \end{equation*}
  Thus,

   \begin{equation}\label{eq:geometric}
    \begin{split}
      L_t  &\leq  u + r L_{t-1} \\
      &\leq   u +r(u + r L_{t-2}) \\
      &\leq  r^t L_0 + u (1+ r +r^2 + \cdots r^{t-1})\\
      &\leq  r^t L_0 + u \frac{1-r^{t}}{1-r}
    \end{split}
    \end{equation}

where 
\begin{equation*}
    \begin{split}
        u & = \frac{5/2C}{\lambda_K(M)}\sqrt{\frac{K\log(n)}{N}}\left(\sqrt{n_{\mathcal{C}}}+\rho(\Gamma)\sqrt{s\lambda_{\max}(\Gamma)}\right) \\
        r & = \frac{200/49C}{\lambda_K(M)}\sqrt{\frac{K\log(n)}{N}}\left(\rho(\Gamma)\sqrt{s\lambda_{\max}(\Gamma)}\vee \sqrt{n}\right)
    \end{split}
\end{equation*}
where $r \leq \frac{1}{2},$ 
In particular, we want to find $t_{\max}$ such that $r^{t_{\max}} L_0$ becomes small enough to satisfy $r^{t_{\max}} L_0 \leq \frac{u}{1-r}$. Using $r \leq \frac{1}{2}$ (as previously shown) and that $L_0 \leq \frac{1}{2}$,
  \begin{equation*}
      \begin{split}
          \frac{r^{t_{\max}}}{2} & \leq \frac{u}{1-r}\\
          \implies t_{\max} &\geq \frac{-\log(2u)+\log(1-r)}{|\log(r)|}      \geq \frac{-2\log(2)-\log(u)}{\log(2)}
      \end{split}
  \end{equation*}
 Combining with the previous inequality (and since $\log(2)\leq \frac{1}{4}$) and the fact that under Assumption 4, we have $\lambda_K(M)\geq c\sqrt{n/K}$, we can choose $t_{\max}$ as,
\begin{equation*}
    t_{\max}=\left(2\log(nN) -4\log(\frac{5/2C}{c}) -4\log(K) -2\log(\log n) -4\log(\sqrt{n_{\mathcal{C}}}+\rho(\Gamma)\sqrt{s\lambda_{\max}(\Gamma)}) -2 \right)\vee 1
\end{equation*}

Thus, it is sufficient to choose $t_{\max}$ as,
\begin{equation}\label{eq:t_max}
    t_{\max}=2\log(\frac{nN}{K^2}) \vee 1
\end{equation}

Lastly, once $t_{\max}$ is chosen as \eqref{eq:t_max}, the bound on $L_{t_{\max}}$ in \eqref{eq:geometric} becomes,
\begin{equation}\label{eq:theorem3_main}
\begin{split}
     L_{t_{\max}} &\leq \frac{2u}{1-r} \leq 4u\\
     &= \frac{10C}{\lambda_K(M)}\sqrt{\frac{K\log(n)}{N}}\left(\sqrt{n_{\mathcal{C}}}+\rho(\Gamma)\sqrt{s\lambda_{\max}(\Gamma)} \right)\\
     &\leq \frac{10C}{c}K\sqrt{\frac{\log(n)}{nN}}\left(\sqrt{n_{\mathcal{C}}}+\rho(\Gamma)\sqrt{s\lambda_{\max}(\Gamma)} \right) 
\end{split}
\end{equation}
This concludes the proof.


\end{proof}

\subsection{Comparison with One-step Graph-Aligned denoising}

\input{appendixB_onestep}
\section{Analysis of the Estimation of $W$ and $A$}

In this section, we adapt the proof of \cite{klopp2021assigning} that derives a high probability bound for the outcome $\wh W$ after successive projections. We evaluate the vertices $\wh H$ detected by SPA with the rows of $\wh U$ as the input. To accomplish this, we first need to bound the row-wise error of $\wh U$ which is closely related to the upper bound of the estimated vertices $\wh H=\wh U_J$ and ultimately, is linked to $\wh W=\wh U \wh H^{-1}$. 

To apply Theorem 1 of \cite{gillis2015semidefinite} on the estimation with SPA, we need to show that the error on each of the row of the estimated left singular vector of $MV^{\top} = U \Lambda $ is controlled, which requires us bounding the error: $\| \wh U - U O\|_{2 \to \infty} = \max_{i \in [n] } \| e_i^\top (\wh U - U O)\|_2  $.

\input{appendixB_2_infty}

\subsection{Deterministic Bounds}

First, denote $\beta(M, \Gamma)$ as:

\begin{equation}\label{eq:beta}
    \beta(M, \Gamma):=\frac{C}{\lambda_K^2(M)}\sqrt{\frac{K\log(n)}{N}}\left(\sqrt{n_{\mathcal{C}}}+\rho(\Gamma)\sqrt{s\lambda_{\max}(\Gamma)}\right)
\end{equation}

which is the upper bound on the maximum row error, $\max_{i=1,\cdots, n}\|{e}_{i}^{T}( \wh U - UO)\|_{2}$ by \eqref{eq:two_infty_bound}. We need the following assumption on $\beta(M, \Gamma)$. 




\setcounter{assumption}{5}
\begin{assumption}
\label{assumption:beta}
    For a constant $\bar{C}>0$, we have
    \begin{equation*}
        \beta(M, \Gamma) \leq \frac{\bar{C}}{\lambda_1(W)\kappa(W)K\sqrt{K}}
    \end{equation*}
\end{assumption}

We will show in \eqref{eq:proof_beta} that this assumption holds in fact with high probability. Similar to \cite{klopp2021assigning}, the proof of the consistency of our estimator relies on the following result from \cite{gillis2015semidefinite}.

\begin{lemma}[Robustness of SPA (Theorem 1 of  \cite{gillis2015semidefinite}]\label{theorem:gillis}
    Let $M = WQ \in \mathbb{R}^{n \times K}$ where $Q \in \mathbb{R}^{K \times K}$ is non degenerate, and $W = [I_r| \tilde{W}^\top]^\top \in \R_+^{n \times K}$ is such that the sum of the entries of each row of $W$ is at most one.   Let $\tilde{M}$ denote a perturbed version of $M$, with $\tilde{M} = M +N$, with:
    $$ \| N_{j \cdot} \|_2 =\| e_j^TN\|_2  = \|\tilde{M}_{j \cdot} - M_{j \cdot}\| \leq \epsilon \text{ for all j}.$$

    Then, if $\epsilon$ is such that:\begin{equation}\label{eq:snmf-bound}
        \|{e}_i^{\top}N\|_2 \leq \epsilon \leq C_* \frac{\lambda_{\min}(Q)}{\sqrt{K}\kappa^2(Q)}
    \end{equation}
 for some small constant $C_*>0$, then SPA identifies the rows of $Q$ up to error $O(\epsilon \kappa^2(Q))$, that is, the  index set $J$ of vertices identified by SPA verifies:
    \begin{equation*}
    \max_{j \in J} \min_{\pi \in \mathcal{P}_K}  \| \tilde{M}_{j\cdot} - Q_{\pi(j)\cdot}\|_2 \leq C^{'} \kappa^2(Q) \epsilon.
    \end{equation*}
    The notation $\kappa(Q) = \frac{\lambda_{\max}(Q)}{\lambda_{\min}(Q)}$ is the condition number of $Q,$ and $\mathcal{P}_K$ denotes the set of all permutations of the set $[K]$.
    
\end{lemma}

\begin{lemma} [Adapted from Corollary 5 of \cite{klopp2021assigning}]
\label{lemma:SPAerrorbound}
    Let Assumptions 1 to 6 hold. Assume that $M \in \R^{n \times p}$ is a rank-\(K\)
    matrix. Let $U, \wh U \in \mathbb{R}^{n \times K}$ be the left singular vectors corresponding to the top \(K\) singular values of $M$ and its perturbed matrix $X \in \mathbb{R}^{n \times p}$, respectively. Let $J$ be the set of indices returned by the SPA with input $(\wh U,K)$, and $\wh H = \wh U_J$. Let $O \in \mathbb{O}_K$ be the same matrix as in (13) of the main manuscript. Then, for a small enough $\bar{C}$, there exists a constant $C^{'}>0$ and a permutation $\tilde{P}$ such that, 
    \begin{equation}
        \|\wh H - \tilde{P}HO\|_{F} \leq C^{'}\sqrt{K}\kappa(W)\beta(M, \Gamma)
    \end{equation}
    where $\beta(M, \Gamma)$ is defined as \eqref{eq:beta}.
\end{lemma}

\begin{proof}
The proof here is a direct combination of Lemma~\ref{theorem:gillis} and Corollary 5 of \cite{klopp2021assigning}, for SPA (rather than pre-conditioned SPA).
    The crux of the argument consists of applying Theorem 1 in \citep{gillis2015semidefinite} (rewritten in a format more amenable to our setting in Lemma~\ref{theorem:gillis}) which bounds the error of SPA in the near-separable nonnegative matrix factorization setting,
    \begin{equation*}
        \wh U = UO + N = WHO + N = WQ + N
    \end{equation*}
    where $Q = HO$ and $N \in \mathbb{R}^{n \times K}$ is the noise matrix. Note that the rows of $U$ lie on a simplex with vertices $Q$ and weights $W$.  We apply Lemma~\ref{theorem:gillis} with $Q = HO$, and $N =\widehat{U} - UO$. 


    Under Assumption 4,6 and Lemma~\ref{lemma:singular_value_H}, the error $\|{e}_i^{\top}N\|_2=\|{e}_i^{\top}(\wh U-UO)\|_2$ satisfies:
    \begin{equation*}
            \|\mathbf{e}_i^{\top}(\wh U-UO)\|_2 \leq \frac{\bar{C}}{\lambda_1(W)\kappa(W)K\sqrt{K}} \leq \frac{\bar{C}}{\lambda_1(W)K\sqrt{K}} \leq \frac{C_*\lambda_{\min}(HO)}{K\sqrt{K}}
    \end{equation*}
    for a small enough $\bar{C} \leq C_*$. Thus the condition on the noise (Equation \eqref{eq:snmf-bound}) in Theorem \ref{theorem:gillis} is met. Noting that $\wh H = \wh U_J$ and $\kappa(H)= \kappa(W)$, with the permutation matrix $\tilde{P}$ corresponding to $\pi$, we get
    \begin{equation*}
        \|\wh H - \tilde{P}HO\|_F \leq C^{'} \kappa^2(W)\beta(M, \Gamma)
    \end{equation*}

    
\end{proof}

We then readapt the proof of Lemma 2 from \cite{klopp2021assigning} with our new $\wh U$.

\begin{lemma} [Adapted from Lemma 2 of \cite{klopp2021assigning}]
\label{lemma:Werrorbound}
    Let the conditions of Lemma~\ref{lemma:SPAerrorbound} hold. Then $\wh H$ is non-degenerate and the estimated topic mixture matrix $\wh W = \wh U \wh H^{-1}$ satisfies,
    \begin{equation*}
        \min_{P \in \mathcal{P}} \|\wh W-WP\|_{F} \leq 2C^{'}\sqrt{K}\lambda_1^2(W)\kappa(W)\beta(M, \Gamma)+ \lambda_1(W) \|\wh U-UO\|_F
    \end{equation*}
     where $\mathcal{P}$ denotes the set of all permutations.
\end{lemma}

\begin{proof}
    The first part of the proof on the invertibility of $\wh H$ is analogous to Lemma 2 in \cite{klopp2021assigning}, where combined with Lemma~\ref{lemma:singular_value_H}, we obtain the inequality, 
    \begin{equation*}
        \lambda_{\min}(\wh H) \geq \frac{1}{2\lambda_1(W)}
    \end{equation*}
    and for the singular values of $H^{-1}$ and $\wh H^{-1}$:
\begin{equation*}\label{eq:H_inverse}
    \lambda_1(\wh H^{-1}) \leq 2\lambda_1(W) = 2\lambda_1(H^{-1})\
\end{equation*}
Using the result of Lemma~\ref{lemma:SPAerrorbound}, we have
    \begin{equation*}
        \begin{split}
            \|\wh W - WP\|_F 
            &= \|\wh U \wh H^{-1}-UH^{-1}P\|_F \\
            &\leq \|\wh U(\wh H^{-1}-O^{\top} H^{-1} P)\|_F + \|(\wh U-UO)[P^{-1}HO]^{-1}\|_F\\
            &\leq \|\wh H^{-1}\|_{op} \|H^{-1}\|_{op}\|\wh H-\tilde{P}HO\|_F + \|\wh U-UO\|_F\|H^{-1}\|_{op}\\
            &\leq 2C^{'}\sqrt{K}\lambda_1^2(W)\kappa(W)\beta(M,\Gamma)+ \lambda_1(W) \|\wh U-UO\|_F
        \end{split}
    \end{equation*}
where we used the well known inequality $\|A^{-1}-B^{-1}\|_F \leq \|A^{-1}\|_{op}\|B^{-1}\|_{op}\|A-B\|_F$.
\end{proof}

\subsection*{Proof of Theorem 3}
\label{sec:W_proof}
We are now ready to prove our main result.

\begin{proof}

We first show that the initialization error in Theorem 1 is upper bounded by $\frac{1}{2}$. Combining Assumption 4 and Lemma~\ref{lemma:singular_value_H}, we have 
\begin{equation*}\label{eq:lambda_K_W}
    \lambda_k(M) \geq c\lambda_1(W) \geq c\sqrt{n/K}
\end{equation*}

Then using the condition on $N$ (Equation 12), in Theorem 2,
\begin{equation}\label{eq:proof_init}
    \begin{split}
        \|\sin\Theta(V, \wh V^0)\|_F 
        &\leq \frac{C}{\lambda_K(M)^2}K\sqrt{\frac{n\log(n)}{N}} \\
        &\leq \frac{CK^2}{c^2n}\sqrt{\frac{n\log(n)}{N}} \\
        &\leq \frac{C}{c^2\sqrt{c^{*}_{\min}}(\sqrt{n_{\mathcal{C}}+\rho^2(\Gamma) s\lambda_{\max}(\Gamma)})}\leq \frac{1}{2}
    \end{split}
\end{equation}
 


Next from the condition on $N$ in Theorem 3 (Equation \eqref{eq:beta}), and Assumption 4,
\begin{equation}\label{eq:proof_beta}
    \beta(M, \Gamma) \leq \frac{C\sqrt{n}}{\sqrt{c^*_{\min}}K\sqrt{K}\lambda_K^2(M)} \leq \frac{\bar{C}}{\lambda_1(W)\kappa(W)K\sqrt{K}}
\end{equation}

which proves Assumption~\ref{assumption:beta}. Thus, we are ready to use Theorem 2 and Lemma~\ref{lemma:Werrorbound}. We can now plug in $\beta(M, \Gamma)$, the result of Equation \eqref{eq:beta} and the error bound of graph-regularized SVD (Equation (13) in Theorem 2) in the upper bound of $\min_{P \in \mathcal{P}} \|\wh W-WP\|_{F}$ formulated in Lemma~\ref{lemma:Werrorbound}. 
\begin{equation*}
        \begin{split}
            \|\wh W - WP\|_F 
            &\leq 2C^{'}\sqrt{K}\lambda_1^2(W)\kappa(W)\beta(M, \Gamma)+ \lambda_1(W) \|\wh U-UO\|_F \\
            &\leq \frac{2C^{'}C_1}{\sqrt{n}}\left( \frac{\lambda_1(W)}{\lambda_K(M)}\right)^2\kappa(W)K\sqrt{\frac{\log(n)}{N}} \left(\sqrt{n_{\mathcal{C}}}+\rho(\Gamma)\sqrt{s\lambda_{\max}(\Gamma)}\right) \\
            &\qquad + 10C_2\frac{\lambda_1(W)}{\lambda_K(M)} \sqrt{\frac{K\log(n)}{N}} \left(\sqrt{n_{\mathcal{C}}}+\rho(\Gamma)\sqrt{s\lambda_{\max}(\Gamma)}\right)\\
            &\leq \frac{2c^*C^{'}C_1}{c^2}K\sqrt{\frac{\log(n)}{N}} \left(\sqrt{n_{\mathcal{C}}}+\rho(\Gamma)\sqrt{s\lambda_{\max}(\Gamma)}\right) \\
            &\qquad +\frac{10C_2}{c} \sqrt{\frac{K\log(n)}{N}}\left(\sqrt{n_{\mathcal{C}}}+\rho(\Gamma)\sqrt{s\lambda_{\max}(\Gamma)}\right) \\
            &\leq \frac{20C}{c}K\sqrt{\frac{\log(n)}{N}} \left(\sqrt{n_{\mathcal{C}}}+\rho(\Gamma)\sqrt{s\lambda_{\max}(\Gamma)}\right)
        \end{split}
    \end{equation*}
    Here, we used the bounds on condition numbers in Assumption 4. 

\end{proof}
 
\subsection*{Proof of Theorem 4}

Using the result of Theorem 3, we now proceed to bound the error of $\wh A$.

\begin{proof}
By the simple basic inequality, letting $P = \arg  \min_{ O \in \mathcal{F}} \| \wh W - W O\|_F$, we get
\begin{equation*}
    \begin{split}
        \|X-\wh W \wh A \|_F^2 &\leq \|X-\wh W P^{-1}A\|_F^2 \\
        \|WA+Z-\wh W \wh A\|_F^2 &\leq \|WA+Z-\wh W P^{-1} A\|_F^2 \\
        \|(W- \wh W P^{-1})A+\wh W (P^{-1} A-\wh A)+Z\|_F^2 &\leq \|(W-\wh W P^{-1})A+Z\|_F^2 
        \end{split}
\end{equation*}
which leads us to,
\begin{equation*}
    \begin{split}
    \|\wh W(P^{-1} A-\wh A)\|_F^2 &\leq 2\< \wh W (\wh A-P^{-1}A), (W- \wh W P^{-1} )A+Z \> \\
    &= 2\< \wh W (\wh A-P^{-1}A), (W- \wh W P^{-1} )A \> + 2\< \wh W (\wh A-P^{-1}A), Z \>\\
    &\leq 2\|\wh W (\wh A-P^{-1}A)\|_F\| (W- \wh W P^{-1} )A \|_F \\
    &\qquad + 2 \|\wh W (\wh A-P^{-1}A) \|_F \max_{U \in \R^{n \times p}: \| U\|_F =1 } \< U, Z \>\\
    \end{split}
\end{equation*}
Plugging the upper bound on $\max_{U \in \R^{n \times p}: \| U\|_F =1 } \< U, Z \>$ which we prove below,

\begin{equation*}
    \begin{split}
 \|\wh W(P^{-1} A-\wh A)\|_F  &\leq  2 \| (W- \wh W P^{-1} )A \|_F  + C_2\sqrt{\frac{\log(n)}{N}}\\
 &\leq  2 \lambda_1(A) \| (W- \wh W P^{-1} ) \|_F  +  C_2\sqrt{\frac{\log(n)}{N}}
     \end{split}
\end{equation*}
\begin{equation*}
    \begin{split}
  &\|P^{-1} A-\wh A\|_F \\
 &\leq  \frac{1}{\lambda_{\min}(\wh W)}\left(2\lambda_1(A) \|W- \wh W P^{-1} \|_F +  C_2\sqrt{\frac{\log(n)}{N}}\right)\\
 &\leq \frac{1}{\lambda_K(W)-\|W- \wh W P^{-1} \|_{op}}\left(2\lambda_1(A)\|W- \wh W P^{-1} \|_F+C_2 \sqrt{\frac{\log(n)}{N}}\right) \qquad (*)\\
 &\leq 2C_1\lambda_1(A)\|W- \wh W P^{-1} \|_F+C_1C_2\sqrt{\frac{\log(n)}{N}}
    \end{split}
\end{equation*}

where in (*) we have used Weyl's inequality to conclude, 
\begin{equation*}
   {\lambda_{\min}(\wh W)} \geq {\lambda_K(W)-\|W- \wh W P^{-1} \|_F}.
\end{equation*}

Also, assume that $N$ is large enough so that the condition on $N$ in Theorem 2 holds. Combining Lemma~\ref{lemma:singular_value_H} and Theorem 3, $\|W- \wh W P^{-1} \|_F$ becomes small enough so that $ \|W- \wh W P^{-1} \|_F < 1 \leq \lambda_K(W)$. Thus, $\frac{1}{\lambda_K(W)-\|W- \wh W P^{-1} \|_{F}}\leq C_1$ for $C_1>1$.

By definition, $Z$ represents some centered multinomial noise, with each entry $Z_{i \cdot}$ being independent. Similar to proof of Lemma~\ref{lemma:concentration_frob_norm_Utz}, $\< U,Z \>$ can be represented as a sum of $nN$ centered variables:
\begin{equation*}
\begin{split}
     \< U,Z \> &= \text{tr}(U^{\top}Z) = \sum_{j=1}^p\sum_{i=1}^nU_{ij}Z_{ij}\\
     &=\frac{1}{N}\sum_{j=1}^p\sum_{i=1}^n\sum_{m=1}^NU_{ij}(T_{im}(j)-\E[T_{im}(j)])\\
     &= \frac{1}{N}\sum_{i=1}^n\sum_{m=1}^N\eta_{im} \quad \text{ with }  \eta_{im}= \sum_{j=1}^pU_{ij}(T_{im}(j)-\E[T_{im}(j)])
\end{split}
\end{equation*}

We have: 
\begin{equation*}
    \text{Var}(\sum_{i=1}^n \eta_{im}) = \sum_{i=1}^n \text{Var}(\sum_{j=1}^pU_{ij}T_{i m}(j)) = \sum_{i=1}^n \left( \sum_{j=1}^p U_{ij}^2 M_{ij} - (\sum_{j=1}^{p} U_{ij} M_{ij})^2 \right) \leq  1,
\end{equation*}

since $\sum_{i=1}^n \sum_{j=1}^p U_{ij}^2 = 1$ and thus:
$$ \sum_{m=1}^N \text{Var}( \sum_{i=1}^n\eta_{im}) \leq N. $$

Moreover, for each $i, m$, 
\begin{equation*}
    \begin{split}
      \sum_{m=1}^N |\sum_{i=1}^n\eta_{im}|^q
      &= N\left| \sum_{i=1}^n\sum_{j=1}^p U_{ij}(T_{im}(j)-\E[T_{im}(j)])\right|^q\\
      &\leq N(\sum_{i=1}^n\sum_{j=1}^p U_{ij}^2 \times  \sum_{i=1}^n\sum_{j=1}^p(T_{im}(j) -M_{ij})^2)^{\frac{q}{2}}\\
      &\leq N(\sum_{i=1}^n\sum_{j=1}^p(T_{im}(j)^2 + M_{ij}^2 - 2M_{ij}T_{im}(j)))^{\frac{q}{2}}\\
      &\leq N2^{\frac{q}{2}}\\
      &=2 N 2^{(q-2)/2} \\
      &< \frac{q!}{2} (4N) (\frac{2^{1/2}}{3})^{q-2}
    \end{split}
\end{equation*}

Thus, by Bernstein's inequality (Lemma~\ref{lemma: bernstein ineq0} with $v=4N$ and $c=\frac{\sqrt{2}}{3}$:
\begin{equation}
    \begin{split}
        \P[ |\frac{1}{N}\sum_{m=1}^{N}  \sum_{i=1}^n \eta_{im} | > t ]&\leq 2 e^{-\frac{N^2t^2/2}{4N + \sqrt{2}Nt/3}}=2 e^{-\frac{Nt^2/2}{4 + \sqrt{2}t/3}}
    \end{split}
\end{equation}

Choosing $t =C^* \sqrt{\frac{\log(n)}{N}}$:

\begin{equation}
    \begin{split}
        \P[ |\frac{1}{N}\sum_{m=1}^{N}  \sum_{i=1}^n \eta_{im} | > t ]&\leq 2 e^{-\frac{(C^*)^2\log(n)/2}{4 + \frac{C^* \sqrt{2}}{3}\sqrt{\frac{\log(n)}{N}}}}
    \end{split}
\end{equation}

Thus, with probability at least $1-o(n^{-1})$,  $|\<U,Z\>|\leq C^*\sqrt{\frac{\log(n)}{N}}.$

Lastly, we use the fact, $\lambda_1(A) \leq \|A\|_F \leq \sqrt{K}$ to get the final bound of $A$,

$$\|\wh A-P^{-1} A\|_F \leq CK^{3/2}\sqrt{\frac{\log(n)}{N}} \left(\sqrt{n_{\mathcal{C}}}+\rho(\Gamma)\sqrt{s\lambda_{\max}(\Gamma)}\right)$$.

where the error is controlled by the first term, $2C_1\lambda_1(A)\|W- \wh W P^{-1} \|_F$.
\end{proof}

%% file: appendixB_onestep.tex
\label{sec:one-step}

We also propose a fast one-step graph-aligned denoising of the matrix $X$ that could be an alternative of the iterative graph-aligned SVD in Step 1 of Section 2.3 of the main manuscript. We denoise the frequency matrix $X$ by the following optimization problem,

\begin{equation}\label{eq:init_denoising}
    \widehat{M} = \text{argmin}_{M \in \R^{n \times p}} \| X - M \|_F^2+ \rho \|\Gamma M\|_{21}
\end{equation}

A SVD on the denoised matrix $\wh M$ yields estimates of the singular values $U$ and $V$. 
Through extensive experiments with synthetic data, we find that one-step graph-aligned denoising provides more accurate estimates than pLSI but still falls short compared to the iterative graph-aligned denoising (GpLSI). We provide a theoretical upper bound on its error as well as its comparison to the error of pLSI where there is no graph-aligned denoising.

\begin{algorithm}
\setstretch{1.35}
\caption{One-step Graph-aligned denoising} 
\label{algo:init}
\begin{algorithmic}[1]
\State \textbf{Input:} Observation $X$, incidence matrix $\Gamma$
\State \textbf{Output:} Denoised singular vectors $\wh U$ and $\wh V$.
\State 1. Graph denoising on $X$ with MST-CV: $\tilde{M} = \arg \min_{M\in \R^{n \times p}}\lVert X - M \rVert_F^2 + \hat{\rho}\lVert \Gamma M \rVert_{21}$
\State 2. Perform the rank-$K$ SVD of $\tilde{M}$: $\tilde{M} \approx \wh U \wh \Lambda \wh V$
\end{algorithmic}
\end{algorithm}

We begin by analyzing the one-step graph-aligned denoising, as proposed in Algorithm~\ref{algo:init}. 
We begin by reminding the reader that, in our proposed setting, the observed word frequencies in each document are assumed to follow a  ``signal + noise" model, $X = M  + Z$
where the true probability $M$ is assumed to admit the following SVD decomposition:
\begin{equation*}\label{eq:def_M}
M = \E[X] = U \Lambda V^{\top}.
\end{equation*}

\begin{theorem}
Let the conditions of Theorem 3 hold. Let $\wh U$ and $\wh V$ be given as estimators obtained from Algorithm~\ref{algo:init}. Then, there exists a constant $C>0$, such that with probability at least $1-o(n^{-1})$,
\begin{equation}\label{eq:init_err}
    \max \{\|\sin \Theta (U, \wh U)\|_F, \|\sin \Theta (V, \wh V)\|_F\} \leq CK\sqrt{\frac{\log(n)}{nN}}\left ( \sqrt{n_C} + \rho(\Gamma) \sqrt{s}\sqrt{\lambda_{\max}(\Gamma)} \right )
\end{equation}
\end{theorem}

\begin{proof}

Let $\widehat M$ be the solution of \eqref{eq:init_denoising}:
\begin{equation*}
        \widehat{M} = \text{argmin}_{M \in \R^{n \times p}} \| M - X \|_F^2+ \rho \|\Gamma M\|_{21}
\end{equation*}

Let $\Delta = \wh M -M$, and $Z = X-M$. By the basic inequality, we have:
\begin{equation}\label{eq:init_BI}
    \begin{split}
        \| \wh M - X\|_{F}^2 + \rho \| \Gamma \widehat{M}\|_{21} &  \leq       \| {M} - X\|_{F}^2 + \rho \| \Gamma {M}\|_{21}\\
             \| \widehat{M} - M\|_{F}^2 &  \leq      2 \<X-M, \widehat{M} - M \>+ \rho \| \Gamma {M}\|_{21} -  
             \rho \| \Gamma \widehat{M}\|_{21} \\
             &  =     2 \<Z,  (\Pi + \Gamma^{\dagger} \Gamma)\Delta\>+ \rho \| \Gamma {M}\|_{21} -  
             \rho \| \Gamma  \Delta + \Gamma M\|_{21}   \\
        &  \leq  2 \underbrace{  \<\Pi Z ,  \Pi\Delta \>}_{(A)} +   \underbrace{2\<(\Gamma^{\dagger})^TZ,   \Gamma \Delta \>+ \rho (\| (\Gamma \Delta)_S\|_{21} -  \| (\Gamma \Delta)_{S^c}\|_{21})}_{(B)} \\
    \end{split}
\end{equation}
where $S = \text{supp}(\Gamma W)$ and in the penultimate line, we have used the decomposition of $\R^n$ on the two orthogonal subspaces: $\R^n = \text{Im}(\Pi) \bigoplus^{\perp} \text{Im}(\Gamma^\dagger \Gamma),$ so that:
\begin{equation*}\label{eq:decompose_gamma}
    \forall x\in \R^n,\qquad x = \Pi x + \Gamma^\dagger \Gamma x
\end{equation*}
We proceed by characterizing the behavior of each of the terms (A) and (B) in the final line of \eqref{eq:init_BI} separately.

\paragraph{ Concentration of (A).}
By Cauchy-Schwarz, it is immediate to see that:
\begin{equation*}
    \<\Pi Z ,  \Pi\Delta \> \leq \|\Pi Z \|_F  \| \Pi\Delta \|_F.
\end{equation*}

By Lemma~\ref{lemma:concentration_pi_z}, with probability at least $1-o(n^{-1})$:
\begin{equation*}
    (A) \leq 2 \sqrt{ C_1Kn_C{ \frac{\log(n) 
    }{ N}}} \| \Delta\|_{F} 
\end{equation*}

\paragraph{ Concentration of (B).}
We have:
\begin{equation*}
    \begin{split}
        2\<(\Gamma^{\dagger})^TZ,   \Gamma \Delta \> &+ \rho (\| (\Gamma \Delta)_{S\cdot}\|_{21} -  \| (\Gamma \Delta)_{S^c\cdot}\|_{21}) \\ 
        &\leq   2\max_{e\in \mathcal{E}}\|[(\Gamma^{\dagger})^TZ]_{e\cdot}\|_2  \times \|  \Gamma \Delta \|_{21}+\rho (\| (\Gamma \Delta)_{S\cdot}\|_{21} -  \| (\Gamma \Delta)_{S^c\cdot}\|_{21})
    \end{split}
\end{equation*}
Let $\mathcal{A}$ denote the event: $ \mathcal{A} = \{\rho \geq 4 \max_{e\in \mathcal{E}}\|[(\Gamma^{\dagger})^TZ]_{e\cdot}\|_2 \}$. 
By Lemma~\ref{lem:concentration_Gamma_dagger_Z},  for a choice of $\rho = 4 C_2 \rho(\Gamma) \sqrt{\frac{K \log(n)}{N}}$, then $\P[\mathcal{A}] \geq 1-o({n}^{-1})$.

Then, on $\mathcal{A}$, we have:
\begin{equation}
    \begin{split}
        2\<(\Gamma^{\dagger})^TZ,   \Gamma \Delta \>+ \rho (\| (\Gamma \Delta)_{S\cdot}\|_{21} -  \|(\Gamma \Delta)_{S^c\cdot}\|_{21}) &\leq \frac{3\rho}{2}\| (\Gamma \Delta)_{S\cdot}\|_{21} -  \frac{\rho}{2}\| (\Gamma \Delta)_{S^c\cdot}\|_{21}
    \end{split}
\end{equation}

\paragraph{Concentration}
We thus have:
\begin{equation*}\label{eq:init_BI2}
    \begin{split}
 \| \Delta\|_{F}^2 &  \leq    4 \sqrt{ C_1Kn_C { \frac{\log(n) 
    }{ N}}}  \| \Delta\|_{F}  +   \frac{3\rho}{2}\| (\Gamma \Delta)_{S\cdot}\|_{21} \\
    &  \leq   4 \sqrt{ C_1Kn_C { \frac{\log(n) 
    }{ N}}} \| \Delta\|_{F}  +   \frac{3\rho}{2} \sqrt{s}\|  \Gamma \Delta\|_{F} \\
     &  \leq   4 \sqrt{ C_1Kn_C { \frac{\log(n) 
    }{ N}}} \| \Delta\|_{F}  +   \frac{3\rho}{2} \sqrt{s} \sqrt{\lambda_{\max}(\Gamma)}\| \Delta\|_{F}
    \end{split}
    \end{equation*}
    \begin{equation*}
        \begin{split}
 \| \Delta\|_{F}&    \leq   4 \sqrt{ C_1Kn_C { \frac{\log(n) 
    }{ N}}}   +   \frac{3\rho}{2} \sqrt{s}\sqrt{\lambda_{\max}(\Gamma)}\\
     \| \Delta\|_{F}&    \leq   4 \sqrt{ C_1Kn_C { \frac{\log(n) 
    }{ N}}}   +  6 \rho(\Gamma)  C_2 \sqrt{\frac{K \log(n)}{N}} \sqrt{s}\sqrt{\lambda_{\max}(\Gamma)}\\
    & \leq C  \left ( \sqrt{n_C} + \rho(\Gamma) \sqrt{s}\sqrt{\lambda_{\max}(\Gamma)} \right )\sqrt{\frac{K \log(n)}{N}}
    \end{split}
\end{equation*}


Then by applying Wedin's sin$\Theta$ theorem \citep{wedin1972perturbation},
\begin{equation*}\label{eq:init_err_U}
\begin{split}
    \|\sin \Theta (U, \wh U)\|_F 
    &\leq \frac{\max\{\|(M-\wh M)V\|_F, \|U^{\top}(M-\wh M)\|_F\}}{\lambda_K(M)}\\
    &\leq \frac{C}{\lambda_K(M)}\sqrt{\frac{K\log(n)}{N}}\left ( \sqrt{n_C} + \rho(\Gamma) \sqrt{s}\sqrt{\lambda_{\max}(\Gamma)} \right )
\end{split}
\end{equation*}




The derivation for $\wh V$ is symmetric, which leads us to the final bound,
\begin{equation*}
\begin{split}
     \max \{\|\sin \Theta (U, \wh U)\|_F, \|\sin \Theta (V, \wh V)\|_F\} &\leq \frac{C}{\lambda_K(M)}\sqrt{\frac{K\log(n)}{N}}\left ( \sqrt{n_C} + \rho(\Gamma) \sqrt{s}\sqrt{\lambda_{\max}(\Gamma)} \right )\\
     &\leq CK\sqrt{\frac{\log(n)}{nN}}\left ( \sqrt{n_C} + \rho(\Gamma) \sqrt{s}\sqrt{\lambda_{\max}(\Gamma)}\right)
\end{split}
\end{equation*}

This concludes our proof.
\end{proof}

We observe first that the error bound for one-step graph-aligned denoising has better rate than the one with no regularization, $O(K\sqrt{ \log(n)/N})$, provided in \cite{klopp2021assigning}. We also note that the rate of one-step denoising and GpLSI is equivalent up to a constant. Although the dependency of the error on parameters $n,p,K,$ and $N$ is the same for both methods, our empirical studies in Section 3 reveal that GpLSI still achieves lower errors compared to one-step denoising.

%% file: appendixB_2_infty.tex
\begin{lemma}[Baseline two-to-infinity norm bound (Theorem 3.7 of \cite{cape2019two}]\label{theorem:cape}
For $C, E\in \mathbb{R}^{n \times p}$, denote $\wh C := C+E$ as the observed matrix that adds perturbation $E$ to unobserved $C$. For $C$ and $\wh C$, their respective singular value decompositions of are given as 
\begin{equation*}
    \begin{split}
        C&= U\Lambda V^{\top}+U_{\perp}\Lambda_{\perp} V_{\perp}^{\top}\\
        \wh C &= \wh U\wh \Lambda \wh V^{\top}+\wh U_{\perp}\wh \Lambda_{\perp} \wh V_{\perp}^{\top}
    \end{split}
\end{equation*}
where $\Lambda, \wh \Lambda \in \mathbb{R}^{r \times r}$ contain the top $r$ singular values of $C, \wh C$, while $\Lambda_{\perp}, \wh \Lambda_{\perp} \in \mathbb{R}^{n-r \times p-r}$ contain the remaining singular values. Provided \(\lambda_r(C) > \lambda_{r+1}(C) \geq 0\) and \(\lambda_r(C) \geq 2\|E\|_{op}\), then,
    \begin{align}\label{eq:cape_decomp}
        \|\hat{U} - U W_U\|_{2 \to \infty} &\leq 2 \left( \frac{\|(U_\perp U_\perp^\top) E (V V^\top)\|_{2 \to \infty}}{\sigma_r(C)} \right) \nonumber \\
        &\quad + 2 \left( \frac{\|(U_\perp U_\perp^\top) E (V_\perp V^\top)\|_{2 \to \infty}}{\sigma_r(C)} \right) \|\sin \Theta (\hat{V}, V)\|_{op} \nonumber \\
        &\quad + 2 \left( \frac{\|(U_\perp U_\perp^\top) C (V_\perp V^\top)\|_{2 \to \infty}}{\sigma_r(C)} \right) \|\sin \Theta (\hat{V}, V)\|_{op} \nonumber \\
        &\quad + \|\sin \Theta (\hat{U}, U)\|_{op}^2 \|U\|_{2 \to \infty}.
    \end{align}
    where $W_U$ is the solution of $\inf_{W\in \mathbb{O}_r}\|\wh U-UW\|_F$.
\end{lemma}

\begin{proof}
    The proof is given in Section 6 of \cite{cape2019two}.
\end{proof}

Adapting Lemma~\ref{theorem:cape} to our setting, we set $C=U\Lambda$ and $\wh C=\bar{U}^{t_{\max}}$. We also use the notation $C^*$ to denote the oracle, $C^* = U \Lambda V^T \hat{V}_{t_{\max}-1}.$  We set $r=K$ and denote the $SVD$s of $C^*$ and $\hat{C}$:

Note that we are performing a rank $K$ decomposition, so in the previous SVD, $V_{\perp}=0$, by which \eqref{eq:cape_decomp} becomes,


\begin{equation}\label{eq:cape_adapt}
    \begin{split}
        \|\hat{U} - U O\|_{2 \to \infty} &\leq 2 \left( \frac{\|(U_\perp U_\perp^\top) E (V V^\top)\|_{2 \to \infty}}{\lambda_K(M)} \right) \nonumber \\
        &\quad + \|\sin \Theta (\hat{U}, U)\|_{op}^2 \|U\|_{2 \to \infty}.
    \end{split}
\end{equation}

with $E = \hat{C} - C^* + C^* - C.$
We thus need to bound the quantity, 
$$\|(U_\perp U_\perp^\top) E (V V^\top) \|_{2 \to \infty}  \leq \|P_{U^{\perp}}E \|_{2, \infty} \leq \underbrace{\| P_{U^{\perp}}(\hat{C} - C^*)\|_{2, \infty}}_{(A)} +  \underbrace{\| P_{U^{\perp}}(C^*-C)\|_{2, \infty}}_{(B)} .$$

The second term (B) is 0 because $P_{U^{\perp}}(C^*-C)=P_{U^{\perp}}(U\Lambda V^T \wh V_{t_{\max}-1} - U\Lambda)=P_{U^{\perp}}U\Lambda(V^T \wh V_{t_{\max}-1}-I_K)=0$. For the first term (A), we note that $\hat{U}$ stems from the SVD of $\wh C=\bar U_{t_{\max}}$, which is itself the solution of:
$$  \wh C = \text{argmin}_{C \in \R^{n \times K}} \frac{1}{2}\| C - \tilde{Y} \|_F^2 + \rho \sum_{i \sim j} \| C_{i \cdot} - C_{j \cdot}\|_2 $$
where $\tilde{Y} = X \wh V_{t_{\max}-1}$. By the KKT conditions, for any $i \in [n]$, 
$$\wh C_{i\cdot}-\tilde{Y}_{i\cdot}+\rho\sum_{j\sim i}z_{ij}=0$$

where $z_{ij}$ denotes the subgradient of $\|C_{i\cdot}-C_{j\cdot}\|_2$ so that $z_{ij}(k)=\frac{\hat C_{ik} - \hat C_{jk}}{\|\hat C_{i}-\hat C_{j}\|_2}$ and $\| z_{ij}\|_2 <1$ if $\hat C_{i\cdot} = \hat C_{j \cdot}$ and $\|z_{ij}\|_2  =1$ if $\hat C_{i} -\hat C_{j} \neq 0.$

Therefore:
\begin{equation}
    \begin{split}
    \hat{C}_{i\cdot} - C^*_{i\cdot} &=  \tilde{Y}_{i\cdot} - C^*_{i\cdot}  + \rho \sum_{j \sim i} z_{ij}     \\
    \implies  
    \|  \hat{C}_{i\cdot} - C^*_{i\cdot}\|_2 &\leq \| \tilde{Y}_{i\cdot} - C_{i \cdot}^*\|_2 + \rho \sum_{j \sim i} \| z_{ij}\|_2\\
 &\leq  \| \tilde{Y}_{i\cdot} - C_{i \cdot}^*\|_2 + \rho  d_{\max}\\
 \implies \|  \hat{C} - C^*\|_{2\to \infty} &\leq \| Z \hat{V}_{t_{\max}-1} \|_{2\to \infty}  + \rho d_{\max}.
    \end{split}
\end{equation}

We have:
\begin{equation}
    \begin{split}
         \| Z \hat{V}_{t_{\max}-1} \|_{2\to \infty} &=  \| Z (V V^T + V_{\perp} V_{\perp}^T) \hat{V}_{t_{\max}-1} \|_{2\to \infty}\\
         &\leq  \| Z V V^T  \hat{V}_{t_{\max}-1} \|_{2\to \infty} +\| Z V_{\perp} V_{\perp}^T \hat{V}_{t_{\max}-1} \|_{2\to \infty}\\
         &\leq  \| Z V\|_{2\to \infty} \|V^T  \hat{V}_{t_{\max}-1} \|_{op} +\| Z V_{\perp}\|_{2\to \infty}  \|V_{\perp}^T \hat{V}_{t_{\max}-1} \|_{op}\\
    \end{split}
\end{equation}
Then, by Lemma~\ref{lemma:concentration_max_norm_zV}, we have
\begin{equation*}
    \begin{split}
        \| Z \hat{V}_{t_{\max}-1} \|_{2\to \infty} &\leq \| Z V \|_{2\to \infty} + \| Z \|_{2\to \infty}L_{t_{\max-1}}\\
        &\leq C_1\sqrt{\frac{K\log(n)}{N}}+ C_2\sqrt{\frac{K\log(n)}{N}}L_{t_{\max}-1}\\
    \end{split}
\end{equation*}

Using the derivation of the geometric series of errors in \eqref{eq:geometric}, recall
\begin{equation*}
    \begin{split}
        u & = \frac{5/2C}{\lambda_K(M)}\sqrt{\frac{K\log(n)}{N}}\left(\sqrt{n_{\mathcal{C}}}+\rho(\Gamma)\sqrt{s\lambda_{\max}(\Gamma)}\right) \\
        r & = \frac{200/49C}{\lambda_K(M)}\sqrt{\frac{K\log(n)}{N}}\left(\rho(\Gamma)\sqrt{s\lambda_{\max}(\Gamma)}\vee \sqrt{n}\right)
    \end{split}
\end{equation*}
and expressing $L_{t_{\max}-1}$ in terms of $u$ and $v$,
\begin{equation*}
    \begin{split}
        \| Z \hat{V}_{t_{\max}-1} \|_{2\to \infty} &\leq C_1\sqrt{\frac{K\log(n)}{N}}+ C_2\sqrt{\frac{K\log(n)}{N}}\cdot u\frac{1-r^{t_{\max}-1}}{1-r}+\sqrt{\frac{K\log(n)}{N}}\cdot r^{t_{\max}-1}L_0\\
        &\leq C_1\sqrt{\frac{K\log(n)}{N}}+ C_2\sqrt{\frac{K\log(n)}{N}}\frac{u}{1-r}+\sqrt{K}r^{t_{\max}}L_0\\
        &\leq C_3\sqrt{\frac{K\log(n)}{N}}\frac{u}{1-r} \qquad \text{since } \frac{u}{1-r}\geq r^{t_{\max}}L_0 \\
        &\leq \frac{C^{'}}{\lambda_K(M)}\sqrt{\frac{K\log(n)}{N}}\left(\sqrt{n_{\mathcal{C}}}+\rho(\Gamma)\sqrt{s\lambda_{\max}(\Gamma)}\right) 
    \end{split}
\end{equation*}

Plugging this into \eqref{eq:cape_adapt},
\begin{equation}\label{eq:two_infty_bound}
    \begin{split}
        \|\hat{U} - U O\|_{2 \to \infty} &\leq 2\left( \frac{\| Z \hat{V}_{t_{\max}-1} \|_{2\to \infty}  + \rho d_{\max}}{\lambda_K(M)} \right) \nonumber + \|\sin \Theta (\hat{U}, U)\|_{op}^2 \|U\|_{2 \to \infty}\\
        &\leq \frac{C}{\lambda_K^2(M)}\sqrt{\frac{K\log(n)}{N}}\left(\sqrt{n_{\mathcal{C}}}+\rho(\Gamma)\sqrt{s\lambda_{\max}(\Gamma)}\right)
    \end{split}
\end{equation}
when $\rho$ is a small value that satisfies $\rho\leq\frac{1}{\lambda_K(M)}\sqrt{\frac{K\log(n)}{N}}\cdot \frac{\sqrt{n_{\mathcal{C}}}+\rho(\Gamma)\sqrt{s\lambda_{\max}(\Gamma)}}{d_{\max}}$, and we know $\|\sin \Theta (\hat{U}, U)\|_{op}^2 \|U\|_{2 \to \infty}\leq \|\sin \Theta (\hat{U}, U)\|_{op}^2\leq \|\sin \Theta (\hat{U}, U)\|_{op}$.

%% file: appendixC_technical_lemmas.tex
Let matrices $X, M, Z,W,\text{ and } A$ be defined as (2) in the main manuscript. In this section, we provide inequalities on the singular values of the unobserved quantities $W, M, \text{ and }H$, perturbation bounds for singular spaces, as well as concentration bounds on noise terms, which are useful for proving our main results in Section 3.

\subsection{Inequalities on Unobserved Quantities}

\begin{lemma}\label{lemma:bounds_sing_values_M}
For the matrix $M$, the following inequalities hold:
\begin{equation*}
        \lambda_K(M) \leq \sqrt{n} \times \min_{j \in [p]} \sqrt{h_j}
\end{equation*}
   \begin{equation*}
        \lambda_1(M) \leq \sqrt{n}
\end{equation*} 
\end{lemma}

\begin{proof}
    We observe that for each $j \in [p]$, the variational characterization of the smallest eigenvalue of the matrix $M^\top M/n$ yields:
    \begin{equation*}
    \begin{split}
        \lambda_K(\frac{M^\top M}{n}) &\leq [\frac{M^\top M}{n}]_{jj} \\
        &= \frac{1}{n} \sum_{i=1}^n M_{ij}^2 \\
                &\leq \frac{1}{n} \sum_{i=1}^n M_{ij} \\
                &\leq \sqrt{h_j}
                    \end{split}
    \end{equation*}
since $$\frac{1}{n}\sum_{i=1}^n M_{ij} = \frac{1}{n}\sum_{i=1}^n\sum_{k=1}^K W_{ik} A_{kj} \leq  \| \frac{1}{n} \sum_{i=1}^n W_{i\cdot }\|_2\|A_{\cdot j}\|_2\leq \sqrt{h_j}.  $$

    Similarly:
        \begin{equation*}
    \begin{split}
        \lambda_1(\frac{M^\top M}{n}) &\leq Tr( \frac{M^\top M}{n}) \\
        &=\sum_{j=1}^p \frac{1}{n} \sum_{i=1}^n M_{ij}^2 \\
         &\leq \frac{1}{n} \sum_{i=1}^n  \sum_{j=1}^p  M_{ij} \\
                &\leq 1
                    \end{split}
    \end{equation*}
\end{proof}

We also add the following lemma from \cite{klopp2021assigning} to make this manuscript self-contained.
\begin{lemma}[Lemma 6 from the supplemental material of \cite{klopp2021assigning}]\label{lemma:singular_value_H}
Let Assumption 2 be satisfied. For the matrices $W$, $H$, $\wh H$ defined in (6) and (7) of the main manuscript, we have 
\begin{equation}\label{eq:sing_val_W_lower_bound}
    \lambda_K(W)\geq 1, \qquad \lambda_1(W) \geq \sqrt{n/K}
\end{equation}
and 
\begin{equation}\label{eq:sing_val_WH}
    \lambda_1(H) = \frac{1}{\lambda_K(W)}, \qquad \lambda_K(H) = \frac{1}{\lambda_1(W)}, \qquad \kappa(H) = \kappa(W)
\end{equation}
\end{lemma}

\subsection{Matrix Perturbation Bounds}
\label{subsec:perturbation}
In this section, we provide rate-optimal bounds for the left and right singular subspaces. While the original Wedin's perturbation bound \citep{wedin1972perturbation} treats the singular subspaces symmetrically, work by \cite{cai2018rate} provides sharper bounds for each subspace individually. This refinement is particularly relevant in our setting where an additional denoising step of the left singular subspace leads to different perturbation behaviors of left and right singular subspaces as iterations progress.

Consider the SVD of an approximately rank-$K$ matrix $M \in \mathbb{R}^{n \times K} (n>K)$,

\begin{equation}\label{eq:svd_rectangular}
    M = \begin{bmatrix}
        U & U_{\perp} \\ 
         \end{bmatrix}
        \begin{bmatrix}
         \Lambda\\\mathbf{0}
        \end{bmatrix}V^{\top} 
\end{equation}

where $U \in \mathbb{O}^{n \times K}$, $U_{\perp} \in \mathbb{O}^{n \times (n-K)}$, $\Lambda \in \mathbb{R}^{K \times K}$, and $V \in \mathbb{O}^{K \times K}$. 

Let $X = M + Z$ be a perturbed version of $M$ with $Z$ denoting a  perturbation matrix. We can write the SVD of $X$ as:

\begin{equation}\label{eq:svd_rectangular_X}
    X = \begin{bmatrix}
        \wh U & \wh U_{\perp} \\ 
        \end{bmatrix}
        \begin{bmatrix}
         \wh \Lambda\\\mathbf{0}
        \end{bmatrix}\wh V^{\top} 
\end{equation}
where $\wh U$, $\wh U_{\perp}$, $\wh \Lambda$, $\wh V$ have the same structures as $U$, $U_{\perp}$, $\Lambda$, $V$. We can decompose $Z$ into two parts,

\begin{equation}\label{eq:decompose_Z}
    Z = Z_1 + Z_2 = P_UZ + P_{U_{\perp}}Z
\end{equation}

\begin{lemma}[Adapted from Theorem 1 of \cite{cai2018rate}]\label{lemma:perturbation}
    Let $M$, $X$, and $Z$ be as given in Equations \eqref{eq:svd_rectangular}-\eqref{eq:decompose_Z}. Then:
    \begin{equation}\label{bound_U}
    \begin{split}
            \|\sin\Theta(U, \wh U) \|_{op} &\leq \frac{\|Z_2\|_{op}}{\lambda_{\min}(U^{\top}XV)} \wedge 1\\
        \|\sin\Theta(U, \wh U) \|_F &\leq \frac{\|Z_2\|_F}{\lambda_{\min}(U^{\top}XV)} \wedge \sqrt{p}
        \end{split}
    \end{equation}
    \begin{equation}\label{bound_V}
        \begin{split}
                \|\sin\Theta(V, \wh V) \|_{op} &\leq \frac{\|Z_1\|_{op}}{\lambda_{\min}(U^{\top}XV)} \wedge 1\\
        \|\sin\Theta(V, \wh V) \|_F &\leq \frac{\|Z_1\|_F}{\lambda_{\min}(U^{\top}XV)} \wedge \sqrt{p}
                \end{split}
    \end{equation}
\end{lemma}

\begin{proof}
    This result is a simplified version of the original theorem under the setting $rank(M) = K$.
\end{proof}

\subsection{Concentration Bounds}
We first introduce the general Bernstein inequality and its variant which will be used for proving high probability bounds for noise terms in Section~\ref{sec:tech_lemmas}.

\label{subsec:concentration}
\subsubsection{General Inequalities}
\begin{lemma}[Bernstein inequality (Corollary 2.11, \cite{boucheron2013concentration})]\label{lemma: bernstein ineq0}
Let $X_1, \dots, X_n$ be independent random variables such that there exists positive numbers $v$ and $c$ such that $\sum_{i=1}^n \E[X_i^2] \leq v$ and 
\begin{equation}\label{bernstein_condition}
\sum_{i=1}^n \E[ (X_i)^q_+]\leq \frac{q!}{2} vc^{q-2}    
\end{equation} for all integers $q \geq 3$. Then for any $t > 0$, 
$$ \mathbb{P}\left(\left|\sum_{i=1}^n X_i\right| \geq t \right) \leq 2 \exp\left(-\frac{t^2/2}{v + ct}\right)$$
\end{lemma}

A special case of the previous lemma occurs when all variables are bounded by a constant $b$, by taking $v = \sum_{i=1}^n \E[X_i^2]$ and $c = b/3$.

\begin{lemma}[Bernstein inequality for bounded variables (Theorem 2.8.4, \cite{vershynin2018high})]\label{lemma: bernstein ineq} 
Let $X_1, \dots, X_n$ be independent random variables with $|X_i|\leqslant b$, $\mathbb{E}[X_i] = 0$ and $\text{Var}[X_i] \leq \sigma_i^2$ for all $i$. Let $\sigma^2 := n^{-1}\sum_{i=1}^n \sigma_i^2$. Then for any $t > 0$, 
$$ \mathbb{P}\left(n^{-1}\left|\sum_{i=1}^n X_i\right| \geq t \right) \leq 2 \exp\left(-\frac{nt^2/2}{\sigma^2 + bt/3}\right)$$
\end{lemma}

\subsubsection{Technical Lemmas}
\label{sec:tech_lemmas}

\begin{lemma}[Concentration of the cross-terms $Z_i^TZ_j$]\label{lem:concentration_Z}
    Let Assumptions 1-5 hold. With probability at least $1-o(n^{-1})$: 
    \begin{equation}\label{eq:cross_ZtZ}
        |Z_j^\top Z_l - \E(Z_j^\top Z_l)| \leq C^*\sqrt{\frac{n{h}_j{h}_l \log n}{N}} \quad \text{ for all } j,l \in [p] \text{ with } j \neq l
    \end{equation}
    \begin{equation}\label{eq:norm_Z}
        |Z_j^\top Z_j - \E(Z_j^\top Z_j)| \leq C^*\sqrt{\frac{n{h}_j^2\log n}{N}} + \frac{C^*}{N}\sqrt{\frac{n{h}_j \log n}{N}} \text{ for all } j \in [p]
    \end{equation}
    where $\forall j\in [p], h_j = \sum_{k=1}^{K}A_{kj}$.
\end{lemma}
\begin{proof}
The proof is a re-adaptation of Lemma C.4 in \cite{tran2023sparse} for any word $j$ .

Similar to the analysis of \cite{ke2017new}, we rewrite each row $X_{i\cdot}$ as a sum over $N$ word entries $T_{im} \in \R^p$, where $T_{im}$ denotes the $m^{th}$ word in document $i$, encoded as a one-hot vector:
\begin{equation}\label{eq:def_t_im}
T_{im}(j) = \begin{cases}
    1 \qquad \text{ if the } m^{th} \text{ word in document } i \text{ is word } j\\
    0\qquad \text{otherwise},
\end{cases}
\end{equation} 
where the notation $a(j)$ denotes the $j^{th}$ entry of the vector $a$.
Under this formalism, we rewrite each row of $Z$ as:
$$ Z_{i\cdot} = \frac{1}{N} \sum_{m=1}^N (T_{im} - \E[T_{im}]) \in \R^p.$$
In the previous expression, under the pLSI model, the $\{T_{im}\}_{m=1}^N$ are i.i.d. samples from a multinomial distribution with parameter $M_{i\cdot}$.

 We can also express each entry $Z_{ij}$ as:
    
    \begin{equation}\label{eq:small_z1}
  \qquad  Z_{ij} = \frac{1}{N}\sum_{m=1}^{N} (T_{im}(j) - \E[T_{im}(j)])  
    \end{equation}
    
    Denote $S_{im}(j) := T_{im}(j) - \mathbb{E}[T_{im}(j)]$. 
    
    Fix $j, l \in [p]$.  The $\left\{ S^{(j)}_{im} \right\}_{\substack{i=1,\ldots, n \\ m=1,\ldots, N}}$ 
are all independent of one another (for all $i$ and $m$) and $T_{im}(j) \sim \text{Bernoulli}(M_{ij})$. 
By \eqref{eq:small_z1}, we note that 
    \begin{align*}
    Z_j^\top  Z_l &= \sum_{i=1}^n Z_{ij}Z_{il} =  \frac{1}{N^2}\sum_{i=1}^n  \sum_{m=1}^{N}\sum_{s=1}^{N} S_{im}(j) S_{is}(l) \\ 
    &= \frac{1}{N^2}\sum_{i=1}^n \sum_{m=1}^{N} S_{im}(j)S_{im}(l) + \frac{1}{N^2} \sum_{i=1}^n \sum_{\substack{1 \leq m, s \leq N \\ m\neq s}}S_{im}(j)S_{is}(l)\\ 
    &= \frac{n}{N}V_1 + \frac{N-1}{N}V_2 
    \end{align*}
    where we define 
    \begin{equation}\label{eq:v1}
        V_1 := \frac{1}{nN}\sum_{i=1}^n \sum_{m=1}^{N} S_{im}(j)S_{im}(l)
    \end{equation}
        \begin{equation}\label{eq:v2}
        V_2 := \frac{1}{N(N-1)}\sum_{i=1}^n \sum_{\substack{1 \leq m, s \leq N \\ m\neq s}}S_{im}(j)S_{is}(l) 
            \end{equation}
    We note that the random variable $V_2$ is centered ($\E(V_2) = 0$), and we need an upper bound with high probability on $|V_1 -\E(V_1)|$ and  $|V_2|$. We deal with each of these variables separately.

    \paragraph{Upper bound on $V_2.$} We remind the reader that we have fixed $j, l \in [p]$. Define $\mathcal{S}_{N}$ as the set of permutations on $\{1, \dots, N\}$ and $N' := \lfloor N / 2\rfloor$. Also define 
    $$W_i(S_{i1} ,\dots, S_{iN}) := \frac{1}{N'} \sum_{m=1}^{N'} S_{i, 2m-1}(j) S_{i, 2m}(l) $$
    Then by symmetry (note that the inner sum over $m,s$ in the definition of $V_2$ has $N(N-1)$ summands),
    $$V_2 = \frac{\sum_{i=1}^n\sum_{\pi \in \mathcal{S}_{N}}W_i(S_{i,\pi(1)}, \dots, S_{i,\pi(N)})}{N!} $$
    Define, for a given $\pi \in S_{N}$,  
    $$Q_\pi := \sum_{i=1}^nN' W_i(S_{\pi(1)},\dots, S_{\pi(N)}) $$
    so that $N' V_2  = \frac{1}{N!}\sum_{\pi\in S_{N}}Q_\pi$.
    For arbitrary $t, s > 0$, by Markov's inequality and the convexity of the exponential function, 
    $$\P(N'V_2 \geq t) \leq e^{-st} \E(e^{sN'V_2}) \leq e^{-st}\frac{\sum_{\pi\in S_{N}}\E(e^{sQ_\pi})}{N!} $$
    Also, define $Q = Q_\pi$ for $\pi$ the identity permutation. Observe that 
    $$Q = \sum_{i=1}^n \sum_{m=1}^{N'}Q_{im} \quad \text{where } Q_{im} = S_{i,2m-1}(j) S_{i,2m}(l)$$ so $Q$ is a (double) summation of mutually independent variables. We have $|Q_{im}| \leq 1$, $\E(Q_{im}) = 0$ and $\E(Q^2_{im}) \leq {h}_j{h}_l$ where $\forall j\in [p], h_j = \sum_{k=1}^{K}A_{kj}$. The rest of the proof for $V_2$ is similar to the standard proof for the usual Bernstein's inequality. 

   Denote $G(x) = \frac{e^x - 1-  x}{x^2}$, $G(x)$ is increasing as a function of $x$. Hence, 
    \begin{align*}
    \E(e^{sQ_{im}}) &= \E\left(1+ sQ_{im} + \frac{s^2Q_{im}^2}{2} +\dots \right)  \\ 
    &= \E[1 + s^2Q_{im}^2G(sQ_{im})]  \qquad \text{since } \E[Q_{im}] = 0\\
    &\leq \E[1 + s^2Q_{im}^2G(s)]  \\ 
    &\leq 1 + s^2{h}_j{h}_lG(s) \leq e^{s^2{h}_j{h}_lG(s)}
    \end{align*}
    Hence, 
    $$e^{-st}\E(e^{sQ}) = \exp(-st + N'n{h}_j{h}_l s^2G(s)) $$
    Since this bound is applicable to all $Q_\pi$ and not just the identity permutation, we have 
    $$\P(N'V_2 \geq t) \leq \exp(-st + N'n {h}_j {h}_ls^2G(s)) = \exp\left(-st + N'n{h}_j{h}_l(e^s - 1 - s)\right) $$
    Now we choose $s = \log \left(1 + \frac{t}{N'n{h}_j{h}_l}\right) > 0$. Then 
    \begin{align*}
        \P(N'V_2 \geq t) &\leq \exp\left[-t\log \left(1 + \frac{t}{N'n{h}_j{h}_l}\right) + N' n{h}_j{h}_l\left(\frac{t}{N'n{h}_j{h}_l} - \log\left(1 + \frac{t}{N'n{h}_j{h}_l}\right)\right)\right] \\  
        &= \exp\left[-N'n{h}_j{h}_lH\left(\frac{t}{N'n{h}_j{h}_l}\right)\right]
    \end{align*}
    where we define the function $H(x)= (1+x)\log (1+x) - x$. Note that we have the inequality
    $$H(x) \geq \frac{3x^2}{6+2x} $$
    for all $x > 0$. Hence, 
    $$\P\left(N'V_2 \geq t\right) \leq \exp\left(-\frac{t^2/2}{N'n{h}_j{h}_l + t/3}\right) $$
    or by rescaling, 
    \begin{equation}\label{eq:Bern}
    \P(N'V_2 \geq N'nt) \leq \exp\left(-\frac{N'nt^2/2}{{h}_j{h}_l +t/3}\right)  
    \end{equation}
    We can choose $t^2 = \frac{C^*{h}_j{h}_l}{N'n}\log n$ and note that ${h}_j{h}_l \geq c_{\min}^2\frac{\log n}{nN^{'}}$ by Assumption 5. Hence, with probability $1-o(n^{-1})$,
    $$|V_2| \leq C^*\sqrt{\frac{n{h}_j{h}_l\log n}{N}} $$
    By a simple union bound, we note that:
    \begin{equation}\label{eq:union_bound1}
    \begin{split}
     \P\left[ \exists (j,l) : \quad |V^{(j,l)}_2| \geq C^*\sqrt{\frac{n{h}_j{h}_l\log n}{N}}\right] &\leq \sum_{j,l}   \P\left[ \quad |V^{(j,l)}_2| \geq C^*\sqrt{\frac{n{h}_j{h}_l\log n}{N}}\right]\\ 
     &\leq p^2 e^{-C^* \log(n)} = e^{2\log(p)-C^* \log(n)}\\ 
      &\leq e^{-C \log(n)} \\ 
         \end{split}
    \end{equation}
where the last line follows by Assumption 5 (which implies that $p$ is small), noting that for some large enough constant $\tilde{C}<C^*$ such that $n^{\tilde{C}} \geq p^2$, 
$2\log(p) - C^*\log(n) \leq \tilde{C} \log(n) -C^*\log(n) = -(C^*-
\tilde{C})\log(n)$ .
Thus, for $C^*$ large enough, for any $j,l$, with probability $1-e^{-C \log(n)} = 1-o(n^{-1})$:
    $$|V_2| \leq C^*\sqrt{\frac{n{h}_j{h}_l\log n}{N}} $$

    \paragraph{Upper bound on $V_1.$} As for $V_1$, we can just apply the usual Bernstein's inequality. We remind the reader that $M_{ij} = \E[T_{im}(j)]$; we further note that $M_{ij} \leq {h}_j$. Since $S_{im}(j) = T_{im}(j) -M_{ij}$,
    \begin{equation}\label{eq:XX}
    S_{im}(j) S_{im}(l) = T_{im}(j) T_{im}(l) -M_{ij} T_{im}(l)- M_{il}T_{im}(j) +M_{ij}M_{il}  
    \end{equation}
    \begin{description}
        \item[ Case 1: If $j \neq l$:] then $T_{im}(j)T_{im}(l) = 0$ and so 
    \begin{align*}
        \text{Var}[S_{im}(j) S_{im}(l)] &= \text{Var}\left[M_{ij
        }T_{im}(l) + M_{il}T_{im}(j)\right] \\
        &\leq \E[M_{ij}T_{im}(l) + M_{il}T_{im}(j)]^2 \\ 
        &=M_{ij}^2 M_{il} + M_{il}^2M_{ij} =M_{ij}M_{il}(M_{ij}+M_{il}) \\ 
        &\leq M_{ij}M_{il} \leq {h}_j {h}_l
    \end{align*}
    since $M_{ij} + M_{il} \leq 1$. Hence, by Bernstein's inequality,
    $$\P\left(|V_1 -\E(V_1)| \geq t\right) \leq 2\exp\left(-\frac{nNt^2/2}{{h}_j{h}_l + t/3}\right)$$
    which is similar to \eqref{eq:Bern}, so picking $t^2 =  C^*\frac{h_j h_l\log(n)}{nN}$, we obtain with probability $1-o(n^{-1})$ that 
    $$\frac{n}{N}|V_1 - \E(V_1)| \leq \frac{C^*}{N}\sqrt{\frac{n{h}_j{h}_l\log n}{N}} \leq C^*\sqrt{\frac{n{h}_j{h}_l\log n}{N}}$$
    and \eqref{eq:cross_ZtZ} is proven. 

   \item[ Case 2: If $j=l$] then since $T_{im}^2(j) = T_{im}(j)$, \eqref{eq:XX} leads to
    \begin{equation}\label{eq:s2}
        S^2_{im}(j) = T_{im}(j)(1-2M_{ij}) +M_{ij}^2
    \end{equation} 
    and since $|1-2M_{ij}| \leq 1$ and $\text{Var}(T_{im}(j)) =M_{ij}(1-M_{ij})$,
    $$\text{Var}[S^2_{im}(j)] \leq M_{ij} \leq {h}_j $$
    and so we obtain \eqref{eq:norm_Z} since with probability $1 - o(n^{-1})$
    $$\frac{n}{N}|V_1-\E(V_1)| \leq \frac{C^*}{N}\sqrt{\frac{n{h}_j \log (n)}{N}} $$
        \end{description}
\end{proof}

\begin{lemma}[Concentration of the covariance matrix $X^\top X$]\label{lem:concentration_xtx}

 Let Assumptions 1-5 hold. With probability $1-o(n^{-1})$, the following statements hold true:

\begin{align*}
\|Z^\top Z - \E[Z^\top Z]\|_F &\leq C^* K\sqrt{\frac{n\log n}{N}}\\
    \|MZ^T \|_F & \leq C^*K\sqrt{\frac{n\log n}{N}}\\
    \|\wh D_0 - D_0 \|_F & \leq C^*\sqrt{\frac{K\log n}{nN}} \label{eq:D0}\\
\end{align*}

\end{lemma}

\begin{proof}

Let  $\forall j\in [p], h_j = \sum_{k=1}^{K}A_{kj}$.

 \paragraph{Concentration of $\|Z^\top Z - \E[Z^\top Z]\|_F.$}       We have:
 \begin{equation*}
     \begin{split}
         \|Z^\top Z - \E[Z^\top Z]\|_F^2 &= \sum_{j, j'=1}^p((Z^\top Z)_{jj'} - \E[(Z^\top Z)_{jj'}])^2\\
         &= \sum_{j}^p((Z^\top Z)_{jj} - \E[(Z^\top Z)_{jj}])^2 + \sum_{j\neq j'}^p((Z^\top Z)_{jj'} - \E[(Z^\top Z)_{jj'}])^2\\
          &= \sum_{j}^p\left( 2(C^*)^2 \frac{n h_j^2 \log(n)}{N} + 2\frac{(C^*)^2}{N^2} \frac{n h_j \log(n)}{N} \right) + \sum_{j\neq j'}^p (C^*)^2 \frac{n h_j h_{j'} \log(n)}{N}\\
          &\leq  C^* \sum_{j, j'}^p  \frac{n h_jh_{j'} \log(n)}{N}  \text{ since by Assumption 5, } \min_j h_j \geq c_{\min} \frac{\log(n)}{N}\\
          &\leq C^* K^2 \frac{n  \log(n)}{N}  \text{ since } \sum_j h_j = K
     \end{split}
 \end{equation*}
where the third line follows by Lemma~\ref{lem:concentration_Z}.
 
 \paragraph{Concentration of $\| \wh D_0 - D_0\|_F$.} For a fixed $j \in [p]$ we have
    $$(\wh D_0)_{j,j} - (D_0)_{j,j} = \frac{1}{n}\sum_{i=1}^n Z_{ij} = \frac{1}{nN}\sum_{i=1}^n \sum_{m=1}^{N}(T_{im}(j)-\E[T_{im}(j)])$$
Note that since $T_{im}(j) \sim \text{Bernoulli}(M_{ij})$, $|T_{im}(j)- \E[T_{im}(j)]| \leq 1$ and 
\begin{equation}\label{eq:h_j}
    \text{Var}(T_{im}(j)) =  M_{ij}(1-M_{ij})\leq M_{ij} = \sum_{k=1}^K A_{jk}W_{ki} \leq \sum_{k=1}^KA_{jk} = h_j
\end{equation}
(and also $\text{Var}(T_{im}(j)) \leq 1$). We apply Bernstein's inequality to conclude for any $t > 0$:
$$\P\left(| (\wh D_0)_{j,j} - (D_0)_{j,j}| \geq t\right) \leq 2\exp\left(-\frac{nNt^2/2}{{h}_j + t/3}\right) $$
Choosing $t = C^*\sqrt{\frac{{h}_j \log n}{nN}}$. Since $h_j \geq c_{\min}\frac{\log(n)}{N}$ (Assumption 5), we obtain that with probability at least $1- o(n^{-1})$, 
\begin{align*}
    |(\wh D_0)_{j,j} - (D_0)_{j,j}| &\leq C^*\sqrt{\frac{{h}_j\log n}{nN}} \\ 
    &\leq C^*\sqrt{\frac{{h}_j \log n }{nN}}
\end{align*}
Taking a union bound over $j \in [p]$, we obtain that:
$$ \P[\exists j \in [p]:  |(\wh D_0)_{j,j} - (D_0)_{j,j}| > C^*\sqrt{\frac{{h}_j\log n}{nN}}   ]\leq pe^{-C^* \log(n)} = e^{\log(p) - C^* \log(n)} \leq e^{-(C^*-1)\log(n)} = o(\frac{1}{n})$$
since we assume that $p \ll n$. 
Therefore, with probability at least $1-o(n^{-1})$:

$$\| (\wh D_0)_{j,j} - (D_0)\|_F^2 \leq \sum_{j=1}^p (C^*)^2\frac{{h}_j\log n}{nN} $$
and since $\sum_j h_j =K$:
$$ \| (\wh D_0)_{j,j} - (D_0)\|_F \leq  C^*\sqrt{\frac{K\log n}{nN}} $$

 \paragraph{Concentration of $\| M^\top Z\|_F$.}
We have:
\begin{equation*}
    \begin{split}
      \| M^\top Z\|_F&=      \| V \Lambda U^\top Z\|_F\\
      &\leq \lambda_1(M)       \| U^\top Z\|_F\\
    \end{split}
\end{equation*}
Noting that $\lambda_1(M) \leq \sqrt{n} $ (Lemma~\ref{lemma:bounds_sing_values_M}), and by Lemma~\ref{lemma:concentration_frob_norm_Utz}, 
with probability at least $1-o(n^{-1})$:
\begin{equation*}
    \begin{split}
      \| M^\top Z\|_F &\leq C^* K\sqrt{\frac{n\log(n)}{N}}\\
    \end{split}
\end{equation*}
\end{proof}

\begin{lemma}[Concentration of $\|(\Gamma^{\dagger} Z)_{e\cdot}\|_2, e \in \mathcal{E}$]\label{lem:concentration_Gamma_dagger_Z} Let Assumptions 1-5 hold. With probability at least $1-o(n^{-1})$, for all edges $e\in \mathcal{E}$: 
    \begin{equation}\label{eq:coef}
        |(\Gamma^{\dagger})^{\top}_{e\cdot} (Z_{\cdot j} - \E[Z_{\cdot j}])| \leq C^*\rho(\Gamma)\sqrt{\frac{{h}_j\log (n)}{N}},
    \end{equation}
        \begin{equation}\label{eq:coef2}
    \begin{split}
 \| ((\Gamma^{\dagger})^{\top}Z)_{e\cdot}  \|_{2} &\leq C^* \rho(\Gamma)\sqrt{\frac{K\log(n) }{N}}.
    \end{split}
\end{equation}
where $\forall j\in [p], h_j = \sum_{k=1}^{K}A_{kj}$.
\end{lemma}

\begin{proof}
    
    Fix $e \in \mathcal{E}$ and define $T_{im}(j)$ as in \eqref{eq:def_t_im}. Decomposing each $Z_{ij}- \E[Z_{ij}]$ as $Z_{ij} -\E[Z_{ij}] = \frac{1}{N} \sum_{m=1}^N (T_{im}(j) - \E[T_{im}(j)])$, we note that the product $((\Gamma^{\dagger})^{\top}(Z-\E[Z]))_{e j}$ can be written as a sum of $nN $ independent terms:
    $$    (\Gamma^{\dagger})^{\top}_{e\cdot} (Z_{\cdot j} - \E[Z_{\cdot j}])= \frac{1}{N} \sum_{m=1}^N \left (\sum_{i=1}^n  \Gamma^{\dagger}_{ie} \left(T_{im}(j) - \E[T_{im}(j)]\right)\right ) =\frac{1}{N} \sum_{m=1}^N \sum_{i=1}^n  \eta_{im}, $$
    with $\eta_{im} =\Gamma^{\dagger}_{ie} \left(T_{im}(j) - \E[T_{im}(j)]\right)$.
        \begin{enumerate}
        \item{\it Each $\eta_{im}$ verifies Bernstein's condition \eqref{bernstein_condition}}: We have:
    \begin{equation*}
        \begin{split}
          \sum_{i=1}^n  \sum_{m=1}^N \E[(\eta_{im})^q_+]&=  \sum_{i=1}^n  \sum_{m=1}^N \E[\left( \Gamma_{ie}^{\dagger} ( T_{im}(j) - \E[ T_{im}(j))\right)_+^q]\\
        \end{split}
    \end{equation*}

    We note that: $ \forall q
    \geq 3, \quad \E[\left( T_{im}(j) - \E[T_{im}(j)] \right)^q] = (1-M_{ij})(-M_{ij})^q + M_{ij}(1-M_{ij})^q $.
    
Therefore, if $q=2k$ for $k\geq 1$, $\E[\left( T_{im}(j) - \E[T_{im}(j)] \right)^q] \leq M_{ij} = \sum_{k} W_{ik}A_{kj}\leq \sum_{k} A_{kj}= h_j$ and:
        \begin{equation*}
        \begin{split}
          \sum_{i=1}^n  \sum_{m=1}^N \E[(\eta_{im})^{2k}_+]&\leq    \sum_{m=1}^N\sum_{i=1}^n  |\Gamma^{\dagger}_{ie}|^{2k} h_j \\
          &\leq   Nh_j\sum_{i=1}^n  (|\Gamma^{\dagger}_{ie}|^{2} )^{k-1}|\Gamma^{\dagger}_{ie}|^{2}\\
           &\leq  N h_j\rho^2(\Gamma) \rho^{2(k-1)}(\Gamma),\\
        \end{split}
    \end{equation*}
where the last line follows by noting that $|\Gamma^{\dagger}_{ie}|^{2} \leq \sum_{i=1}^n |\Gamma^{\dagger}_{ie}|^{2} \leq \rho^2(\Gamma),$ so  $|\Gamma^{\dagger}_{ie}|^{2(k-1)}\leq \rho^{2(k-1)}(\Gamma)$.

For $q=2k+1$ odd ($k\geq 1)$, we note that:
        \begin{equation*}
        \begin{split}
          \sum_{i=1}^n  \sum_{m=1}^N \E[(\eta_{im})^{2k+1}_+]&\leq         \sum_{i=1}^n  \sum_{m=1}^N \E[|\eta_{im}|^{2k+1}]\\
          &\leq (\sum_{m=1}^N \sum_{i=1}^n  \E[|\eta_{im}|^{2k}])^{\frac{1}{2}}(\sum_{m=1}^N\sum_{i=1}^n  |\eta_{im}|^{2k+2})^{\frac{1}{2}}  \qquad  \text{(Cauchy Schwartz along $i,m$)}\\
           &\leq  N h_j\rho^{2k+1}(\Gamma)\\
           &\leq N h_j\rho^2(\Gamma)\rho^{2k-1}(\Gamma)
        \end{split}
    \end{equation*}

        \item {\it  Each of the variance $\mathrm{Var}(S_m) = \sum_{i=1}^n \mathrm{Var}(\eta_{im})$ is also bounded}:
        $$ \text{Var}(\eta_{im})  =(\Gamma^{\dagger})_{ie}^2 \text{Var}(T_{im}(j))\leq(\Gamma^{\dagger})_{ie}^2  h_j.$$
        Thus:
        $$ \sum_{m=1}^N \sum_{i=1}^n \text{Var}(\eta_{im}) \leq N\rho^2(\Gamma) h_j.$$
    \end{enumerate}

    Therefore, by Bernstein's inequality (Lemma~\ref{lemma: bernstein ineq0}), plugging in  $v = N\rho^2(\Gamma)h_j$ and $c =\frac{\rho(\Gamma)}{3}$:
    $$ \P[\frac{1}{N} | \sum_{i=1}^n \sum_{m=1}^N \eta_{im} | > t] \leq 2e^{-\frac{N^2t^2/2}{N\rho(\Gamma)^2 h_j +  \frac{\rho(\Gamma)}{3}\times Nt}}.$$
    Choosing $t = C^* \rho(\Gamma)\sqrt{\frac{h_j\log(n) }{ N}},$ with $C^*>1$, we have:
   \begin{equation*}
       \begin{split}
           \P[\frac{1}{N} | \sum_{i=1}^n \sum_{m=1}^N \eta_{im} |  > C^* \rho(\Gamma)\sqrt{ \frac{h_j\log(n) }{N}}] & \leq 2e^{-\frac{(C^*)^2 \log(n)/2}{1 + \frac{C^*}{3} \sqrt{\frac{\log(n) }{h_j N}} }}.
       \end{split}
   \end{equation*} 

   Therefore, by Assumption 5, $h_j > c_{\min}\frac{\log(n)}{N},$ then, with probability at least $1-o(n^{-1})$, $|((\Gamma^{\dagger})^{\top}Z)_{ej} | \leq C^* \rho(\Gamma)\sqrt{ \frac{h_j\log(n) }{N}} $.

   Therefore, by a simple union bound and following the argument in \eqref{eq:union_bound1}:
   $$ \P[\exists j: |((\Gamma^{\dagger})^{\top}Z)_{ej} | \geq C^* \rho(\Gamma)\sqrt{ \frac{h_j\log(n) }{N}}]\leq  p e^{-C^* \log(n)} =e^{\log(p)-C^*\log(n)}\leq e^{-(C^*-1)\log(n)}. $$

since we assume that $p \ll n$. Writing $\| ((\Gamma^{\dagger})^{\top}Z)_{e\cdot}  \|_{2}^2 = \sum_{j=1}^p |((\Gamma^{\dagger})^{\top}Z)_{ej} |^2$, we thus have:
\begin{equation}
    \begin{split}
      \P[ \| ((\Gamma^{\dagger})^{\top}Z)_{e\cdot}  \|_{2}^2 \geq \sum_{j=1}^p (C^*)^2 \rho^2(\Gamma){\frac{h_j\log(n) }{N}} ] &\leq \P[\exists j: |((\Gamma^{\dagger})^{\top}Z)_{ej} | \geq C^* \rho(\Gamma)\sqrt{ \frac{h_j\log(n) }{N}}] \\
      \implies  \P[ \| ((\Gamma^{\dagger})^{\top}Z)_{e\cdot}  \|_{2}^2 \leq (C^*)^2 \rho^2(\Gamma){\frac{K\log(n) }{N}} ] &\geq 1-o(n^{-1}).\\
    \end{split}
\end{equation}
where the last line follows by noting that $\sum_{j=1}^p h_j = K.$

Finally, to show that this holds for any $e\in \mathcal{E}$, it suffices to apply a simple union bound:
\begin{equation}
    \begin{split}
        \P[ \exists e \in \mathcal{E}: \quad \|((\Gamma^{\dagger})^{\top}Z)_{e\cdot}\|^2 \geq C^*\rho^2(\Gamma){\frac{K\log(n) }{N}} ]&\leq \sum_{e \in \mathcal{E}} \P[\|((\Gamma^{\dagger})^{\top}Z)_{e\cdot}  \|_{2}^2 \geq C^* \rho^2(\Gamma)\frac{K\log(n) }{N} ]\\
        &\leq |\mathcal{E}|e^{-C\log(n)}\\
        &\leq e^{c_0 \log(n) - C^*\log(n)}
    \end{split}
\end{equation}
with $c_0 < 2$.
Therefore, $\P[ \exists e \in \mathcal{E}: \quad  \|(\Gamma^{\dagger})^{\top}Z)_{e\cdot}\|^2\geq C^*\rho^2(\Gamma){\frac{K\log(n) }{N}} ] = o(\frac{1}{n})$ for a choice of $C^*$ sufficiently large.
\end{proof}

\begin{lemma}[Concentration of $\|\Pi Z\|_F$]
\label{lemma:concentration_pi_z}

Let Assumptions 1-5 hold. With probability at least $1-o(n^{-1})$:

\begin{equation}
    \begin{split}
        \|  \Pi {Z}\|_F^2 &\leq C^* n_C K  \frac{\log(n)  }{ N}\
    \end{split}
 \end{equation}
\end{lemma}

\begin{proof}
    We remind the reader that letting $\C_j$ denote the $j^{th}$ connected component of the graph $\mathcal{G}$ and $n_{\C_l} = | \C_l|$ its cardinality, $\Pi$ can be arranged in a block diagonal form where each block represents a connected component, $\Pi_{[\C_l]} = \frac{1}{n_{\C_l}} \mathbf{1}_{\mathcal{C}_l} \mathbf{1}_{\mathcal{C}_l}^T$. 
Since the components $\mathcal{C}_l$ are all disjoint, $\|\Pi Z \|_F$ can be further decomposed as:
\begin{equation*}
    \begin{split}
         \|\Pi Z \|_F^2 &= \sum_{l=1}^{n_C} \| \frac{1}{\nj} \mathbf{1}_{\mathcal{C}_l}\mathbf{1}_{\mathcal{C}_l}^T Z_{[\C_l]}\|_F^2\\
         &= \sum_{l=1}^{n_C} \nj \| \frac{1}{\nj} \mathbf{1}_{\mathcal{C}_l}^T Z_{[\C_l]}\|_2^2
    \end{split}
\end{equation*}
By Assumption 3, $\forall i, N_i=N$. Following Equation \eqref{eq:def_t_im}, we rewrite each row of $Z$ as:
$$ Z_{i\cdot} = \frac{1}{N} \sum_{m=1}^N (T_{im} - \E[T_{im}]) \in \R^p.$$
In the previous expression, under the pLSI model, the $\{T_{im}\}_{m=1}^N$ are i.i.d. samples from a multinomial distribution with parameter $M_{i\cdot}$.
Thus, for each word $j$ and each connected component $\C_l$:
\begin{equation*}
    \begin{split}
        \frac{1}{n_{\C_l}} \mathbf{1}_{\mathcal{C}_l}^T Z_{[\C_l]j} &= \frac{1}{n_{\C_l} N} \sum_{i \in \mathcal{C}_l} \sum_{m=1}^N (T_{im}(j) - \E[T_{im}(j)]).
    \end{split}
\end{equation*}

Fixing $j$ and denoting $S^{(j)}_{im} =T_{im}(j) - \E[T_{im}(j)]$, we note that the $\left\{ S^{(j)}_{im} \right\}_{\substack{i=1,\ldots, n \\ m=1,\ldots, N}}$ 
are independent of one another (for all $i$ and $m$), and since $T_{im}(j) \sim \text{Bernouilli}(M_{ij})$, $|S^{(j)}_{im} | \leq 2.$  Define $h_j := \sum_{k=1}^K A_{kj}$. Then,
$$\text{Var}(S^{(j)}_{im}) = \E[(T^{(j)}_{im})^2] - M_{ij}^2 =\E[T^{(j)}_{im}] - M_{ij}^2 \leq M_{ij} = \sum_{k=1}^K W_{ik} A_{kj} \leq \sum_{k=1}^K A_{kj} = h_j.$$

Therefore, by the Bernstein inequality (Lemma~\ref{lemma: bernstein ineq}), for the $l^{th}$ connected component $\C_l$ of the graph $\mathcal{G}$ and for any word $j \in [p]$:
\begin{equation*}
    \begin{split}
        \forall t>0, \quad \P[|\frac{1}{n_{\C_l}} \mathbf{1}_{\mathcal{C}_l}^T Z_{[\C_l],j}|>t]  = \P[\frac{1}{n_{\C_l} N}|\sum_{i \in \mathcal{C}_l} \sum_{m=1}^N S^{(j)}_{im}|>t] \leq 2\exp\{-\frac{n_{\C_l} Nt^2/2}{h_j + \frac{2}{3}t}\}. 
    \end{split}
\end{equation*}

Choosing $t^2 = C^*\frac{h_j}{n_{\C_l}N} \log(n)$, the previous inequality becomes:
\begin{equation*}
    \begin{split}
    \P[|\frac{1}{n_{\C_l}} \mathbf{1}_{\mathcal{C}_l}^T Z_{[\C_l]j}|>\sqrt{C^*\frac{h_j}{n_{\C_l}N} \log(n)}]  & \leq 2\exp\{-\frac{C^* h_j\log(n)}{h_j + \frac{2}{3}\sqrt{C^*\frac{h_j \log(n)}{n_{\C_l}N}}}\}=2\exp\{-\frac{C^*\log(n)}{1+ \frac{2}{3}\sqrt{C^*\frac{\log(n)}{h_jn_{\C_l}N}}}\}\\
    &\leq 2\exp\{-{C^* \log(n)} \}.
    \end{split}
\end{equation*}

as long as $h_{j} \geq c_{\min} \frac{\log(n)}{n_{C_l}N}$ (which follows from Assumption 5 since $h_{j} \geq c_{\min} \frac{\log(n)}{N}$).
Therefore, by a simple union bound:
\begin{equation*}
    \begin{split}
& \P \left[ \exists j \in [p], \exists l \in [n_C]:\quad \frac{1}{n_{\C_l}N}|\sum_{i \in \mathcal{C}_l} \sum_{m=1}^N S^{(j)}_{im}|> \sqrt{C^*\frac{h_j}{n_{\C_l}N} \log(n)} \right] \\
 &\leq 2pn_C\exp\{-{C^* \log(n)} \}\\
 &=\exp\{\log(2) + \log(p) + \log(n_C)- C^*\log(n) \}\\
 &\leq \exp\{ - (C^*-3)\log(n) \},
    \end{split}
\end{equation*}
As a consequence of  Assumption 5, we know that $p \ll n$ (see Remark 2) and under the graph-aligned setting, $n_{\C} \ll n$.
Thus with probability $1-o(n^{-1})$, for all $j \in [p]$ and all $l \in [n_C]$:
\begin{equation*}
    \begin{split}
 \| \frac{1}{n_{\mathcal{C}_l}} \mathbf{1}_{\mathcal{C}_l}^T Z_{[n_C]}\|_2^2&\leq \sum_{j\in [p]}C^*\frac{h_j}{n_{\C_l}N} \log(n) =  C^*\frac{K}{n_{\C_l}N} \log(n).
      \end{split}
\end{equation*}  
 where the last equality follows from the fact that $\sum_{j=1}^p h_j = \sum_{j=1}^p \sum_{k=1}^K A_{kj} = K.$
Therefore:
\begin{equation*}
    \begin{split}
         \|\Pi Z \|_F^2  &= \sum_{l=1}^{n_C} n_{\C_l} \| \frac{1}{n_{\C_l}} \mathbf{1}_{\mathcal{C}_l}^T Z_{[n_{\mathcal{C}}]}\|_2^2\\
         &\leq  \sum_{l=1}^{n_C} C^* n_{\C_l} K  \frac{  \log(n)}{n_{\C_l}N} \\
      &\leq C^* n_C K  \frac{\log(n)  }{ N}\\
    \end{split}
\end{equation*}
\end{proof}

 \begin{lemma}[Concentration of $\| U^{\top} Z \|_{F}$]
\label{lemma:concentration_frob_norm_Utz}

Let Assumptions 1-5 hold. Let  $U\in \R^{n \times r}$ denote a projection matrix: $U^TU =I_r$, with $r$ a term that does not grow with $n$ or $p$ and $r \leq n$, and let $Z$ denote some centered multinomial noise as in $X = M+Z$. Then with probability at least $1-o(n^{-1})$:

\begin{equation}
    \begin{split}
 \|U^{\top} Z \|_F &\leq  C \sqrt{\frac{ K r\log(n)}{N}}  \\
    \end{split}
 \end{equation}
\end{lemma}

\begin{proof}
Let $\tilde{Z}=U^{\top} Z.$
    We have:
    \begin{equation}
        \begin{split}
            \| \tilde{Z} \|_F^2 &=  \sum_{j=1}^p\sum_{k=1}^r \tilde{Z}_{kj}^2 
        \end{split}
    \end{equation}

    We first note that 
    \begin{align}\label{eq:small_z}
  \forall k \in [r], \forall j \in [p], \qquad  {\tilde{Z}}_{kj} &= \frac{1}{N}\sum_{m=1}^{N} (U_{\cdot k}^{\top}T_{\cdot m}(j) - \E[U^{\top}_{\cdot k} T_{\cdot m}(j)])  \\
  &= \frac{1}{N}\sum_{m=1}^{N}  \sum_{i=1}^n( U_{ik}T_{i m}(j) - \E[U_{ik} T_{im}(j)])\\
  & = \frac{1}{N}\sum_{m=1}^{N}  \sum_{i=1}^n \eta_{im} \quad \text{ with }  \eta_{im}= U_{ik}T_{i m}(j) - \E[U_{ik} T_{im}(j)]
    \end{align}

Thus, ${\tilde{Z}}_{kj}$ is a sum of $N$ centered independent variables.

Fix $k \in [r], j \in [p]$. We have: $\text{Var}(\sum_{i=1}^n \eta_{im}) = \sum_{i=1}^n U_{ik}^2  M_{ij} (1-M_{ij}) \leq\sum_{i=1}^n U_{ik}^2  h_j$ where $h_j= \sum_{k=1}^K A_{kj}$, since $M_{ij}\leq h_j$. Therefore, as $\sum_{i=1}^n U_{ik}^2 = 1$:
$$\sum_{m=1}^N  \sum_{i=1}^n \text{Var}(\eta_{im}) =N h_j. $$
Moreover, for each $i,m$,  $|\eta_{im}| < | U_{ik}| \leq 1.$
Thus, by Bernstein's inequality (Lemma~\ref{lemma: bernstein ineq0}, with $v = N h_j$ and $c=1/3$):
\begin{equation*}
    \begin{split}
        \P[ |\frac{1}{N}\sum_{m=1}^{N}  \sum_{i=1}^n \eta_{im} | > t ]&\leq 2 e^{-\frac{Nt^2/2}{h_j + t/3}}
    \end{split}
\end{equation*}

Choosing $t =C^* \sqrt{\frac{h_j\log(n)}{N}}$:
\begin{equation*}
    \begin{split}
        \P[ |\frac{1}{N}\sum_{m=1}^{N}  \sum_{i=1}^n \eta_{im} | > t ]&\leq 2 e^{-\frac{(C^*)^2\log(n)/2}{1 + \frac{C^*}{3}\sqrt{\frac{\log(n)}{h_j N}}}}
    \end{split}
\end{equation*}

Therefore, by Assumption 5, $h_j > c_{\min}\frac{\log(n)}{N}$, then, with probability at least $1-o(n^{-1})$,  $|\tilde{Z}_{kj} |^2 \leq C^*\frac{h_j\log(n)}{N} .$
    
Therefore, by a simple union bound:
\begin{equation}
    \begin{split}
    &  \P[ \exists (j,k): |\tilde{Z}_{kj} |^2 > C^*\frac{h_j\log(n)}{N} ] < rp e^{-C^* \log(n)}  \\
      \implies&  \P[  \|\tilde{Z} \|_F^2 > C\frac{Kr\log(n)}{N} ] < rp e^{-C^* \log(n)} \quad \text{ since }\sum_{j=1}^p h_j = K.
    \end{split}
\end{equation}
Since we assume that $pr \ll n$, the result follows.
\end{proof}

\begin{lemma}[Concentration of $\| \Pi Z V\|_F$]
\label{lemma:concentration_frob_norm_pitildeZ}

Let Assumptions 1-5 hold. Let $V$ be a orthogonal matrix: $V\in \mathbb{R}^{p \times K}, V^\top V = I_K$. Let $\Pi$ denote the projection matrix unto $\text{Ker}(\Gamma^{\dagger} \Gamma)$, such that $I_n = \Pi \oplus^{\perp} \Gamma^{\dagger} \Gamma$.
With probability at least $1-o(n^{-1})$:

\begin{equation}
    \begin{split}
        \|  \Pi \tilde{Z}\|_F^2 &\leq C^* n_C K  \frac{\log(n)  }{ N}\
    \end{split}
 \end{equation}
 where $\tilde{Z}=ZV$.
\end{lemma}

\begin{proof}
We follow the same procedure as the proof of Lemma~\ref{lemma:concentration_pi_z}. Letting $\C_j$ denote the $j^{th}$ connected component of the graph $\mathcal{G}$ and $n_{\C_l} = | \C_l|$ its cardinality, $\|\Pi \tilde{Z} \|_F$ can be decomposed as:
\begin{equation*}
    \begin{split}
         \|\Pi \tilde{Z} \|_F^2 &\leq \sum_{l=1}^{n_C} \| \frac{1}{\nj} \mathbf{1}_{\mathcal{C}_l}\mathbf{1}_{\mathcal{C}_l}^T \tilde{Z}_{[\C_l]}\|_F^2\\
         &= \sum_{l=1}^{n_C} \nj \| \frac{1}{\nj} \mathbf{1}_{\mathcal{C}_l}^T \tilde{Z}_{[\C_l]}\|_2^2
    \end{split}
\end{equation*}
By Assumption 3, $\forall i, N_i=N$. Using the definition of $T_{im}$ provided in \eqref{eq:def_t_im}, for each $k \in [K]$, and each connected component $\C_l$:
\begin{equation*}
    \begin{split}
        \frac{1}{n_{\C_l}} \mathbf{1}_{\mathcal{C}_l}^T \tilde{Z}_{[\C_l]k} &= \frac{1}{n_{\C_l} N} \sum_{i \in \mathcal{C}_l} \sum_{m=1}^N \sum_{j=1}^p (T_{im}(j) - \E[T_{im}(j)])V_{jk}.
    \end{split}
\end{equation*}

Fix $j$ and denote $\eta_{jm} =(T_{im}(j) - \E[T_{im}(j)]) V_{jk}$. We have $|\sum_{j=1}^{p}\eta_{jm}|\leq 2$ and

$$\text{Var}(\sum_{j=1}^{p}\eta_{jm})=\sum_{j=1}^{p}M_{ij}(V_{jk})^2-(\sum_{j=1}^pM_{ij}V_{jk})^2\leq 1$$

Therefore, by Bernstein's inequality (Lemma~\ref{lemma: bernstein ineq}), for the $l^{th}$ connected component $\C_l$ of the graph $\mathcal{G}$ and for any $k \in [K]$:
\begin{equation*}
    \begin{split}
        \forall t>0, \quad \P[|\frac{1}{n_{\C_l}} \mathbf{1}_{\mathcal{C}_l}^T \tilde{Z}_{[\C_l]k}|>t]  = \P[\frac{1}{n_{\C_l} N}|\sum_{i \in \mathcal{C}_l} \sum_{m=1}^N \sum_{j=1}^p \eta_{jm}|>t] \leq 2\exp\{-\frac{n_{\C_l} Nt^2/2}{1 + \frac{2}{3}t}\}. 
    \end{split}
\end{equation*}

Choosing $t^2 = C^*\frac{\log(n)}{n_{\C_l}N}$, the previous inequality becomes:
\begin{equation*}
    \begin{split}
    \P[|\frac{1}{n_{\C_l}} \mathbf{1}_{\mathcal{C}_l}^T \tilde{Z}_{[\C_l]k}|>\sqrt{C^*\frac{\log(n)}{n_{\C_l}N}}]  & \leq 2\exp\{-\frac{C^* \log(n)}{1 + \frac{2}{3}\sqrt{C^*\frac{\log(n)}{n_{\C_l}N}}}\}\leq 2\exp\{-{C \log(n)} \}.
    \end{split}
\end{equation*}

as long as $n_{\C_l}N \gtrsim \log(n)$.
Therefore, by a simple union bound:
\begin{equation*}
    \begin{split}
& \P \left[ \exists k \in [K], \exists l \in [n_C]:\quad \frac{1}{n_{\C_l}N}|\sum_{i \in \mathcal{C}_l} \sum_{m=1}^N \sum_{j=1}^{p}\eta_{jm}|> \sqrt{C^*\frac{\log(n)}{n_{\C_l}N}} \right] \\
 &\leq 2Kn_C\exp\{-{C^* \log(n)} \}\\
 &=\exp\{\log(2) + \log(K) + \log(n_C)- C^*\log(n) \}\\
 &\leq \exp\{ - (C^*-3)\log(n) \},
    \end{split}
\end{equation*}
Thus with probability $1-o(n^{-1})$, for all $k \in [K]$ and all $l \in [n_C]$:
\begin{equation*}
    \begin{split}
 \| \frac{1}{n_{\C_l}} \mathbf{1}_{\mathcal{C}_l}^T \tilde{Z}_{[n_{\C_l}] \cdot}\|_2^2&\leq \sum_{k\in [K]}\frac{C^*\log(n)}{n_{\C_l}N}  =  C^*\frac{K}{n_{\C_l}N} \log(n).
      \end{split}
\end{equation*}  
 and
\begin{equation*}
    \begin{split}
         \|\Pi \tilde{Z} \|_F^2  &= \sum_{l=1}^{n_C} n_{\C_l} \| \frac{1}{n_{\C_l}} \mathbf{1}_{\mathcal{C}_l}^T \tilde{Z}_{[n_{\mathcal{C}}]}\|_2^2\\
         &\leq  \sum_{l=1}^{n_C} C^* n_{\C_l} K  \frac{  \log(n)}{n_{\C_l}N} \\
      &\leq C^* n_C K  \frac{\log(n)  }{ N}\\
    \end{split}
\end{equation*}
\end{proof}

\begin{lemma}[Concentration of $\| e_i^{\top}Z \|_2$ and $\| e_i^{\top}ZV  \|_2$]
\label{lemma:concentration_max_norm_zV}

Let Assumptions~1-5 hold. Let $\tilde{Z} = ZV$, with $V\in \R^{p \times r}$ a projection matrix: $V^TV =I_r$, with $r$ a term that does not grow with $n$ or $p$ and $r \leq p$. Then with probability at least $1-o(n^{-1})$:

\begin{equation}
    \begin{split}
        \max_{i \in [n]}\| e_i^{\top}Z \|_2&\leq C_1 \sqrt{\frac{ K\log(n)}{N}}\\
        \max_{i \in [n]}\|  e_i^{\top}ZV\|_2 &\leq  C_2 \sqrt{\frac{ r\log(n)}{N}}  \\
    \end{split}
 \end{equation}
\end{lemma}

\begin{proof}

We first note that 
    \begin{equation*}
        \begin{split}
            \| e_i^{\top}Z \|_2^2 &= \sum_{j=1}^p Z_{ij}^2 
        \end{split}
    \end{equation*}
\begin{align*}
  \forall j \in [p], \qquad  {Z}_{ij} &= \frac{1}{N}\sum_{m=1}^{N} (T_{im}(j) - \E[T_{i m}(j)])  \\
  & = \frac{1}{N}\sum_{m=1}^{N}  \eta_{im} \quad \text{ with }  \eta_{im}= T_{i m}(j) - \E[T_{im}(j)]
    \end{align*}

Thus ${Z}_{ij}$ is a sum of $N$ centered independent variables.

Fix $j \in [p]$. We have: $\text{Var}(\eta_{im}) = M_{ij}^2-M_{ij}\leq M_{ij}\leq h_j$ and 
$$\sum_{m=1}^N  \text{Var}(\eta_{im}) \leq Nh_j. $$
Moreover,for each $m$,  
$ |\eta_{im}| \leq 1.$ Thus, by Lemma \ref{lemma: bernstein ineq}:
\begin{equation*}
    \begin{split}
        \P[ |\frac{1}{N}\sum_{m=1}^{N}  \eta_{im} | > t ]&\leq 2 e^{-\frac{Nt^2/2}{1 + t/3}}
    \end{split}
\end{equation*}

Choosing $t =C^* \sqrt{\frac{h_j\log(n)}{N}}$:
\begin{equation*}
    \begin{split}
        \P[ |\frac{1}{N}\sum_{m=1}^{N}  \sum_{j=1}^p \eta_{jm} | > t ]&\leq 2 e^{-(C^*)^2\frac{\log(n)/2}{1 + \frac{C^*}{3}\sqrt{\frac{\log(n)}{h_jN}}}}
    \end{split}
\end{equation*}

Therefore, by Assumption 5, $N > c_{\min}\log(n)$, then, with probability at least $1-o(n^{-1})$,  $|Z_{ij} |^2 \leq C^*\frac{h_j\log(n)}{N} .$ By a simple union bound:
\begin{equation*}
    \begin{split}
    &  \P[ \exists j: |Z_{ij} |^2 > C\frac{h_j\log(n)}{N} ] \leq p e^{-C^* \log(n)}  \\
      \implies&  \P[  \max_{i\in[n]}\|e_i^{\top}Z \|_2^2 > C\frac{K\log(n)}{N} ] \leq npe^{-C^* \log(n)}\leq e^{-C^*\log(n) +\log(p) + \log(n)} \leq e^{-(C^*-2) \log(n)} 
    \end{split}
\end{equation*}
since $\sum_{j=1}^ph_j=K$. Also, since we assume that $\max(p, r) \ll n$, the result follows.
    
Similarly, denote $\tilde{Z}=ZV$,
    \begin{align*}
  \forall k \in [r], \qquad  {\tilde{Z}}_{ik} &= \frac{1}{N}\sum_{j=1}^p\sum_{m=1}^{N} (T_{im}(j)V_{jk} - \E[T_{i m}(j)V_{jk}])  \\
  & = \frac{1}{N}\sum_{m=1}^{N}  \sum_{j=1}^p \eta_{jm} \quad \text{ with }  \eta_{jm}= V_{jk}T_{i m}(j) - \E[V_{jk} T_{im}(j)]
    \end{align*}

Note that: $\sum_{j=1}^p(T_{im}(j)V_{jk} - \E[T_{i m}(j)V_{jk}]) = V_{j_0K} - \sum_{j=1}^pM_{ij}V_{jk} $ with probability $M_{ij_0}, j_0 \in [p].$

Thus ${\tilde{Z}}_{ik}$ is a sum of $N$ centered independent variables.

Fix $k \in [r]$. We have: $\text{Var}(\sum_{j=1}^p \eta_{jm}) = \sum_{j=1}^p V_{jk}^2   - ( \sum_{j=1}^p V_{jk}M_{ij})^2,$ and since $\sum_{j=1}^p V_{jk}^2 = 1$:
$$\sum_{m=1}^N  \text{Var}(\sum_{j=1}^p \eta_{jm}) \leq N. $$
Moreover,for each $m$,  
\begin{equation*}\label{eq:bound_eta}
    |\sum_{j=1}^p\eta_{jm}| \leq \max_{j} | V_{jk}| + \sum_{j=1}^p M_{ij} |V_{jk}|\leq  2  \max_{j} | V_{jk}| \leq 2.
\end{equation*}
Thus, by Lemma \ref{lemma: bernstein ineq}:
\begin{equation*}
    \begin{split}
        \P[ |\frac{1}{N}\sum_{m=1}^{N}  \sum_{j=1}^p \eta_{jm} | > t ]&\leq 2 e^{-\frac{Nt^2/2}{1 + 2t/3}}
    \end{split}
\end{equation*}

Choosing $t =C^* \sqrt{\frac{\log(n)}{N}}$:
\begin{equation*}
    \begin{split}
        \P[ |\frac{1}{N}\sum_{m=1}^{N}  \sum_{j=1}^p \eta_{jm} | > t ]&\leq 2 e^{-(C^*)^2\frac{\log(n)/2}{1 + \frac{2}{3}C^*\sqrt{\frac{\log(n)}{N}}}}
    \end{split}
\end{equation*}

Therefore, by Assumption 5, $N > c_{\min}\log(n)$, then, with probability at least $1-o(n^{-1})$,  $|\tilde{Z}_{kj} |^2 \leq C^*\frac{\log(n)}{N} .$ By a simple union bound:
\begin{equation*}
    \begin{split}
    &  \P[ \exists k: |\tilde{Z}_{ik} |^2 > C\frac{\log(n)}{N} ] \leq r e^{-C^* \log(n)}  \\
      \implies&  \P[  \max_{i\in[n]}\|e_i^{\top}\tilde{Z} \|_2^2 > C\frac{r\log(n)}{N} ] \leq rne^{-C^* \log(n)}\leq e^{-C^*\log(n) +\log(r) + \log(n)} \leq e^{-(C^*-2) \log(n)} 
    \end{split}
\end{equation*}
Since we assume that $\max(p, r) \ll n$, the result follows.
\end{proof}

%% file: appendixD_synthetic_experiments.tex
We propose the following procedure for generating synthetic datasets such that the topic mixture matrix $W$ is aligned with respect to a known graph.

\begin{itemize}
    \item [1.] \textbf{Generate spatially coherent documents} Generate $n$ points (documents) over a unit square $[0,1]^2$. Divide the unit square into $n_{grp}=30$ equally spaced grids and get the center for each grid. Apply k-means algorithm to the points with these as initial centers. This will divide the unit square into 30 different clusters. Next, randomly assign these clusters to $K$ different topics. In the end, within the same topic, we will observe some clusters of documents that are not necessarily next to each other (see Figure~\ref{fig:gt}). This is a more challenging setting where the algorithm has to leverage between the spatial information and document-word frequencies to estimate the topic mixture matrix. Based on the coordinates of documents, construct a spatial graph where for each document, edges are set for the $m=5$ closest documents and weights as the inverse of the euclidean distance between two documents. \\
    \item [2. ] \textbf{Generate matrices $W$ and $A$} For each cluster, we generate a topic mixture weight $\mathbf{\alpha} \sim \mathrm{Dirichlet}(\mathbf{u})$ where $u_k \sim \mathrm{Unif}(0.1, 0.5)$ ($k \in [K]$). We order $\mathbf{\alpha}$ so that the biggest element of $\mathbf{\alpha}$ is assigned to the cluster's dominant topic. We also add small Gaussian noise to $\alpha$ so that in the end, for each document in the cluster, $W_{i\cdot}=\mathbf{\alpha}+\mathbf{\epsilon}_{i}$, $\mathbf{\epsilon}_{ik} \sim N(0,0.03)$. We sample $K$ rows of $W$ as \textit{anchor documents} and set them as $\mathbf{e}_k$. The elements of $A$ are generated from $\mathrm{Unif}(0,1)$ and normalized so that each row of $A$ sums up to 1. Similarly to anchor documents, $K$ columns of $A$ are selected as \textit{anchor words} and set to $\mathbf{e}_k$. \\
    \item [3. ] \textbf{Generate frequency matrix $X$} We obtain the ground truth $M = WA$ and sample each row of $D$ from $\mathrm{Multinomial}(N, M_{i\cdot})$. Each row of $X$ is obtained by $X_{i\cdot} = D_{i\cdot}/N$. 
\end{itemize}

Figure~\ref{fig:gt} illustrates the ground truth mixture weights, $W_{\cdot k}$, for each topic generated with parameters $n=1000, N=30, p=30$ and $K=3$. Here, each dot in the unit square represents a document, with lighter colors indicating higher mixture weights. We observe patches of documents that share similar topic mixture weights.

\begin{figure}
    \centering
    \includegraphics[width=0.7\textwidth]{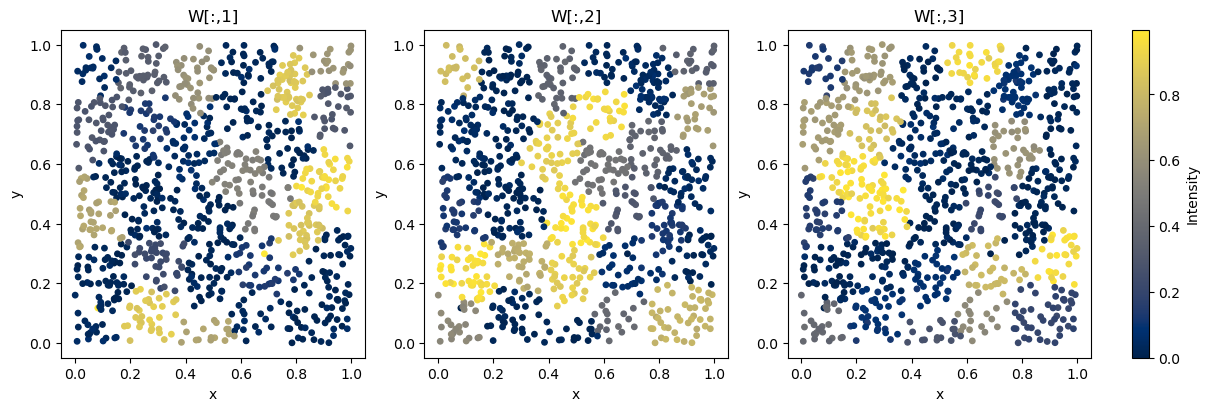}
    \caption{Heatmap of generated ground truth $W_{\cdot1}, W_{\cdot2}, W_{\cdot3}$, representing each topic mixture weight.}
    \label{fig:gt}
\end{figure}

\begin{figure}
    \centering
    \includegraphics[width=0.8\textwidth]{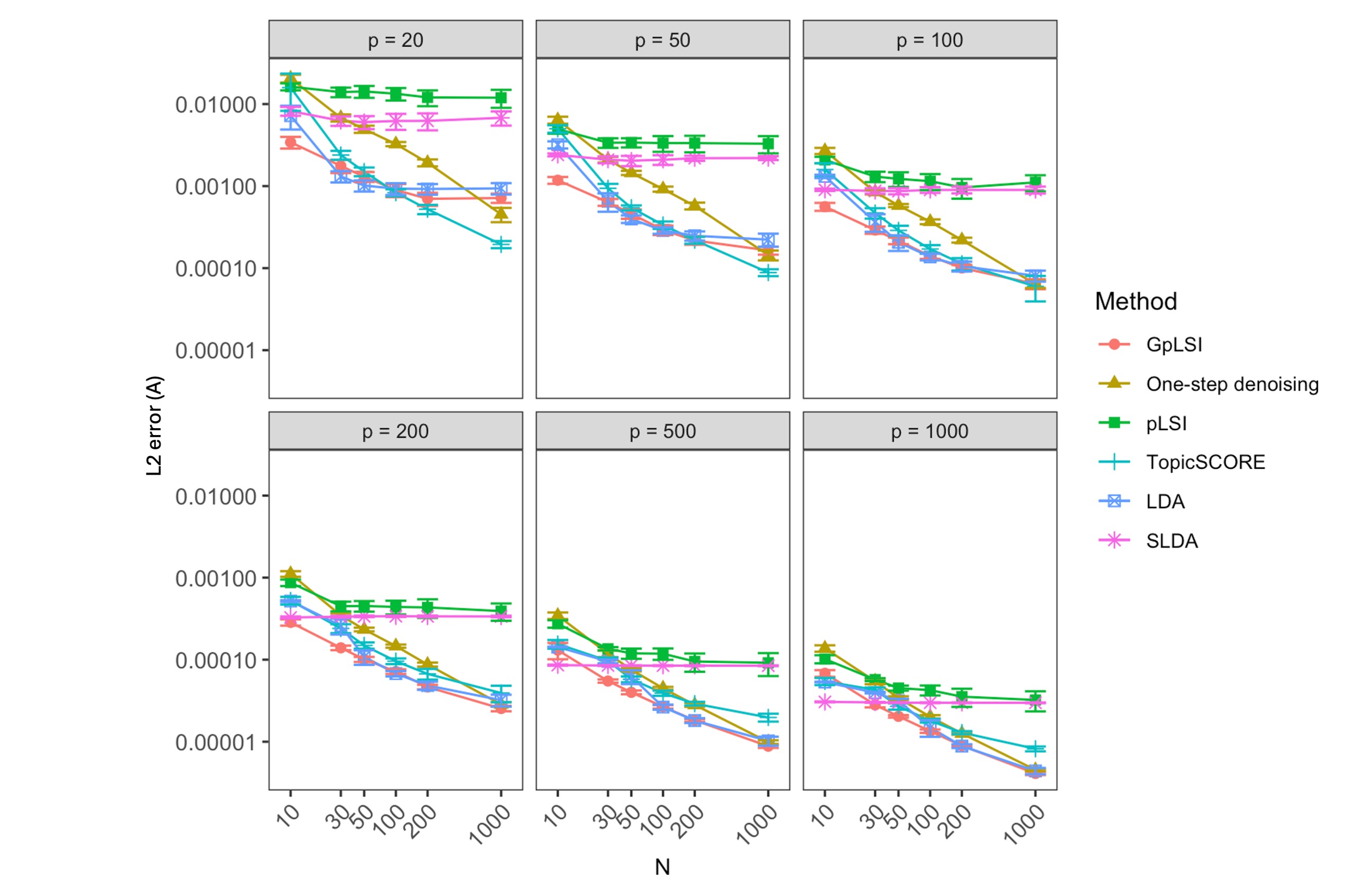}
    \caption{$\ell_2$ error for the estimator $\wh A$ (defined as $\text{min}_{P \in \mathcal{P}}\frac{1}{p}\| \wh A - PA\|_{F}$)  for different combinations of document length $N$ and vocabulary size $p$. Here, $n=1000$ and $K=3$.}
    \label{fig:Al2pN}
\end{figure}

\begin{figure}
    \centering
    \includegraphics[width=0.8\textwidth]{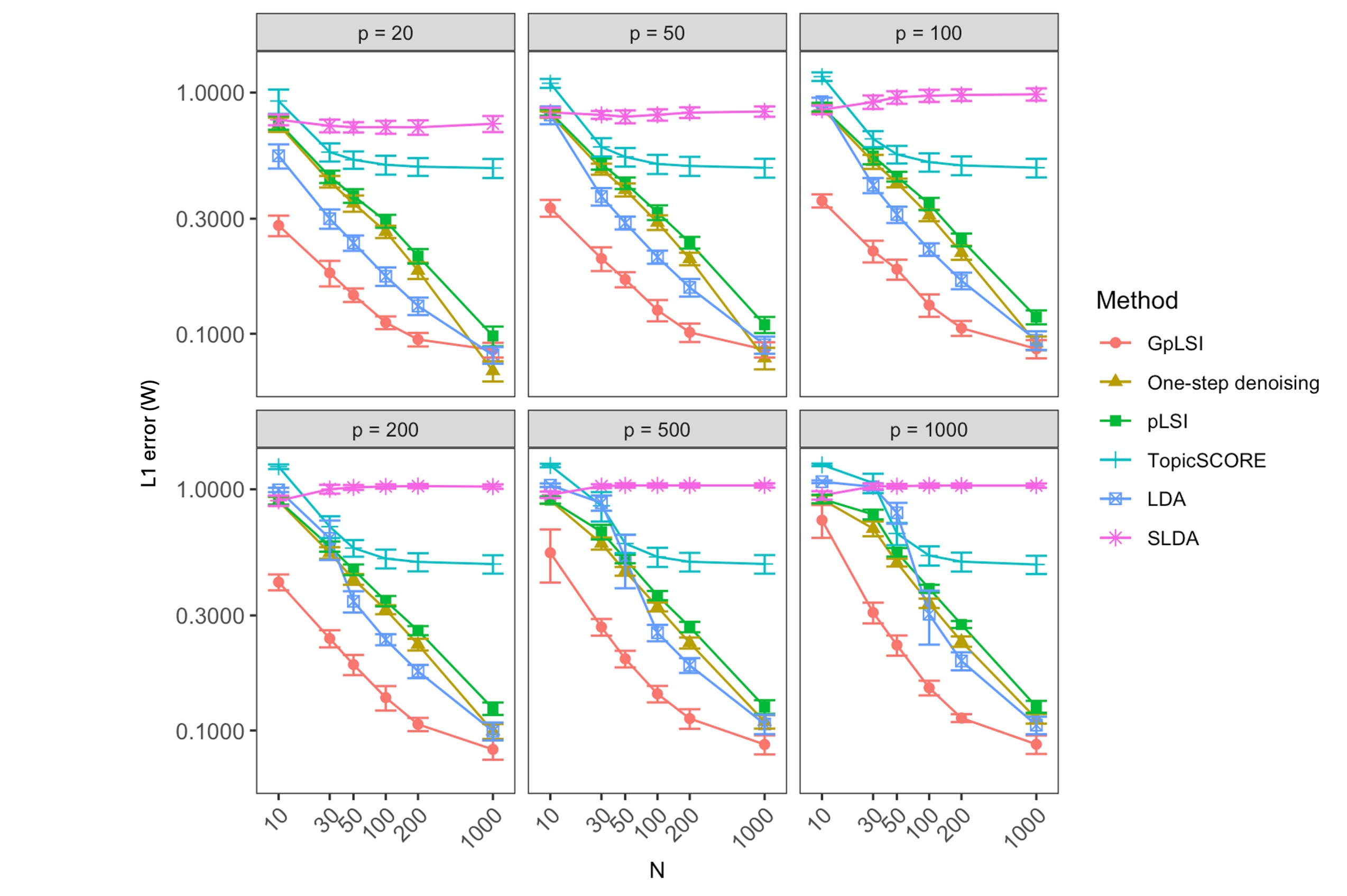}
    \caption{$\ell_1$ error  for the estimator $\wh W$ (defined as $\text{min}_{P \in \mathcal{P}}\frac{1}{n}\| \wh W - WP\|_{11}$) for different combinations of document length $N$ and vocabulary size $p$. Here, $n=1000$ and $K=3$.}
    \label{fig:Wl1pN}
\end{figure}

\begin{figure}
    \centering
    \includegraphics[width=0.8\textwidth]{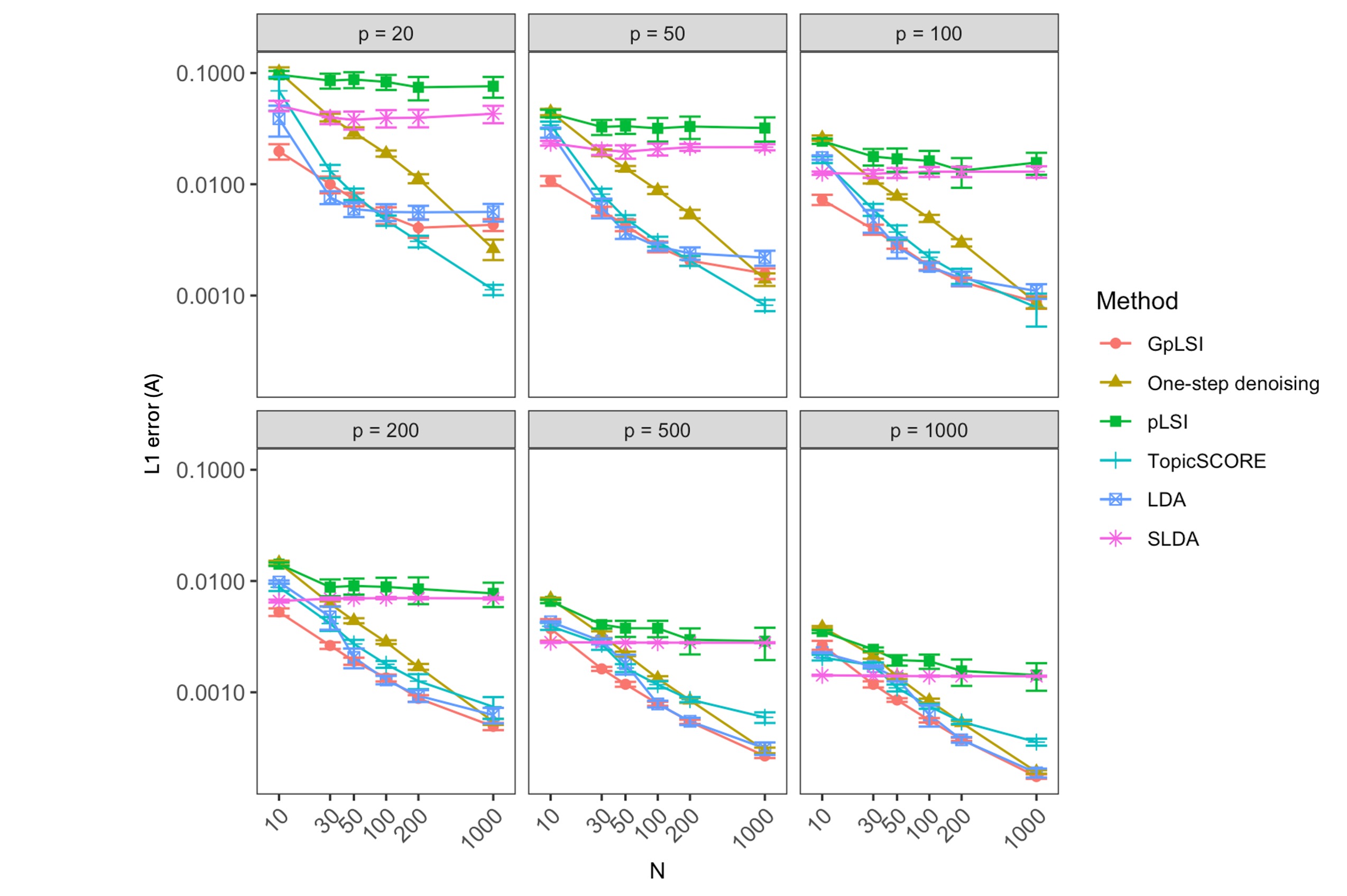}
    \caption{$\ell_1$ error for the estimator $\wh A$ (defined as $\text{min}_{P \in \mathcal{P}}\frac{1}{p}\| \wh A - PA\|_{11}$) for different combinations of document length $N$ and vocabulary size $p$. Here, $n=1000$ and $K=3$.}
    \label{fig:Al1pN}
\end{figure}


\begin{figure}
    \centering
    \includegraphics[width=0.7\textwidth]{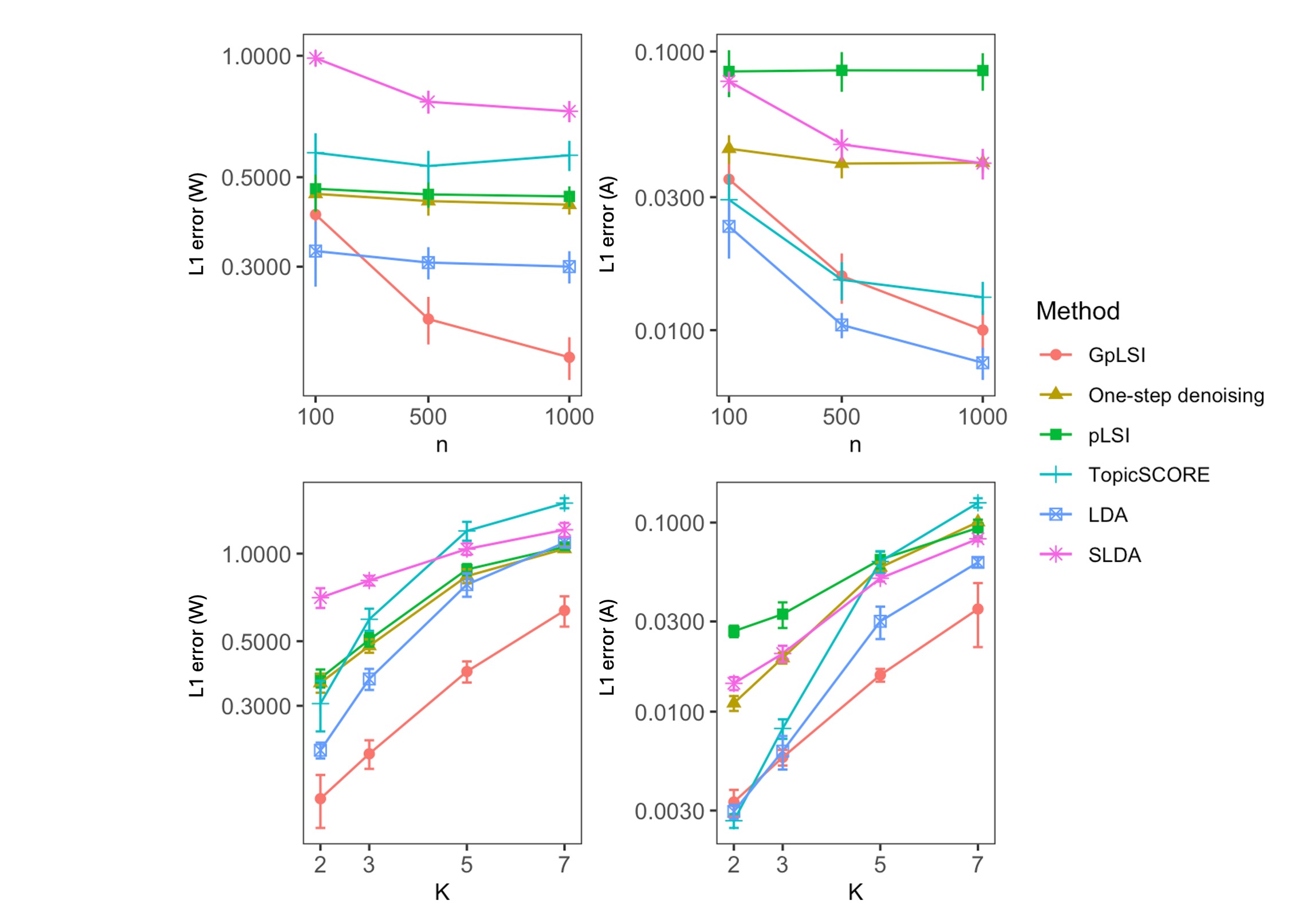}
    \caption{$\ell_1$ error of $W$ (left), $A$ (right) for different corpus size $n$ and number of topics $K$. Here, $N=30$ and $p=30$.}
    \label{fig:nK1}
\end{figure}

Next, we show the errors of estimated $\wh W$ and $\wh A$ under the same parameter settings as Section 3.4 of the main manuscript.  From Figure~\ref{fig:Al2pN}-\ref{fig:Al1pN}, GpLSI achieves the lowest errors of $\wh W$ and $\wh A$ in all parameter settings, followed by LDA. For the estimation of $A$, as highlighted in Remark 4, our rates and procedure is not optimal compared with existing results (see in particular \cite{ke2017new}, which achieves similar results to ours in Figure 2 in the main manuscript). However, compared to the procedure proposed by \cite{klopp2021assigning}, the estimation error is considerably improved.

%% file: appendixE_real_data.tex

In this section, we provide supplementary plots for our analysis on the real datasets discussed in Section 4 of the main manuscript. 

\subsection{Estimated tumor-immune microenvironment topic weights}

We present the estimated tumor-immune microenvironment topics estimated with GpLSI, pLSI, and LDA for $K=1$ to $6$. The topics are aligned among the methods as well as among different number of topics, $K$. Topics dominated by stroma, granulocyte, and B cells, recur in both GpLSI and LDA.

\begin{figure}
    \centering
    \includegraphics[width=1.0\linewidth]{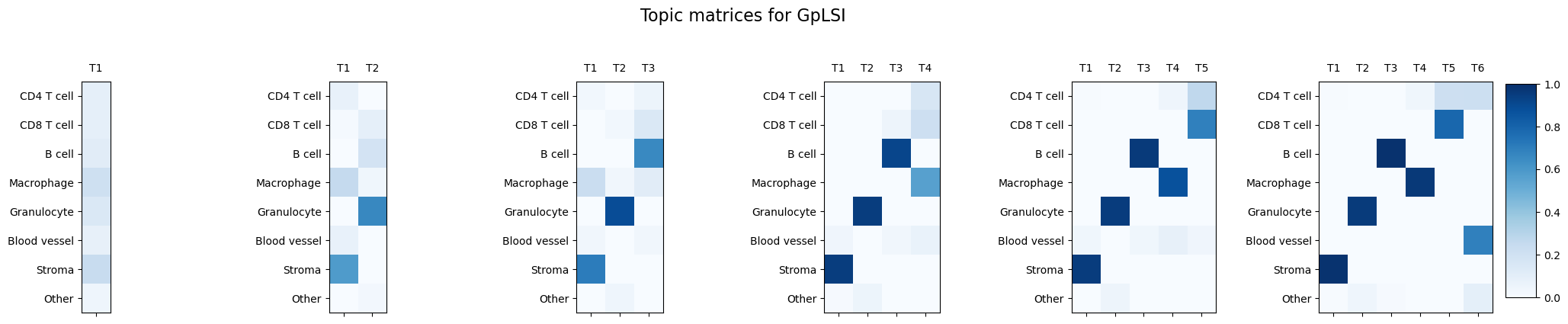}
    \caption{Estimated topic weights of tumor-immune microenvironments using GpLSI.}
    \label{fig:topics6_gplsi}
\end{figure}

\begin{figure}
    \centering
    \includegraphics[width=1.0\linewidth]{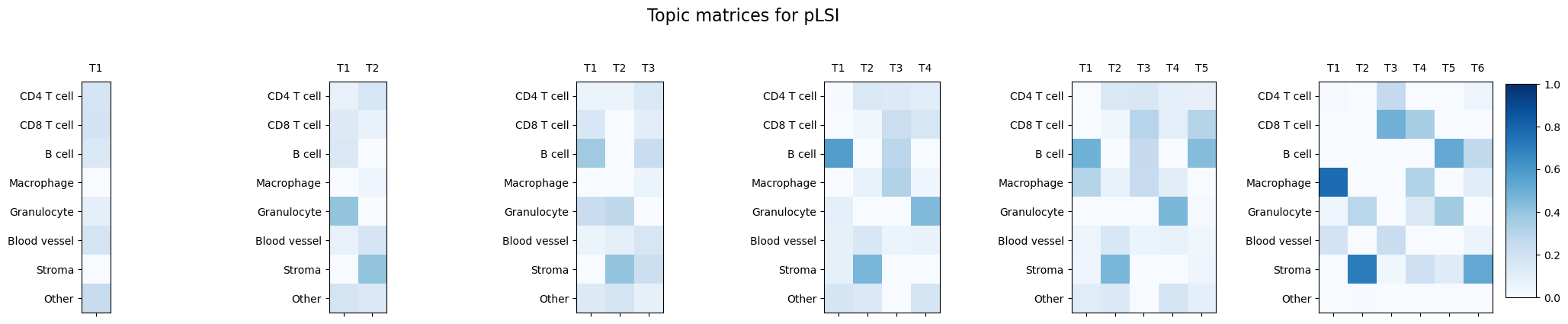}
    \caption{Estimated topic weights of tumor-immune microenvironments using pLSI.}
    \label{fig:topics6_plsi}
\end{figure}

\begin{figure}
    \centering
    \includegraphics[width=1.0\linewidth]{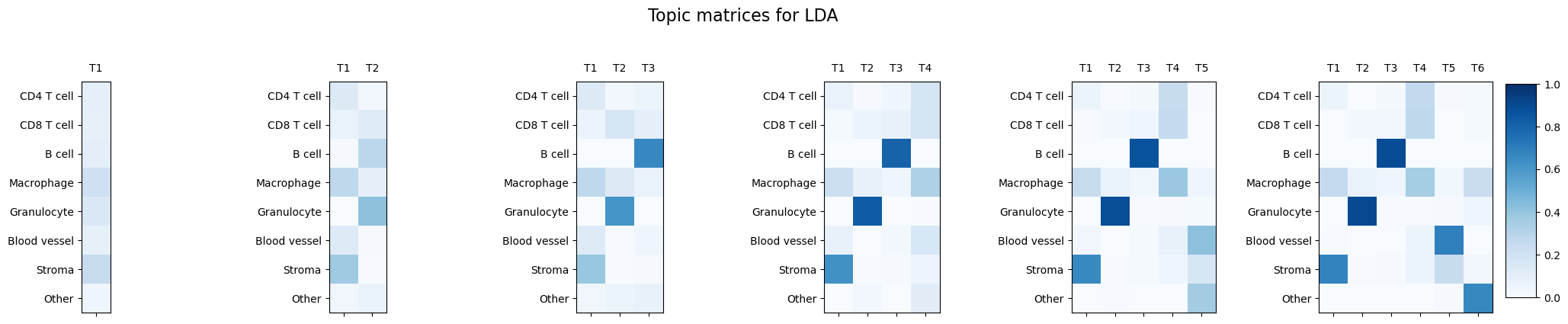}
    \caption{Estimated topic weights of tumor-immune microenvironments using LDA.}
    \label{fig:topics6_lda}
\end{figure}

\subsection{Kaplan-Meier curves of Stanford Colorectal Cancer dataset}

We plot Kaplan-Meier curves for tumor-immune micro-environment topics using the dichotomized topic proportion for each patient. We observe that granulocyte (Topic 2) is associated with lower risk of cancer recurrence across all methods.

\begin{figure}[h]
    \centering
    \includegraphics[width=1.0\linewidth]{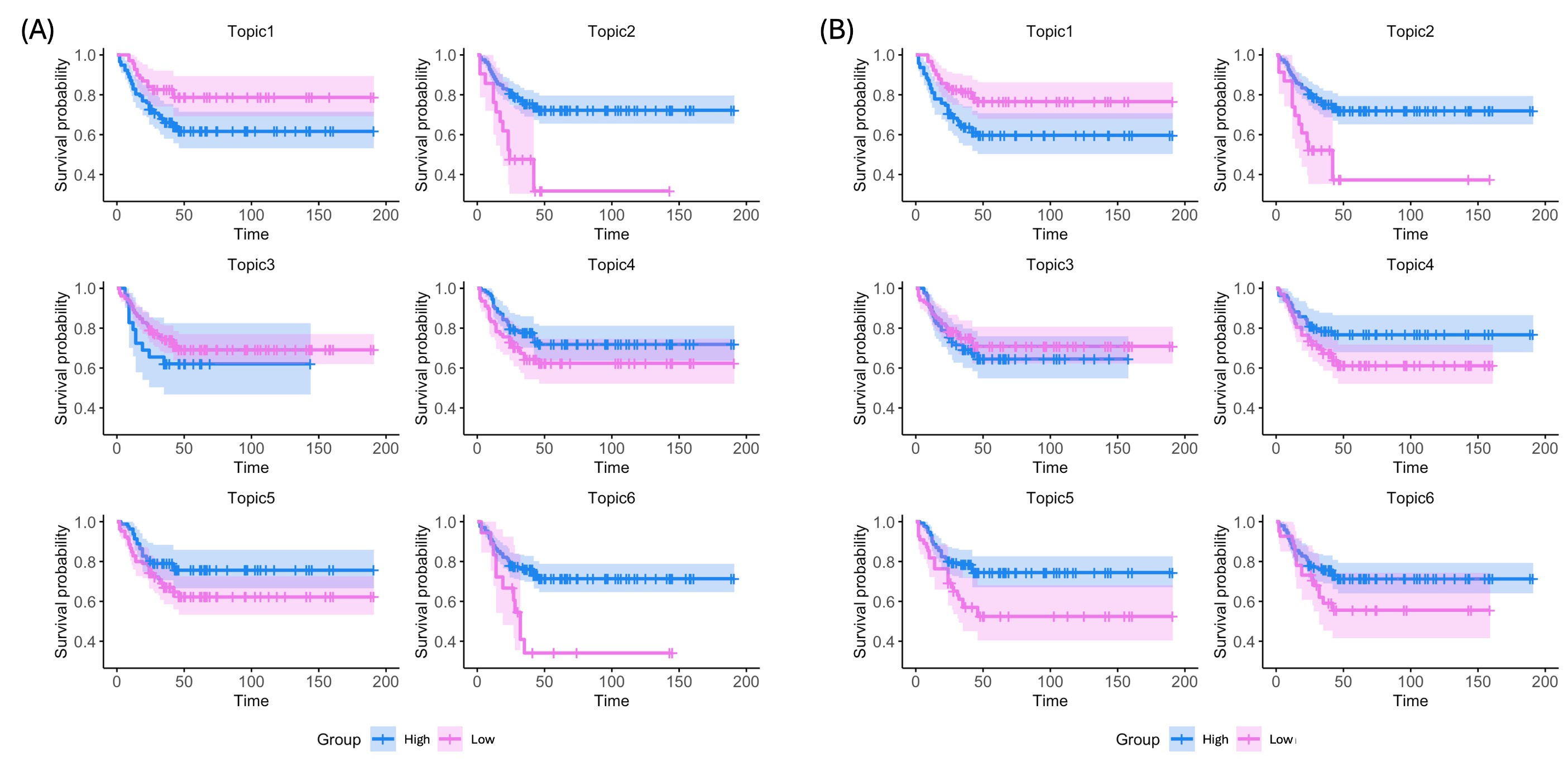}
    \caption{Kaplan-Meier curves of dichotomized topic proportions using pLSI (left) and LDA (right).}
    \label{fig:km_others}
\end{figure}

\subsection{Topics by top common ingredients in What's Cooking dataset}

\begin{figure}
    \centering
    \includegraphics[width=1.0\textwidth]{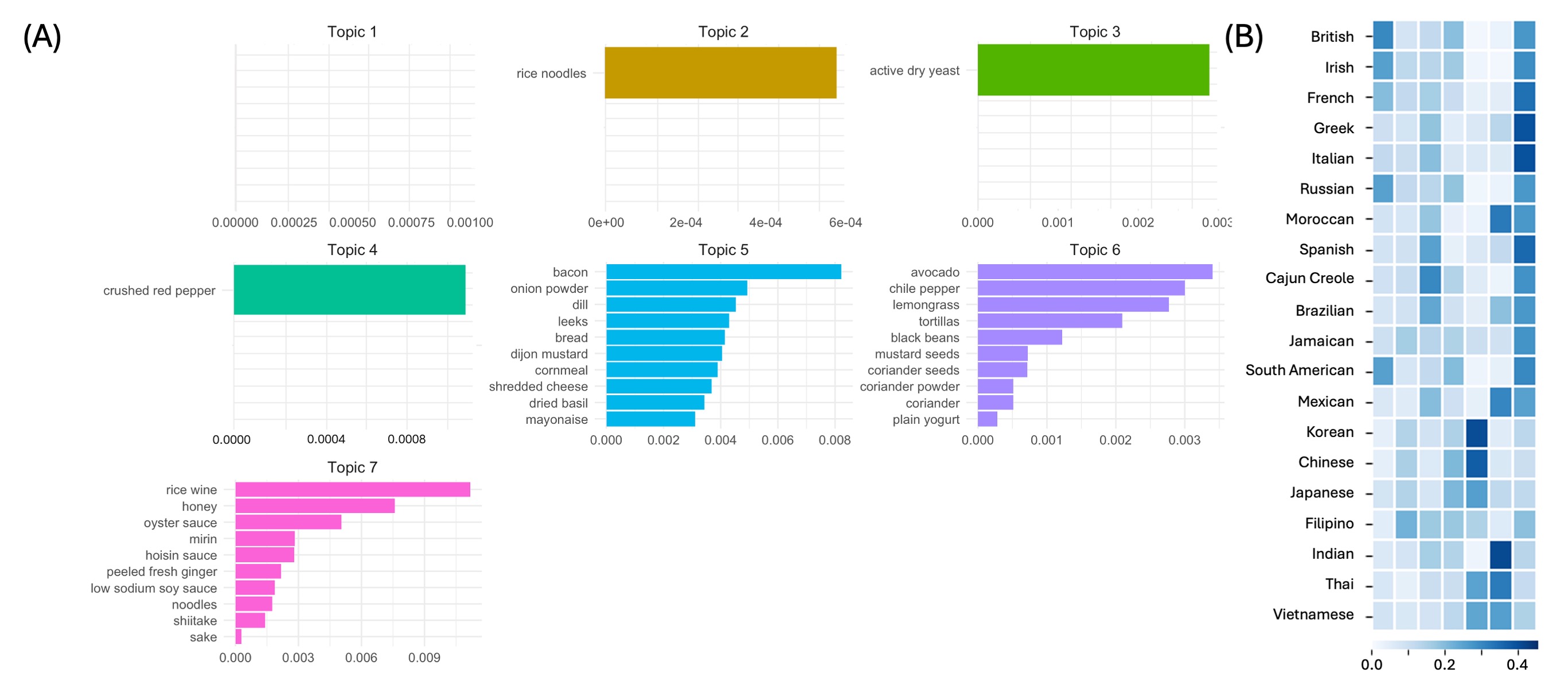}
    \caption{(A) Estimated anchor ingredients for each topic using pLSI. (B) Proportion of topics for each cuisine.}
    \label{fig:anchor_plsi}
\end{figure}

\begin{figure}
    \centering
    \includegraphics[width=1.0\textwidth]{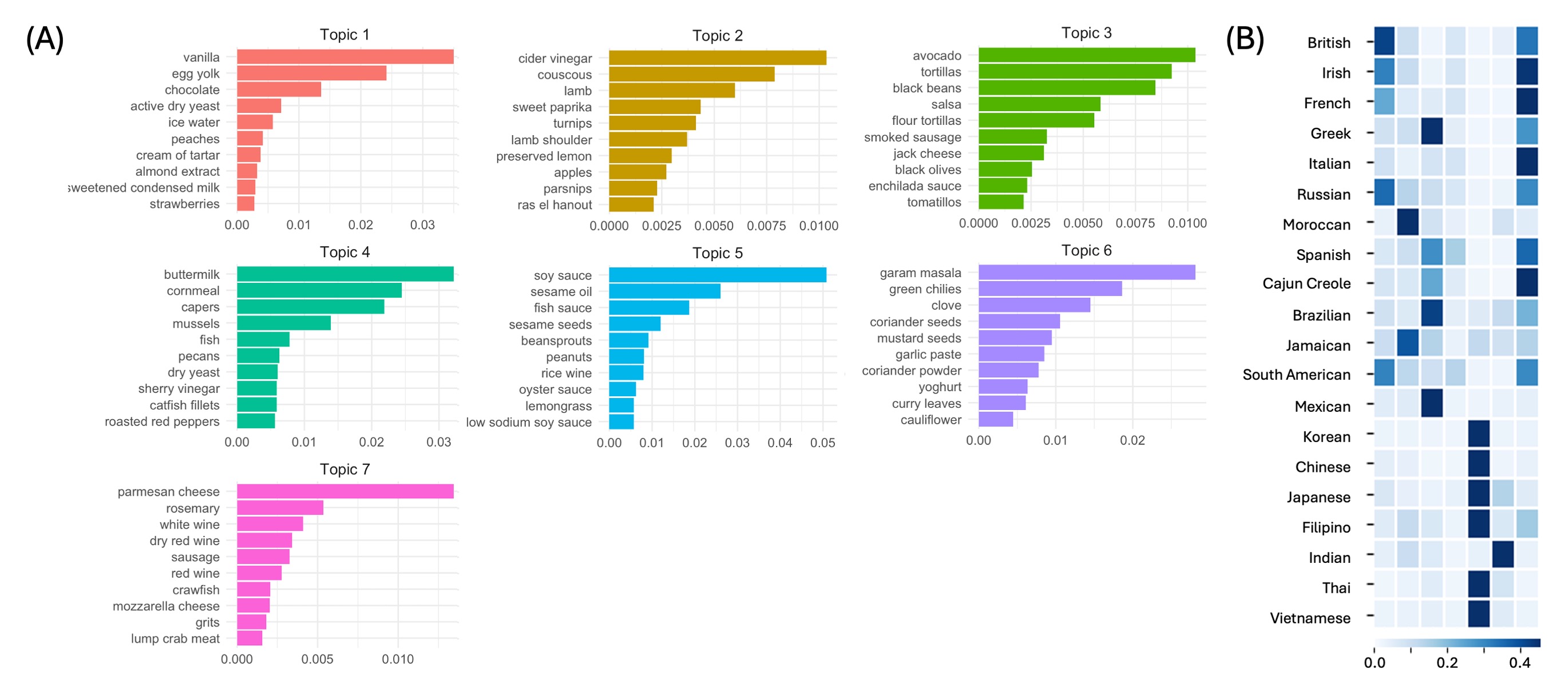}
    \caption{(A) Estimated anchor ingredients for each topic using LDA. (B) Proportion of topics for each cuisine.}
    \label{fig:anchor_lda}
\end{figure}

 \begin{figure}
    \centering
    \includegraphics[width=0.7\textwidth]{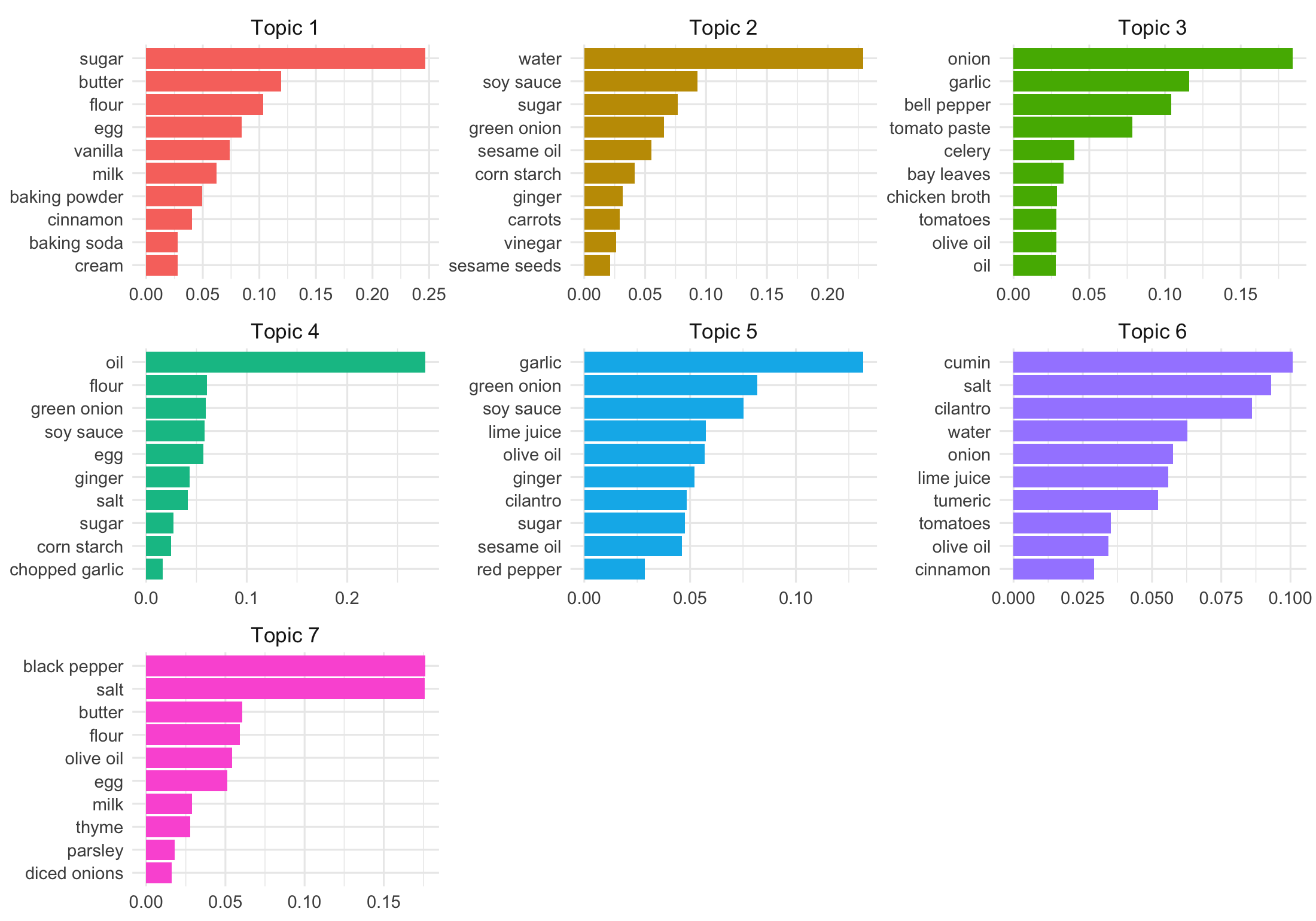}
    \caption{Top ten common words for each topic estimated by GpLSI.}
    \label{fig:top_gplsi}
\end{figure}

\begin{figure}
    \centering
    \includegraphics[width=0.7\textwidth]{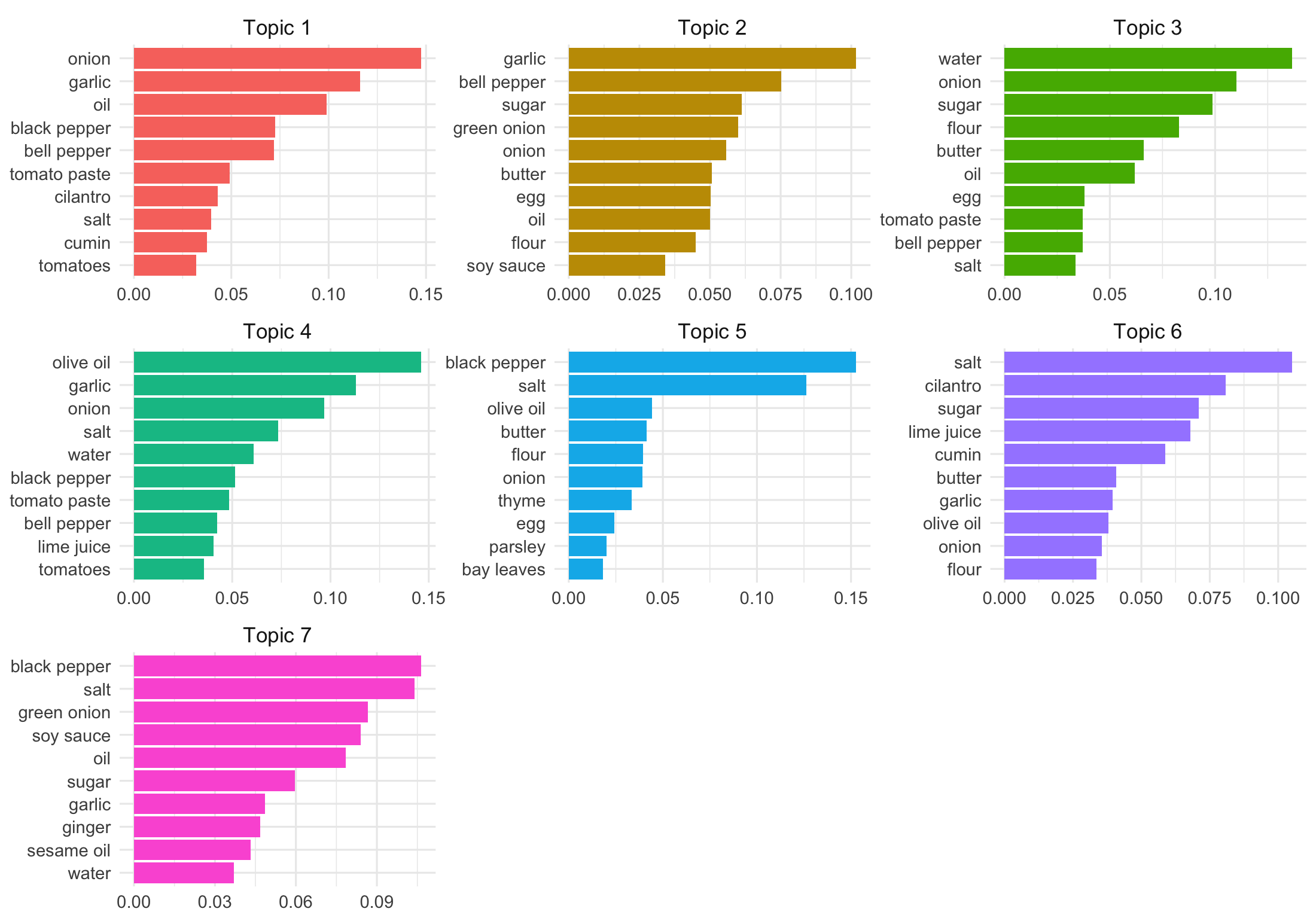}
    \caption{Top ten common words for each topic estimated by pLSI.}
    \label{fig:top_plsi}
\end{figure}

\begin{figure}
    \centering
    \includegraphics[width=0.7\textwidth]{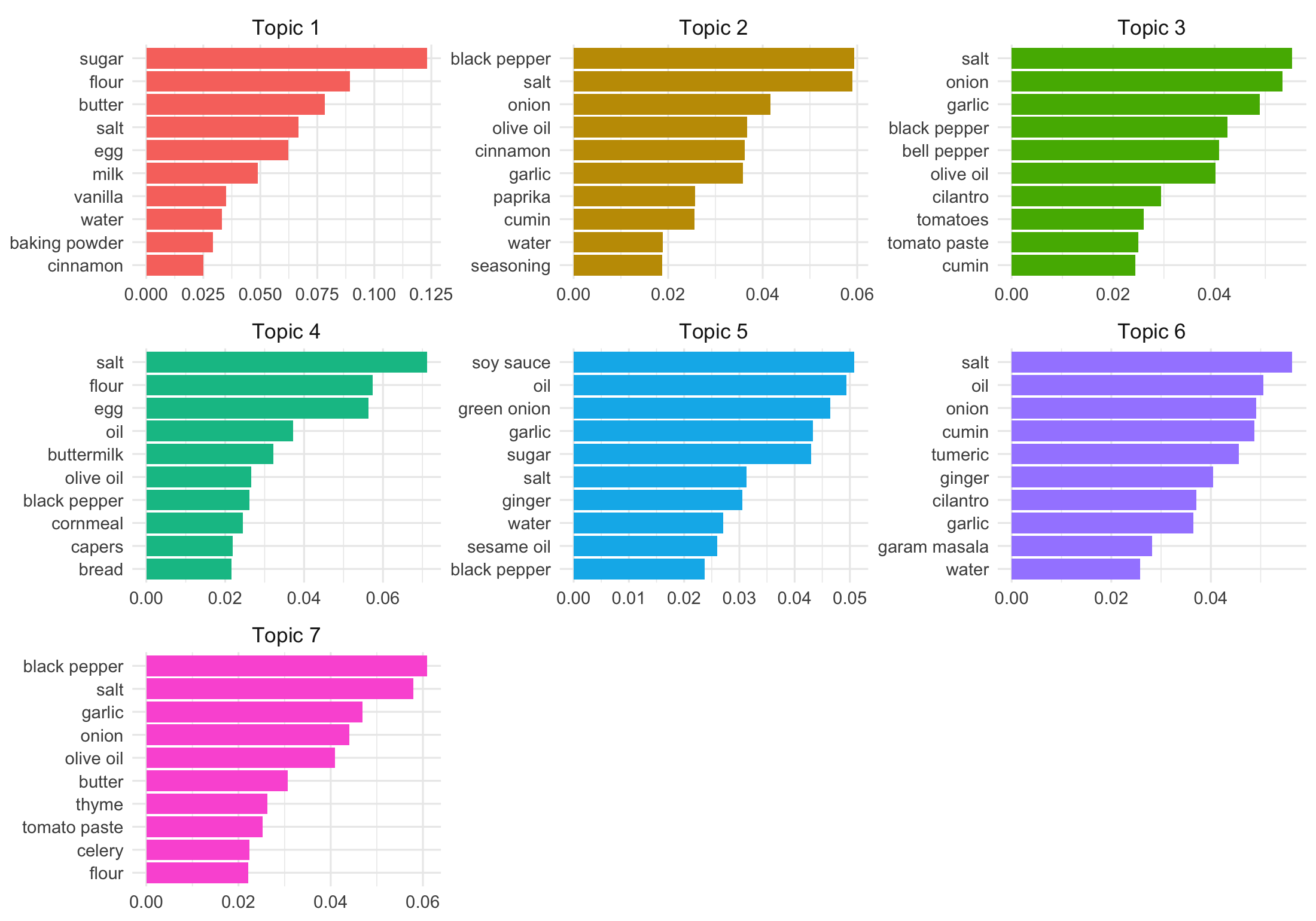}
    \caption{Top ten common words for each topic estimated by LDA.}
    \label{fig:top_lda}
\end{figure}

We illustrate each topic with the top ten ingredients with the highest estimated weights (Figures~\ref{fig:top_gplsi}-\ref{fig:top_lda}) as well as anchor ingredients for pLSI and LDA (Figures~\ref{fig:anchor_plsi}-\ref{fig:anchor_lda}). Compared to anchor ingredients, we observe that there are more overlapping ingredients among topics. While the top ten and anchor ingredients for each topic in GpLSI and LDA reflect similar styles, it is difficult to match anchor ingredients to the top ten ingredients in pLSI because the top ten ingredients are too similar across topics.

%% file: Reference.bib
@article{yang2016rate,
  title={Rate optimal denoising of simultaneously sparse and low rank matrices},
  author={Yang, Dan and Ma, Zongming and Buja, Andreas},
  journal={The Journal of Machine Learning Research},
  volume={17},
  number={1},
  pages={3163--3189},
  year={2016},
  publisher={JMLR. org}
}

@article{cape2019two,
  title={The two-to-infinity norm and singular subspace geometry with applications to high-dimensional statistics},
  author={Cape, Joshua and Tang, Minh and Priebe, Carey E},
  year={2019}
}

@book{van2024random,
  title={Random graphs and complex networks},
  author={Van Der Hofstad, Remco},
  volume={54},
  year={2024},
  publisher={Cambridge university press}
}

@article{anderson1985eigenvalues,
  title={Eigenvalues of the Laplacian of a graph},
  author={Anderson Jr, William N and Morley, Thomas D},
  journal={Linear and multilinear algebra},
  volume={18},
  number={2},
  pages={141--145},
  year={1985},
  publisher={Taylor \& Francis}
}

@article{zhang2011laplacian,
  title={The Laplacian eigenvalues of graphs: a survey},
  author={Zhang, Xiao-Dong},
  journal={arXiv preprint arXiv:1111.2897},
  year={2011}
}

@book{chung1997spectral,
  title={Spectral graph theory},
  author={Chung, Fan RK},
  volume={92},
  year={1997},
  publisher={American Mathematical Soc.}
}

@article{sun2021convex,
  title={Convex clustering: Model, theoretical guarantee and efficient algorithm},
  author={Sun, Defeng and Toh, Kim-Chuan and Yuan, Yancheng},
  journal={The Journal of Machine Learning Research},
  volume={22},
  number={1},
  pages={427--458},
  year={2021},
  publisher={JMLRORG}
}

@article{chen2020modeling,
  title={Modeling multiplexed images with spatial-LDA reveals novel tissue microenvironments},
  author={Chen, Zhenghao and Soifer, Ilya and Hilton, Hugo and Keren, Leeat and Jojic, Vladimir},
  journal={Journal of Computational Biology},
  volume={27},
  number={8},
  pages={1204--1218},
  year={2020},
  publisher={Mary Ann Liebert, Inc., publishers 140 Huguenot Street, 3rd Floor New~…}
}

@article{ke2017new,
  title={A new SVD approach to optimal topic estimation},
  author={Ke, Zheng Tracy and Wang, Minzhe},
  journal={arXiv preprint arXiv:1704.07016},
  volume={2},
  number={4},
  pages={6},
  year={2017}
}

@misc{klopp2021assigning,
      title={Assigning Topics to Documents by Successive Projections}, 
      author={Olga Klopp and Maxim Panov and Suzanne Sigalla and Alexandre Tsybakov},
      year={2021},
      eprint={2107.03684},
      archivePrefix={arXiv},
      primaryClass={math.ST}
}

@inproceedings{hutter2016optimal,
  title={Optimal rates for total variation denoising},
  author={H{\"u}tter, Jan-Christian and Rigollet, Philippe},
  booktitle={Conference on Learning Theory},
  pages={1115--1146},
  year={2016},
  organization={PMLR}
}

@article{blei2006correlated,
  title={Correlated topic models},
  author={Blei, David and Lafferty, John},
  journal={Advances in neural information processing systems},
  volume={18},
  pages={147},
  year={2006},
  publisher={MIT; 1998}
}

@article{roberts2014structural,
  title={Structural topic models for open-ended survey responses},
  author={Roberts, Margaret E and Stewart, Brandon M and Tingley, Dustin and Lucas, Christopher and Leder-Luis, Jetson and Gadarian, Shana Kushner and Albertson, Bethany and Rand, David G},
  journal={American journal of political science},
  volume={58},
  number={4},
  pages={1064--1082},
  year={2014},
  publisher={Wiley Online Library}
}

@article{tibshirani2012degrees,
  title={Degrees of freedom in lasso problems},
  author={Tibshirani, Ryan J and Taylor, Jonathan},
  year={2012}
}

@article{mcauliffe2007supervised,
  title={Supervised topic models},
  author={Mcauliffe, Jon and Blei, David},
  journal={Advances in neural information processing systems},
  volume={20},
  year={2007}
}

@misc{boucheron2013concentration,
  title={Concentration Inequalities: A Nonasymptotic Theory of Independence Oxford, UK: Oxford Univ},
  author={Boucheron, S and Lugosi, G and Massart, P},
  year={2013},
  publisher={Press}
}

@article{donoho2003does,
  title={When does non-negative matrix factorization give a correct decomposition into parts?},
  author={Donoho, David and Stodden, Victoria},
  journal={Advances in neural information processing systems},
  volume={16},
  year={2003}
}

@article{bing2020optimal,
  title={Optimal estimation of sparse topic models},
  author={Bing, Xin and Bunea, Florentina and Wegkamp, Marten},
  journal={Journal of machine learning research},
  volume={21},
  number={177},
  pages={1--45},
  year={2020}
}

@article{wu2023sparse,
  title={Sparse topic modeling: Computational efficiency, near-optimal algorithms, and statistical inference},
  author={Wu, Ruijia and Zhang, Linjun and Tony Cai, T},
  journal={Journal of the American Statistical Association},
  volume={118},
  number={543},
  pages={1849--1861},
  year={2023},
  publisher={Taylor \& Francis}
}

@article{tu2016topic,
  title={Topic modeling and improvement of image representation for large-scale image retrieval},
  author={Tu, Nguyen Anh and Dinh, Dong-Luong and Rasel, Mostofa Kamal and Lee, Young-Koo},
  journal={Information Sciences},
  volume={366},
  pages={99--120},
  year={2016},
  publisher={Elsevier}
}

@article{dey2017visualizing,
  title={Visualizing the structure of RNA-seq expression data using grade of membership models},
  author={Dey, Kushal K and Hsiao, Chiaowen Joyce and Stephens, Matthew},
  journal={PLoS genetics},
  volume={13},
  number={3},
  pages={e1006599},
  year={2017},
  publisher={Public Library of Science San Francisco, CA USA}
}

@article{yang2019characterizing,
  title={Characterizing Alzheimer's disease with image and genetic biomarkers using supervised topic models},
  author={Yang, Jie and Feng, Xinyang and Laine, Andrew F and Angelini, Elsa D},
  journal={IEEE journal of biomedical and health informatics},
  volume={24},
  number={4},
  pages={1180--1187},
  year={2019},
  publisher={IEEE}
}

@article{liu2016overview,
  title={An overview of topic modeling and its current applications in bioinformatics},
  author={Liu, Lin and Tang, Lin and Dong, Wen and Yao, Shaowen and Zhou, Wei},
  journal={SpringerPlus},
  volume={5},
  pages={1--22},
  year={2016},
  publisher={Springer}
}

@inproceedings{kho2017novel,
  title={A novel approach for classifying gene expression data using topic modeling},
  author={Kho, Soon Jye and Yalamanchili, Hima Bindu and Raymer, Michael L and Sheth, Amit P},
  booktitle={Proceedings of the 8th ACM international conference on bioinformatics, computational biology, and health informatics},
  pages={388--393},
  year={2017}
}

@article{zheng2015deep,
  title={A deep and autoregressive approach for topic modeling of multimodal data},
  author={Zheng, Yin and Zhang, Yu-Jin and Larochelle, Hugo},
  journal={IEEE transactions on pattern analysis and machine intelligence},
  volume={38},
  number={6},
  pages={1056--1069},
  year={2015},
  publisher={IEEE}
}

@inproceedings{shao2009semi,
  title={Semi-supervised topic modeling for image annotation},
  author={Shao, Yuanlong and Zhou, Yuan and He, Xiaofei and Cai, Deng and Bao, Hujun},
  booktitle={Proceedings of the 17th ACM international conference on Multimedia},
  pages={521--524},
  year={2009}
}

@inproceedings{feng2010topic,
  title={Topic models for image annotation and text illustration},
  author={Feng, Yansong and Lapata, Mirella},
  booktitle={Human Language Technologies: The 2010 Annual Conference of the North American Chapter of the ACL},
  pages={831--839},
  year={2010},
  organization={Association for Computational Linguistics}
}

@article{sankaran2019latent,
  title={Latent variable modeling for the microbiome},
  author={Sankaran, Kris and Holmes, Susan P},
  journal={Biostatistics},
  volume={20},
  number={4},
  pages={599--614},
  year={2019},
  publisher={Oxford University Press}
}

@book{giraud2021introduction,
  title={Introduction to high-dimensional statistics},
  author={Giraud, Christophe},
  year={2021},
  publisher={Chapman and Hall/CRC}
}

@article{tran2023sparse,
  title={Sparse topic modeling via spectral decomposition and thresholding},
  author={Tran, Huy and Liu, Yating and Donnat, Claire},
  journal={arXiv preprint arXiv:2310.06730},
  year={2023}
}

@inproceedings{blei2006dynamic,
  title={Dynamic topic models},
  author={Blei, David M and Lafferty, John D},
  booktitle={Proceedings of the 23rd international conference on Machine learning},
  pages={113--120},
  year={2006}
}

@article{blei2001latent,
  title={Latent dirichlet allocation},
  author={Blei, David and Ng, Andrew and Jordan, Michael},
  journal={Advances in neural information processing systems},
  volume={14},
  year={2001}
}

@article{symul2023sub,
  title={Sub-communities of the vaginal microbiota in pregnant and non-pregnant women},
  author={Symul, Laura and Jeganathan, Pratheepa and Costello, Elizabeth K and France, Michael and Bloom, Seth M and Kwon, Douglas S and Ravel, Jacques and Relman, David A and Holmes, Susan},
  journal={Proceedings of the Royal Society B},
  volume={290},
  number={2011},
  pages={20231461},
  year={2023},
  publisher={The Royal Society}
}

@article{reder2021supervised,
  title={Supervised topic modeling for predicting molecular substructure from mass spectrometry},
  author={Reder, Gabriel K and Young, Adamo and Altosaar, Jaan and Rajniak, Jakub and Elhadad, No{\'e}mie and Fischbach, Michael and Holmes, Susan},
  journal={F1000Research},
  volume={10},
  pages={Chem--Inf},
  year={2021},
  publisher={F1000 Research Limited London, UK}
}

@article{araujo2001successive,
  title={The successive projections algorithm for variable selection in spectroscopic multicomponent analysis},
  author={Ara{\'u}jo, M{\'a}rio C{\'e}sar Ugulino and Saldanha, Teresa Cristina Bezerra and Galvao, Roberto Kawakami Harrop and Yoneyama, Takashi and Chame, Henrique Caldas and Visani, Valeria},
  journal={Chemometrics and intelligent laboratory systems},
  volume={57},
  number={2},
  pages={65--73},
  year={2001},
  publisher={Elsevier}
}

@article {wu2022spacegm,
    author = {Wu, Zhenqin and Trevino, Alexandro E. and Wu, Eric and Swanson, Kyle and Kim, Honesty J. and D{\textquoteright}Angio, H. Blaize and Preska, Ryan and Charville, Gregory W. and Dalerba, Piero D. and Egloff, Ann Marie and Uppaluri, Ravindra and Duvvuri, Umamaheswar and Mayer, Aaron T. and Zou, James},
    title = {SPACE-GM: geometric deep learning of disease-associated microenvironments from multiplex spatial protein profiles},
    elocation-id = {2022.05.12.491707},
    year = {2022},
    doi = {10.1101/2022.05.12.491707},
    publisher = {Cold Spring Harbor Laboratory},
    URL = {https://www.biorxiv.org/content/early/2022/05/13/2022.05.12.491707},
    eprint = {https://www.biorxiv.org/content/early/2022/05/13/2022.05.12.491707.full.pdf},
    journal = {bioRxiv}
}

@article{fukuyama2023multiscale,
  title={Multiscale analysis of count data through topic alignment},
  author={Fukuyama, Julia and Sankaran, Kris and Symul, Laura},
  journal={Biostatistics},
  volume={24},
  number={4},
  pages={1045--1065},
  year={2023},
  publisher={Oxford University Press}
}

@article{shang2022spatially,
  title={Spatially aware dimension reduction for spatial transcriptomics},
  author={Shang, Lulu and Zhou, Xiang},
  journal={Nature communications},
  volume={13},
  number={1},
  pages={7203},
  year={2022},
  publisher={Nature Publishing Group UK London}
}

@article{goltsev2018deep,
  title={Deep profiling of mouse splenic architecture with CODEX multiplexed imaging},
  author={Goltsev, Yury and Samusik, Nikolay and Kennedy-Darling, Julia and Bhate, Salil and Hale, Matthew and Vazquez, Gustavo and Black, Sarah and Nolan, Garry P},
  journal={Cell},
  volume={174},
  number={4},
  pages={968--981},
  year={2018},
  publisher={Elsevier}
}

@article{gillis2015semidefinite,
  title={Semidefinite programming based preconditioning for more robust near-separable nonnegative matrix factorization},
  author={Gillis, Nicolas and Vavasis, Stephen A},
  journal={SIAM Journal on Optimization},
  volume={25},
  number={1},
  pages={677--698},
  year={2015},
  publisher={SIAM}
}

@article{cai2018rate,
  title={Rate-optimal perturbation bounds for singular subspaces with applications to high-dimensional statistics},
  author={Cai, T Tony and Zhang, Anru},
  year={2018}
}

@article{wedin1972perturbation,
  title={Perturbation bounds in connection with singular value decomposition},
  author={Wedin, Per-{\AA}ke},
  journal={BIT Numerical Mathematics},
  volume={12},
  pages={99--111},
  year={1972},
  publisher={Springer}
}

@book{vershynin2018high,
  title={High-dimensional probability: An introduction with applications in data science},
  author={Vershynin, Roman},
  volume={47},
  year={2018},
  publisher={Cambridge university press}
}
